\documentclass{article}
\usepackage{iclr2019,times}
\iclrfinalcopy 

\usepackage{graphicx}
\usepackage{algorithm, algorithmic}
\usepackage{amsmath, amsthm, amssymb}
\usepackage{relsize}
\usepackage{color}
\usepackage{subcaption}
\usepackage{url}
\usepackage{enumitem}
\usepackage[colorlinks=true,linkcolor=blue,urlcolor=blue,citecolor=blue]{hyperref} 
\usepackage{multirow}
\usepackage{colortbl}
\definecolor{LightCyan}{rgb}{0.88,1,1}

\newcommand{\ff}{f^*\!\!~}
\newcommand{\tsim}{\!\sim\!}

\newcommand{\dloss}{J_{\mathsmaller{D}}}
\newcommand{\gloss}{J_{\mathsmaller{G}}}
\newcommand\Tstrut{\rule{0pt}{2.0ex}}         % = `top' strut
\newtheoremstyle{examplestyle}
  {1.1\topsep} % Space above
  {1.1\topsep} % Space below
  {\itshape} % Body font
  {} % Indent amount
  {\bfseries} % Theorem head font
  {.} % Punctuation after theorem head
  {0.5em} % Space after theorem head
  {} % Theorem head spec (can be left empty, meaning `normal')
\theoremstyle{examplestyle}

\newtheorem{theorem}{Theorem}%[subsection]
\newtheorem{lemma}{Lemma}%[subsection]
\newtheorem*{remark}{Remark}

\newcommand{\rnum}[1]{\expandafter{\romannumeral #1\relax}}
\newcommand{\RNum}[1]{\uppercase\expandafter{\romannumeral #1\relax}}

\newcommand{\tpdv}[2]{\frac{\partial{#1}}{\partial{#2}}}
\newcommand{\pdv}[2]{\frac{\partial{#1}}{\partial{#2}}}

\newcommand{\pdata}{P_r}
\newcommand{\pgen}{P_g}

\newcommand{\sdata}{\bar{P_r}}
\newcommand{\sgen}{\bar{P_g}}

\newcommand{\E}{\mathbb E}
\newcommand{\mE}{\mathbb E}

\DeclareMathOperator*{\argmin}{arg\,min}

\newcommand{\jxloss}{[{\partial_{x}}{\dloss}]}

\title{Understanding the Effectiveness of Lipschitz-Continuity in Generative Adversarial Nets}

\author{
    \vspace{1.5pt}
	\hspace{-2.0cm}Zhiming Zhou\footnotemark[2]$\,\,^1\!\!$, Yuxuan Song\footnotemark[3]$\,\,^1\!\!$, Lantao Yu$^2\!\!$, Hongwei Wang$^1\!\!$, Jiadong Liang$^3\!\!$, Weinan Zhang$^1\!\!$, Zhihua Zhang\footnotemark[4]$\,\,^3\!\!$, Yong Yu\footnotemark[6]$\,\,^1$\\
	\hspace{1.8cm}$^1\!\!$~Shanghai Jiao Tong University, ~$^2\!\!$~Stanford University, ~$^3\!\!$~Peking University \\
	\hspace{-1.8cm}\texttt{\footnotemark[2]$\,\,$heyohai,\footnotemark[3]$\,\,$yuxuansong@apex.sjtu.edu.cn;\footnotemark[4]$\,\,$zhzhang@math.pku.edu.cn;\footnotemark[6]$\,\,$yyu@apex.sjtu.edu.cn}%\\
	\vspace{-12pt}
}

\begin{document}

\maketitle

\begin{abstract}
In this paper, we investigate the underlying factor that leads to failure and success in the training of GANs. We study the property of the optimal discriminative function and show that in many GANs, the gradient from the optimal discriminative function is not reliable, which turns out to be the fundamental cause of failure in the training of GANs. We further demonstrate that a well-defined distance metric does not necessarily guarantee the convergence of GANs. Finally, we prove in this paper that Lipschitz-continuity condition is a general solution to make the gradient of the optimal discriminative function reliable, and characterized the necessary condition where Lipschitz-continuity ensures the convergence, which leads to a broad family of valid GAN objectives under Lipschitz-continuity condition, where Wasserstein distance is one special case. We experiment with several new objectives, which are sound according to our theorems, and we found that, compared with Wasserstein distance, the outputs of the discriminator with new objectives are more stable and the final qualities of generated samples are also consistently higher than those produced by Wasserstein distance. 

\end{abstract} 

\section{Introduction}
Generative Adversarial Networks (GANs) \citep{gan}, as a new way of learning generative
models, have recently shown promising results in various challenging tasks. %, consisting of a generator network and a discriminator network, is a promising framework of generative model, where the generator learns to generate sample with the input from a given distribution (e.g.~Gaussian distribution) to a target distribution and the discriminator learns to distinguish generated samples from the real ones. %Such an adversarially trained discriminator provides guidance to the generator on the updating of fake samples, and 
%With properly defined objectives, the distribution of generated samples will converge to the target distribution with theoretical guarantees.
%
Although GANs are popular and widely-used \citep{image2image,photo_editing_with_gan,ppgn,cycle_gan,progressive_growing_gan}, they are notoriously hard to train \citep{gan_tutorial}. The underlying obstacles, though have been widely studied \citep{principled_methods,lucic2017gans,heusel2017gans,mescheder2017numerics,mescheder2018training,yadav2017stabilizing}, are still not fully understood. 

In this paper, we study the convergence of GANs from the perspective of the optimal discriminative function $\ff(x)$. %, which leads to a comprehensive understanding of GANs' training problems. 
%
%To be more specific, 
We show that in original GAN and its most variants, $\ff(x)$ is a function of densities at the current point $x$ but does not reflect any information about the densities/locations of other points in the real and fake distributions. Moreover, \cite{principled_methods} state that the supports of real and fake distributions are usually disjoint. In this paper, we argue that the fundamental cause of failure in training of GANs (Section~\ref{sec_tells_nothing}) stems from the combination of the above two facts. The generator uses $\nabla_{\!x} \ff(x)$ as the guidance for updating the generated samples, but $\nabla_{\!x} \ff(x)$ actually tells nothing about where $P_r$ is. As the result, the generator is not guaranteed to converge to $P_r$. %, i.e., $P_r(x)$ and $P_g(x)$ with $P_g$ and $P_r$ represent the real and fake distribution respectively,

Accordingly, \cite{wgan} proposed Wasserstein distance (in its dual form) as an alternative objective, which can properly measure the distance between two distributions no matter whether their supports are disjoint. However, as shown in Section~\ref{sec_w_also_suffer}, when the supports of the $P_g$ and $P_r$ are disjoint, the gradient of $\ff(x)$ from the dual form of Wasserstein distance given a compacted dual constraint also does not reflect any useful information about other points in $P_r$. %That is Wasserstein distance might also unable to guarantee the convergence. 
Based on this observation, we provide further investigation in Section~\ref{sec_not_enough} and argue that measuring the distance properly does not necessarily imply that the gradient is well-defined in terms of $\nabla_{\!x} \ff(x)$. 

In Section~\ref{sec_lip}, we show that incorporating Lipschitz-continuity condition in the objectives of GANs is a general solution to the above mentioned problem, and prove that for a \textbf{broad} family of discriminator objectives, Lipschitz-continuity condition can build strong connections between $P_g$ and $P_r$ through $\ff(x)$ such that $\nabla_{\!x} \ff(x)$ at each sample $x \tsim P_g$ will point towards some real sample $y \tsim P_r$. This guarantees that $P_g$ is moving towards $P_r$ at every step, i.e, the convergence of GANs is guaranteed. %\lantao{Does each fake sample's gradient pointing to a real sample necessarily implies final convergence?}

Finally, in Section~\ref{sec_overlapped}, we extend our discussion on $\ff(x)$ and $\nabla_{\!x} \ff(x)$ to the case where the supports of $P_g$ and $P_r$ are overlapped and show that the locality of $\ff(x)$ and  $\nabla_{\!x} \ff(x)$ in traditional GANs turns out to be an intrinsic cause to mode collapse. 

%Finally, in Section~\ref{sec_success}, we explain the reason of empirical success of traditional GANs under the circumstance that they have no convergence guarantee. 

\begin{table}[h]
\vspace{-0pt}
\caption{Comparison of different objectives in GAN models.}
\label{table1}
\centering
\vspace{-6pt}
\resizebox{0.99\textwidth}{!}{
\begin{tabular}{|c|c|c|c|c|}
\hline 
                    & $\phi$                & $\varphi$           & $\mathcal{F}$                  & $\ff(x)$ \\
\hline 
JS-Divergence       & $-\log(\sigma(-x))$  & $-\log(\sigma(x))$  & $\{f: \mathbb R^n \rightarrow \mathbb R \}$    & $\log\frac{P_r(x)}{P_g(x)}$ \\
\hline
Least Square        & $(x-\alpha)^2$        & $(x-\beta)^2$       & $\{f: \mathbb R^n \rightarrow \mathbb R \}$    & $\frac{\alpha \cdot P_g(x)+\beta \cdot P_r(x)}{P_g(x)+P_r(x)}$ \\ 
\hline 
Wasserstein-1 with Lip$_1$       & $x$                   & $-x$                & $\{f: \mathbb R^n \rightarrow \mathbb R, ~ \lVert f \rVert_{lip} \leq 1 \}$ & $\mathsmaller{N/A}$  \\
\hline 
$\mu$-Fisher IPM    & $x$                   & $-x$                & $\{f: \mathbb R^n \rightarrow \mathbb R, ~ \E_{x\sim \mu} \lVert  f(x) \rVert^2 \leq 1 \}$   & $\frac{1}{\mathcal{F}_{\mu}(P_r,P_g)}\frac{P_r(x)-P_g(x)}{\mu(x)}$ \\
\hline 
\end{tabular}
}
\vspace{-5pt}
\end{table}

\section{The Fundamental Cause of Failure in Training of GANs} \label{sec_failue_cause}

Typically, the objectives of GANs can be formulated as follows:
\begin{equation} \label{eq_gan_formulation}
\begin{aligned}
	&\min_{f \in \mathcal{F}} \ J_D \triangleq \E_{z \sim P_z} [ \phi(f(g(z))) ] + \E_{x \sim P_r} [ \varphi(f(x)) ], \\ 
	&\min_{g \in \mathcal{G}} \ J_G \triangleq \E_{z \sim P_z} [ \psi(f(g(z))) ],
\end{aligned}
\end{equation}
where $P_z$ is the source distribution of the generator (usually a Gaussian distribution) in $\mathbb R^m$ and $P_r$ is the target (real) distribution in $\mathbb R^n$. The generative function $g: \mathbb R^m\rightarrow \mathbb R^n$ learns to output samples that shares the same dimension as $P_r$, while the discriminative function $f: \mathbb R^n \rightarrow \mathbb R$ learns to output a score indicating the authenticity of a given sample. We denote the implicit distribution of the generated samples as $P_g$, i.e., $P_g = g(P_z)$. 

$\mathcal{F}$ and $\mathcal{G}$ denote discriminative and generative function spaces parameterized by neural networks, respectively; functions $\phi$, $\varphi$, $\psi$: $\mathbb R \rightarrow \mathbb R$ are loss metrics. We list the choices of $\mathcal{F}$, $\phi$ and $\varphi$ in some representative GAN models in Table~\ref{table1}, where we denote $f^* = \argmin_{f \in \mathcal{F}} J_D$. 

In these GANs, the gradient that the generator receives from the discriminator with respect to a generated sample $x \tsim P_g$ is
\vspace{-2pt}
\begin{equation} \label{eq_gen_grad}
\begin{aligned}
\nabla_{\!x} \gloss(x) = \nabla_{\!f(x)} \psi(f(x)) \cdot \nabla_{\!x} f(x).
\end{aligned} 
\end{equation}
In Eq.~(\ref{eq_gen_grad}), the first term $\nabla_{\!f(x)} \psi(f(x))$ is a step-related scalar that is out of the scope of our discussion in this paper; the second term $\nabla_{\!x} f(x)$ is a vector indicating the direction that the generator should follow for optimizing on sample $x$. %In addition, 

\subsection{$\nabla_{\!x} \ff(x)$ on $P_g$ does not reflect useful information about $P_r$} \label{sec_tells_nothing}

In this section, we will show that when the supports of $P_g$ and $P_r$ are disjoint, $\nabla_{\!x} \ff(x)$ in traditional GANs does not reflect any useful information about $P_r$, and $P_g$ is not guaranteed to converge to $P_r$. 
We argue that this is the fundamental cause of non-convergence and instability in traditional GANs.\footnote{In this paper, traditional GANs mainly refers to the original GAN and Least-Squares GAN, where $\ff(x)$ depends only on the densities $P_g(x)$ and $P_r(x)$. Broadly, it refers to all GANs where $\ff(x)$ does not reflect information about the locations of the other points in $P_g$ and $P_r$, such as the Fisher GAN.}

\begin{figure}[t]
\vspace{-5pt}
\begin{subfigure}{0.33\linewidth}
    \centering
    \includegraphics[width=0.90\columnwidth]{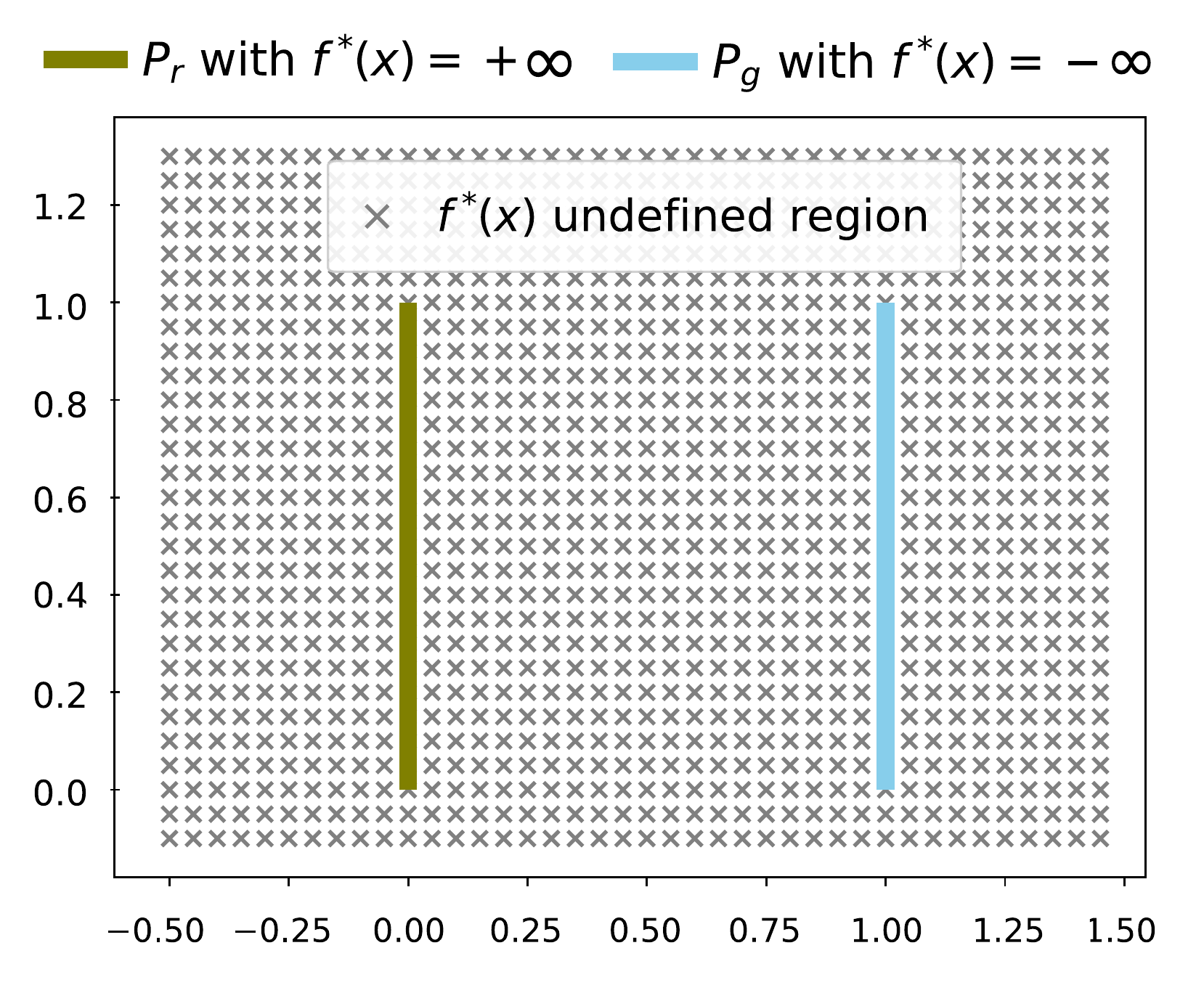}
    \vspace{-3pt}
    \caption{Original GAN}
    \label{fig1_originalgan}
\end{subfigure}	
\begin{subfigure}{0.33\linewidth}
    \centering
    \includegraphics[width=0.90\columnwidth]{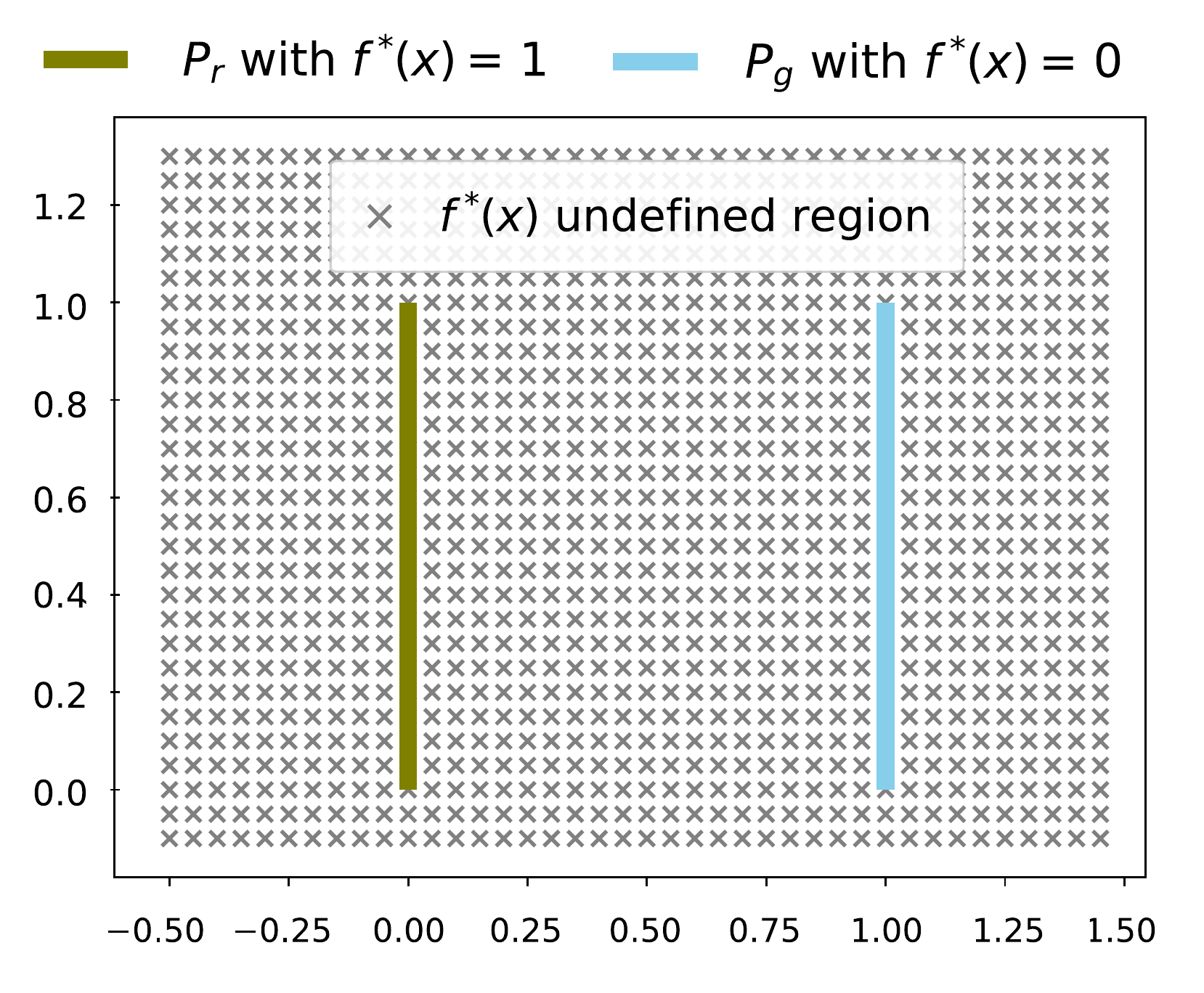}
    \vspace{-3pt}
    \caption{Wasserstein distance*}
    \label{fig1_wdistance}
\end{subfigure}
\begin{subfigure}{0.33\linewidth}
    \centering
    \includegraphics[width=0.90\columnwidth]{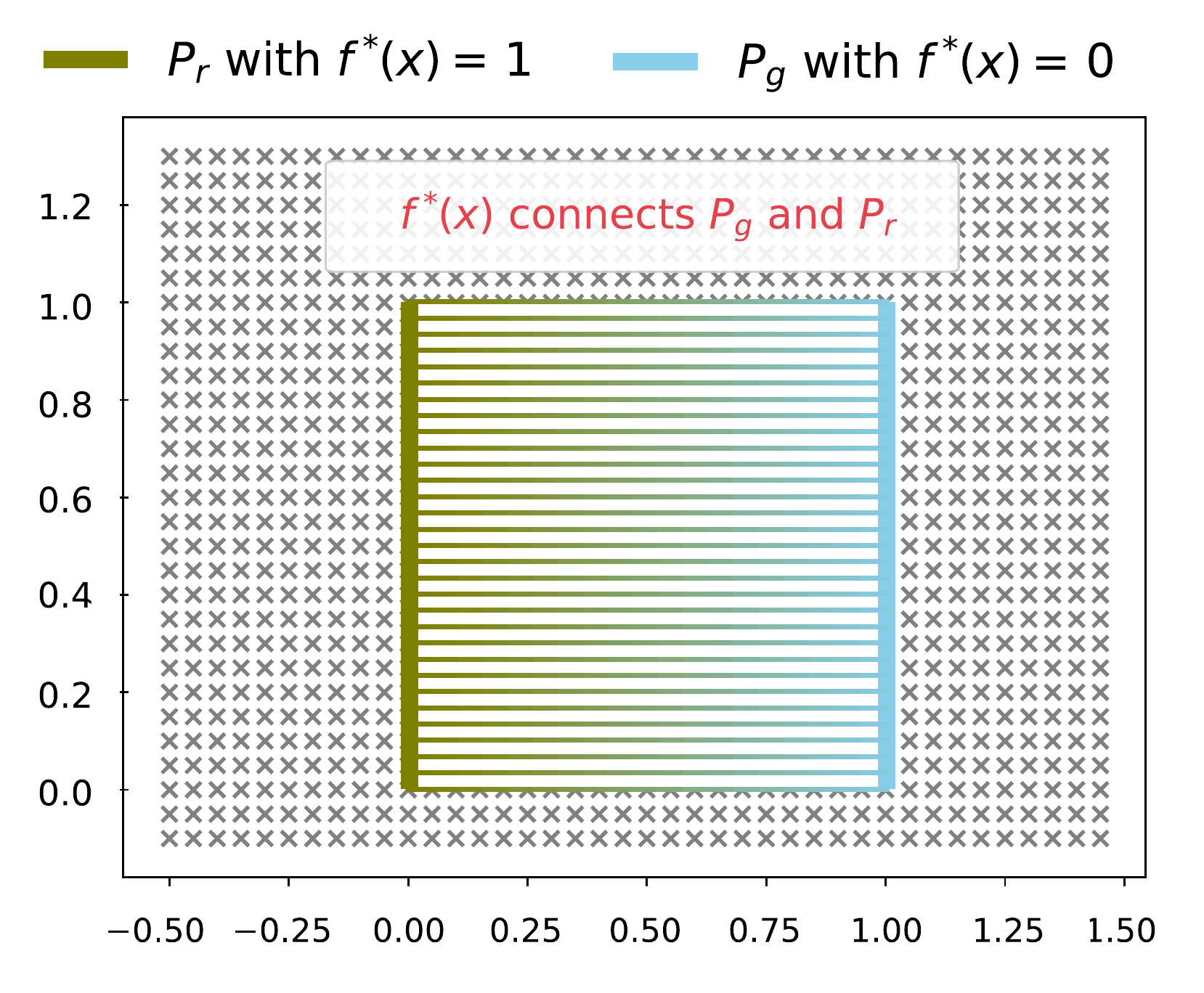}
    \vspace{-3pt}
    \caption{Lipschitz-continuity condition}
    \label{fig1_lipschitz}
\end{subfigure}
\vspace{-5pt}
\caption{
%\weinan{please make the legend font larger.}
In traditional GANs, $\ff(x)$ is only defined on the supports of $P_g$ and $P_r$ and its values do not reflect any information about the locations of other points in $P_g$ and $P_r$. Therefore, they have no guarantee on the convergence. Wasserstein distance in a compacted dual form suffers from the same problem. GAN objectives with Lipschitz-continuity constraint build connection between $P_g$ and $P_r$ where $\nabla_{\!x} \ff(x)$ pulls $P_g$ towards $P_r$.}
\label{figure1}
\vspace{-0pt}
\end{figure}

\subsubsection{The original GAN and Least-Squares GAN}

In the simplest case of Eq.~(\ref{eq_gan_formulation}), e.g., the original GAN \citep{gan} and Least-Squares GAN \citep{lsgan}, there is no restriction on $\mathcal{F}$. Therefore, $\ff(x)$ for each point $x$ is independent of other points, and we have
\begin{equation} \label{ff}
\begin{aligned}
\ff(x) = \argmin_{f(x) \in \mathbb R} \ P_g(x) \cdot \phi(f(x)) + P_r(x) \cdot \varphi(f(x)), \forall x.
\end{aligned} 
\end{equation}
Since we assume supports of $P_g$ and $P_r$ are disjoint, we further have 
\begin{equation} \label{ff2}
\begin{aligned}
\ff(x) = 
\begin{cases}
\argmin_{f(x) \in \mathbb R} \ P_g(x) \cdot \phi(f(x)), \forall x \tsim P_g, \\[4pt]
\argmin_{f(x) \in \mathbb R} \ P_r(x) \cdot \varphi(f(x)), \forall x \tsim P_r.
\end{cases}
\end{aligned} 
\end{equation}
For $x \sim P_g$, the value of $\ff(x)$ is irrelevant to $P_r$. Since $P_g$ and $P_r$ are disjoint\footnote{Here and later, ``two distributions are disjoint'' means that their supports  are disjoint.}, $\nabla_{\!x} \ff(x)$ for $x \tsim P_g$ also tells nothing about $P_r$. In consequence, the generator can hardly learn useful information and is not guaranteed to converge to the case where $P_g=P_r$.

\subsubsection{The Fisher GAN}

\cite{sobolevgan} prove that the optimal $\ff$ of $\mu$-Fisher IPM
%\weinan{what is IPM short for?}
$\mathcal{F}_{\mu}(P_r, P_g)$, the objective used in Fisher GAN \citep{fishergan}, has the following form
\begin{equation}
\ff(x)= \frac{1}{\mathcal{F}_{\mu}(P_r, P_g)}\frac{P_r(x)-P_g(x)}{\mu(x)}.
\end{equation}
%\weinan{what is $\mu(x)$?}
where $\mu$ is a distribution whose support covers $P_r$ and $P_g$. 
Given $P_r$ and $P_g$ are disjoint, we have
\begin{equation} \label{fisherff}
\begin{aligned}
\ff(x) = 
\begin{cases}
\frac{1}{\mathcal{F}_{\mu}(P_r, P_g)}\frac{-P_g(x)}{\mu(x)},& \forall x \tsim P_g; \\[4pt]
\frac{1}{\mathcal{F}_{\mu}(P_r, P_g)}\frac{P_r(x)}{\mu(x)},& \forall x \tsim P_r; \\[4pt]
0,&  \text{otherwise}.
\end{cases}
\end{aligned} 
\end{equation}
Note that the scalar $\frac{1}{\mathcal{F}_{\mu}(P_r, P_g)}$ is a constant. Eq.~(\ref{fisherff}) also defines $\ff(x)$ on $P_g$ and $P_r$ independently. Therefore, for $x \sim P_g$, $\ff(x)$ and $\nabla_{\!x} \ff(x)$ tell nothing about $P_r$.

\subsection{Connection to gradient vanishing}

The non-convergence problem of the original GAN has once been considered as the gradient vanishing problem. In \citep{gan}, it is addressed by using an alternative objective for the generator. However, it actually only changes the scalar $\nabla_{\!f(x)} \psi(f(x))$ while the aforementioned problem in $\nabla_{\!x} \ff(x)$ still exists. The least-squares GAN \citep{lsgan} is proposed to address the gradient vanishing problem, but it also focuses on $\nabla_{\!f(x)} \psi(f(x))$ basically. As we have discussed, the least-squares GAN also belongs to traditional GANs, which is not guaranteed to converge when $P_g$ and $P_r$ are disjoint. 

\cite{wgan} provided a new perspective on understanding the gradient vanishing problem. They argued that gradient vanishing stems from the ill-behaving of traditional metrics, i.e., the distance between $P_g$ and $P_r$ remains constant when they are disjoint. Wasserstein distance is thus proposed as an alternative metric, which can properly measure the distance between two distributions no matter they are disjoint or not. However, as we will show next, Wasserstein distance may also suffer from the same problem on $\nabla_{\!x} \ff(x)$, if a more compact dual form is used. % as traditional GANs

In summary, \textbf{gradient vanishing is about the scalar term $\nabla_{\!f(x)} \psi(f(x))$ in $\nabla_{\!x} \gloss(x)$ or the overall scale of $\nabla_{\!x} \gloss(x)$, and in this paper we investigate its direction $\nabla_{\!x} \ff(x)$, where the problem is more fundamental and challenging}. We will next show that Wasserstein distance, which can properly measure the distance for disjoint distributions, may also suffer from the same issue. 
 
\subsection{Wasserstein distance in compact dual form suffers from the same problem} \label{sec_w_also_suffer}

The ($1^{st}$-)Wasserstein distance is a distance function defined between two probability distributions:
\begin{equation}\label{eq_w_primal}
W_1(P_r,P_g) =  \inf_{\pi \in \Pi(P_r,P_g)} \, \E_{(x,y) \sim \pi} \, [d(x, y)],
\end{equation}
where $\Pi(P_r, P_g)$ denotes the collection of all probability measures with marginals $P_r$ and $P_g$ on the first and second factors, respectively. Since solving it in the primal form (Eq.~(\ref{eq_w_primal})) is burdensome, Wasserstein distance is usually solved in its dual form. Though Wasserstein distance in its dual form is usually written with Lipschitz-continuity condition, we here provide a more compact version. The proof of this dual form can be found in Appendix~\ref{app_dual_form}.
\begin{equation}
\begin{aligned}
W_1(P_r,P_g) = {\sup}_{\mathsmaller{f}} \,\, \E_{x \sim P_r} \, [f(x)] - \E_{x \sim P_g} \, [f(x)], \, \\
\emph{s.t.} \, f(x) - f(y) \leq d(x, y), \,\, \forall x \sim P_r, \forall y \sim P_g.
\end{aligned}
\label{eq_w_dual_form_1}
\end{equation}
We leave the detailed discussion on the relationship between Lipschitz-continuity condition and Wasserstein distance in Section~\ref{lip_stronger}. In Eq.~(\ref{eq_w_dual_form_1}), we replace the strong Lipschitz-continuity condition with a looser constraint. Note that Eq.~(\ref{eq_w_primal}) and Eq.~(\ref{eq_w_dual_form_1}) are still equivalent.
In the next, we will demonstrate that a well-defined distance metric, e.g., Wasserstein distance in this compacted dual form, may also suffer from the same problem in $\nabla_{\!x} \ff(x)$ and does not necessarily ensure the convergence of GANs. % which can also properly measure the distance between disjoint distributions, 

We now study the optimal discriminative function $\ff(x)$ of Wasserstein distance in this dual form. Since there is generally no closed-form solution for $\ff(x)$ in Eq.~(\ref{eq_w_dual_form_1}), we use an illustrative example for demonstration here, but the conclusion is general. 
Let $Z \tsim U[0,1]$ be a uniform variable on interval $[0, 1]$, $P_g$ be the distribution of $(1, Z) \in \mathbb{R}^2$, and $P_r$ be the distribution of $(0, Z) \in \mathbb{R}^2$, as shown in Figure \ref{figure1}. According to Eq.~(\ref{eq_w_dual_form_1}), one of the optimal $\ff$ is as follows
\begin{equation}
\ff(x)=
\begin{cases}
\begin{aligned}
&0 \,\,\,\,\,\, &&\forall x \sim P_g, \\
&1 \,\,\,\,\,\, &&\forall x \sim P_r. \\
\end{aligned}
\end{cases}
\end{equation}
Though having the constraint ``$f(x) - f(y) \leq d(x, y), \, \forall x \sim P_r, \forall y \sim P_g$'', Wasserstein distance in this dual form also only defines the value of $\ff(x)$ on the supports of $P_g$ and $P_r$, and the values of $\ff(x)$ on $P_g$ contain no useful information about the location of $P_r$. Therefore, if $P_g$ and $P_r$ are disjoint, $\nabla_{\!x} \ff(x)$ hardly provides useful information to the generator about how to change $P_g$ into $P_r$ and the generator is not guaranteed to converge to the case $P_g=P_r$.
It is worth noticing that %in Wasserstein distance, 
the value of $\ff(x)$ on the supports of $P_g$ and $P_r$ is sufficient to evaluate the Wasserstein distance. % (see Figure \ref{fig1_wdistance}).

\subsection{A well-defined metric does not necessarily guarantee the convergence} \label{sec_not_enough}

The objectives of GANs are usually defined as (or proved equivalent to) minimizing a distance metric between $P_g$ and $P_r$, which implies that $P_g=P_r$ is the unique global optimum and is in accordance with the final goal of the generative model, i.e., estimating the distribution of real samples. However, in this section, we emphasize that a well-defined (e.g., smooth, continuous, with $P_g=P_r$ being the optimum) distance metric does not necessarily guarantee the convergence of GANs. 

Given an objective is convex with respect to $P_g$ and holds the property that $P_g=P_r$ is the unique optimum, the convergence of GANs is guaranteed if only it directly optimizes $P_g$. 
However, directly optimizing the distribution $P_g$ is usually unfeasible and the practice is optimizing the generated samples according to $\nabla_{\!x} \ff(x)$. As shown in previous sections, when $P_g$ and $P_r$ are disjoint, $\nabla_{\!x} \ff(x)$, the direction that the generator follows for updating the generated samples, tells nothing about how to pull $P_g$ to $P_r$. Therefore, the convergence of GANs are not necessarily guaranteed. 

It is worth noticing that $\nabla_{\!x} \ff(x)$ indeed indicates the direction of decreasing the objective in terms of the current $\ff(x)$, but updating $x$ to make the value of $\ff(x)$ increase/decrease does not necessarily imply that $P_g$ is getting closer to $P_r$. Recall that in the failure case of Wasserstein distance dual form in the above section, the values of $\ff(x)$ on $P_g$ is $0$, while the values of $\ff(x)$ around $P_g$ is undefined. % which implies it is meaningless to update $x$ to increase or decrease $\ff(x)$.

In conclusion, a smooth distance metric satisfying $P_g=P_r$ is the optimum does not guarantee the convergence and sample updating according to $\nabla_{\!x} \ff(x)$ does not necessarily decrease the distance between $P_g$ and $P_r$. Therefore, if we use $\nabla_{\!x} \ff(x)$ for the update of the generator\footnote{Alternative strategies actually exist, for example, \cite{swgan} use the optimal transport plan \citep{largescaleotmap} between $P_g$ and $P_r$ to update the generator. %However, their results tend to be blurry.
}, it is necessary to make $\nabla_{\!x} \ff(x)$ aware of how to pull $P_g$ to $P_r$. 
In the next section, we will introduce the Lipschitz-continuity condition as a general solution for making $\nabla_{\!x} \ff(x)$ well-behaving and guaranteeing the convergence of $\nabla_{\!x} \ff(x)$-based GANs.

\section{A General Solution: Lipschitz-continuity Condition} \label{sec_lip}

Lipschitz-continuity condition becomes popular in GANs recently as part of the discriminator's objective %The first attempt of introducing Lipschitz-continuity condition in GANs is the Wasserstein GAN \citep{wgan}, and later Lipschitz-continuity condition is extended to other objectives such as 
\citep{wgan,kodali2017convergence,fedus2017many,sngan}, achieving great success.
%Lipschitz-continuity constraint becomes popular in GANs recently. The first attempt of introducing Lipschitz-continuity condition in GANs is the Wasserstein GAN \citep{wgan}, and later Lipschitz-continuity condition is extended to other objectives such as \citep{kodali2017convergence,fedus2017many,sngan}, achieving great success.
%
In this section, we explain the significance of Lipschitz-continuity condition when introduced into the objective of the discriminator. 

In a nutshell, under a board family of GAN objectives, Lipschitz-continuity condition is able to connect $P_g$ and $P_r$ through $\ff(x)$ such that when $P_r$ and $P_g$ are disjoint, $\nabla_{\!x} \ff(x)$ for each generated sample $x \tsim P_g$ will point towards some real sample $y \tsim P_r$, which guarantees the trend that $P_g$ is getting closer to $P_r$ at every step. More detailed results are presented as follows. 
%\weinan{why `point towards some real sample' is good for training a generator? Although you might think it is straightforward, I think it is still necessary to explicitly state the reason.}\zhiming{added explanation.}
%\footnote{In this paper, Lipschitz-continuity condition by default refer to the Lipschitz-continuity condition with respect to Euclidean distance.} %which makes $\nabla_{\!x} \ff(x)$ meaningful for the updating of the generator. if $\ff(x)$ is differentiable, 

\subsection{The main result} \label{sec_main_theorem}

A function $f: X \rightarrow Y$ is $k$-Lipschitz continuous if it satisfies the following property:
\begin{equation} \label{eq_lip}
\begin{aligned}
d_Y(f(x), f(y)) \leq k \cdot d_X(x, y), \forall \; x, y \in X,
\end{aligned} 
\end{equation}
where $d_X$ and $d_Y$ are distances metrics in domains $X$ and $Y$, respectively. The smallest constant $k$ is called the Lipschitz constant of function $f$. In this paper (and most GAN papers), $d_X$ and $d_Y$ are defined as Euclidean distance.\footnote{Actually, we argue that the distance metrics must be Euclidean distance in GANs. See Appendix~\ref{lip_direction}.} We let $\lVert y \!-\! x \rVert$ denote Euclidean distance. 
%instead of $d(x,y)$. 

As proved by \cite{wgangp}, when the Lipschitz-continuity condition is combined with Wasserstein distance, we have the following property if $\ff(x)$ is differentiable, then
\begin{equation} \label{lipgrad}
\begin{aligned}
{\rm Pr} \left(\nabla_{\!x} \ff(x_t) = \frac{y-x}{\lVert y - x \rVert}\right) = 1, \ {\rm for} \ (x, y)\sim \pi^*,
\end{aligned}
\end{equation}
where $x_t=t x + (1-t) y$, $0 \leq t \leq 1$, and $\pi^*$ is the optimal $\pi$ in Eq.~(\ref{eq_w_primal}). 
The meaning of this proposition is two-fold: (\rnum{1}) for each $x \tsim P_g$, there exists a $y\tsim P_r$ such that $\nabla_{\!x} \ff(x_t)=\frac{y-x}{\lVert y - x \rVert}$ for all linear interpolations $x_t$ between $x$ and $y$; (\rnum{2}) these $(x, y)$ pairs match the optimal coupling $\pi^*$. 

Next we introduce our theorem on the Lipschitz-continuity condition. It turns out when combining the Lipschitz-continuity condition with generalized objectives, Property-(\rnum{1}) still holds and Property-(\rnum{2}) is naturally dismissed as it is now not restricted to Wasserstein distance. 

\vspace{4pt}
\begin{theorem} \label{theorem} Let $\dloss\triangleq\E_{x \sim P_g} [ \phi(f(x)) ] + \E_{x \sim P_r} [ \varphi(f(x)) ]$ and $\partial_x\dloss$ denotes $P_g(x) \phi(f(x)) + P_r(x) \varphi(f(x))$. Let $\bar{P_r}$ and $\bar{P_g}$ denote the supports of $P_r$ and $P_g$, respectively. Assume $\ff = \argmin_f~\![\dloss + \lambda \cdot k(f)^2]$, where $k(f)$ is the Lipschitz constant of $f$. If $\phi(x)$ and $\varphi(x)$ in $\dloss$ satisfy
\vspace{-6pt}
\begin{align}
\begin{cases} \label{eq_solvable}
\phi'(x) > 0, \phi''(x) \geq 0, \\[3pt]
\varphi'(x) < 0, \varphi''(x) \geq 0, \\[3pt]
\exists \, a, \, \phi'(a) + \varphi'(a) = 0,
\vspace{-3pt}
\end{cases}
\end{align}
%\lantao{The third line in Eq.~12, should be $\phi'(a) + \varphi'(a)$?}
then we have that
\vspace{-3pt}
\begin{enumerate}[label=(\alph*),leftmargin=18pt]
\setlength\itemsep{0.0em}
\item $\forall x \in \bar{P_g} \cup \bar{P_r}$, $\exists y_{\neq x} \in \bar{P_g}\cup \bar{P_r}$ such that $|\ff(y)\!-\!\ff(x)|=k(\ff) \cdot \lVert x \!-\! y \rVert $ or $\nabla_{\!\ff(x)} \partial_x\dloss = 0$;
\item $\forall x \in \bar{P_g} \cup \bar{P_r} - \bar{P_g} \cap \bar{P_r}$, $\exists y_{\neq x} \in \bar{P_g}\cup \bar{P_r}$ such that $|\ff(y)\!-\!\ff(x)|=k(\ff) \cdot \lVert x \!-\! y \rVert $;
\item if $\bar{P_g}=\bar{P_r}$ and $ P_g\neq P_r$, then $\exists x$, $\exists y_{\neq x}$ such that $|\ff(y)\!-\!\ff(x)|=k(\ff) \cdot \lVert x \!-\! y \rVert $;
\item the only Nash Equilibrium of $\dloss + \lambda \cdot k(f)^2$ is reached when $ P_g= P_r$, where $k(f)=0$.
\end{enumerate}
\end{theorem}

The above theorem states that when the Lipschitz-continuity condition is combined with an objective that satisfies Eq.~(\ref{eq_solvable}), then: (a) for the optimal discriminative function $\ff(x)$ at any point $x\in \bar{P_g} \cup \bar{P_r}$, it either is bounded by the Lipschitz constant or $\partial_x\dloss$ holds a zero-gradient with respect to $\ff(x)$; (b) for any point that only appears in $\bar{P_g}$ or $\bar{P_r}$, there must exist a point that bounds this point in terms of $|\ff(y)\!-\!\ff(x)| = k(\ff) \cdot \lVert x\! -\! y \rVert $, because for these points, $\partial_x\dloss$ will never get zero gradient with respect to $\ff(x)$ as we prove in the Appendix~\ref{app_proof}; (c) when $P_g$ and $P_r$ are totally overlapped, as long as $P_g$ still not converges to $P_r$, there exists at least one pair $(x, y)$ that bounds each other; %\weinan{what is `bounds each other'? shall we say Lipschitz-bounds?}\zhiming{bounded means $|\ff(y)\!-\!\ff(x)| = k(\ff) \cdot \lVert x\! -\! y \rVert$, it has been mentioned in (b), but maybe we should make it clear. I just think it might be too long} 
(d) the only Nash Equilibrium among $P_g$ and $\ff(x)$ under this objective is ``$P_g=P_r$ with $k(\ff)=0$''. The formal proof is in Appendix~\ref{app_proof}. 

Wasserstein distance, i.e., $\phi(x)=\varphi(-x)=x$ is one instance that satisfies Eq.~(\ref{eq_solvable}); and it is a very special case, which holds $\phi''(x)= 0$ and $\varphi''(x)= 0$. Eq.~(\ref{eq_solvable}) is actually quite general and there exists many other settings, e.g., $\phi(x)=\varphi(-x)=-\log(\sigma(-x))$, $\phi(x)=\varphi(-x)=x+\sqrt{x^2+1}$ and $\phi(x)=\varphi(-x)=\exp(x)$. Generally, it is feasible to set $\phi(x)=\varphi(-x)$. As such, to build a new objective, one only needs to find a function that is increasing and has non-decreasing derivative. See Figure~\ref{fig_function_curve}. In addition, all linear combinations of feasible $(\phi, \varphi)$ pairs also lie in the family. %We empirically study these objectives in Section~\ref{sec_exp}.

It is worth noting that $k(f)$ is also optimized here and it is actually necessary for Property-(c) and Property-(d). This is the key difference when the Lipschitz-continuity condition is extended to general objectives. The underlying reason for the need of also minimizing $k(f)$ comes from the existence of case ``$\nabla_{\!\ff(x)}\partial_x\dloss=0$ for  $P_g(x)\neq P_r(x)$'', which does not hold when the objective is Wasserstein distance. Minimizing $k(f)$ guarantees that the only Nash Equilibrium is ``$P_g=P_r$ with $k(\ff)=0$''. On the other hand, if $k(f)$ is not minimized towards zero, Wasserstein distance dual form based GANs are not guaranteed to have zero gradient $\nabla_{\!x}\ff(x)$ at the convergence state $P_g=P_r$. It indicates that minimizing $k(f)$ is also beneficial to the Wasserstein GAN~\citep{wgan}.

\subsection{Lipschitz-continuity connects $P_g$ and $P_r$ through $\ff(x)$} \label{sec_connections}

From Theorem \ref{theorem}, we know that for any point $x$, as long as $\partial_x\dloss$ does not hold a zero gradient with respect to $\ff(x)$, $\ff(x)$ %\weinan{what does `it' means here? the gradient?}\zhiming{it means $\ff(x)$} 
must be bounded by another point $y$ such that $|\ff(y)\!-\!\ff(x)|=k(\ff) \cdot \lVert x \!-\! y \rVert $. We here further clarify that, when there is a bounding relationship, it must involve both real sample(s) and fake sample(s). More formally, we have

\begin{theorem} \label{corollary}
%\weinan{the index of this corollary is weird}\zhiming{do you mean the 3.2.1? could you help handle it, this my best trial. it looks always weird.}
If $\ff = \argmin_{f} [\dloss + \lambda \cdot k(f)^2]$, then
\vspace{-3pt}
\begin{itemize}[leftmargin=20pt]
\item $\forall x \in \bar{P_g}$, if $\;\exists z_{\neq x} \in \bar{P_g}\cup \bar{P_r}$ such that $|\ff(x)\!-\!\ff(z)|=k(\ff)  \cdot \lVert x \!-\! z \rVert $, then $\exists y_{\neq x} \in \bar{P_r}$ such that $\ff(y)\!-\!\ff(x)=k(\ff) \cdot \lVert x \!-\! y \rVert $, 
\item $\forall y \in \bar{P_r}$, if $\;\exists z_{\neq y} \in \bar{P_g}\cup \bar{P_r}$ such that $|\ff(z)\!-\!\ff(y)|=k(\ff)  \cdot \lVert z \!-\! y \rVert $, then $\exists x_{\neq y} \in \bar{P_g}$ such that $\ff(y)\!-\!\ff(x)=k(\ff)  \cdot \lVert x \!-\! y \rVert $. 
\end{itemize}
\end{theorem}

The intuition behind the above theorem is that samples from the same distribution, e.g., the fake samples, will not bound each other. It is worth noticing that there might exist a chain of bounding relationships that involves a dozen of fake samples and real samples, and these points all lie in the same line and bounds each other.

Under the Lipschitz-continuity condition, the \textbf{bounded line} in the value surface of $\ff$ is the basic \textbf{building block} that connects $P_g$ and $P_r$, and each fake sample lies in one of the bounded lines. Next we will further interpret the implication of bounding relationship and show that it guarantees meaningful $\nabla_{\!x} \ff(x)$ for all involved points.

\subsection{Lipschitz-continuity ensures the convergence of $\nabla_{\!x} \ff(x)$-based GANs} \label{sec_convergence}

Recall that the proposition in Eq.~(\ref{lipgrad}) states that $\nabla_{\!x} \ff(x_t)=\frac{y-x}{\lVert y - x \rVert}$. We next show that it is actually a direct consequence of bounding relationship between $x$ and $y$. We formally state it as follows:
\begin{theorem} \label{lemma}
\vspace{2pt}
Assume $f(x)$ is  differentiable and $k$-Lipschitz continuous. For all $x$ and $y$ which satisfy $x\neq y$ and $f(y)-f(x)=k \cdot \lVert x - y \rVert$, we have $\nabla_{\!x} f(x_t)=k \cdot \frac{y-x}{\rVert y-x \rVert}$, where $x_t=tx+(1-t)y$ for $0 \leq t \leq 1$. 
\vspace{-5pt}
\end{theorem}
In other words, if two points $x$ and $y$ bound each other in terms of $f(y)\!-\!f(x)\!=\!k \cdot \lVert x\! -\! y \rVert$, there is a straight line between $x$ and $y$ in the value surface of $f$. Any point in this line holds the maximum gradient slope $k$, and the direction of these gradient all point towards the $x\! \rightarrow\! y$ direction. % %\lantao{no link}
Combining Theorem~\ref{theorem} and Theorem~\ref{corollary}, we can conclude that when $P_g$ and $P_r$ are disjoint, $\nabla_{\!x} \ff(x)$ for each sample $x \tsim P_g$ points to a sample $y \tsim P_r$, which guarantees that $P_g$ is moving towards $P_r$. 

In fact, Theorem~\ref{theorem} provides further guarantee on the convergence. Property-(b) implies that for any $x \tsim P_g$ that does not lies in $P_r$, $\nabla_{\!x} \ff(x)$  points to some real sample $y \tsim P_r$. In the fully overlapped case, according to Property-(c), unless $P_g=P_r$, there exists a pair $(x, y)$ in bounding relationship and $\nabla_{\!x} \ff(x)$ pulls $x$ towards $y$. Property-(d) guarantees that the only Nash Equilibrium is ``$P_g=P_r$''. The proof of Theorem~\ref{lemma} is provided in Appendix~\ref{lip_direction}.

%Given the only Nash Equilibrium is $P_g=P_r$, Lipschitz-continuity condition under the objectives in Theorem~\ref{theorem} can guarantee the convergence of the corresponding GAN models.
%. One can prove that: if $(x, y)$ is in bounding relationship in $\ff(x)$, then $P_r(y)/P_g(y)>P_r(x)/P_g(x)$ which means that transferring density from $x$ to $y$ makes the two distributions closer. A

\begin{figure}[t]
\vspace{-10pt}
\begin{subfigure}{0.33\linewidth}
    \centering
    \includegraphics[width=0.8\columnwidth]{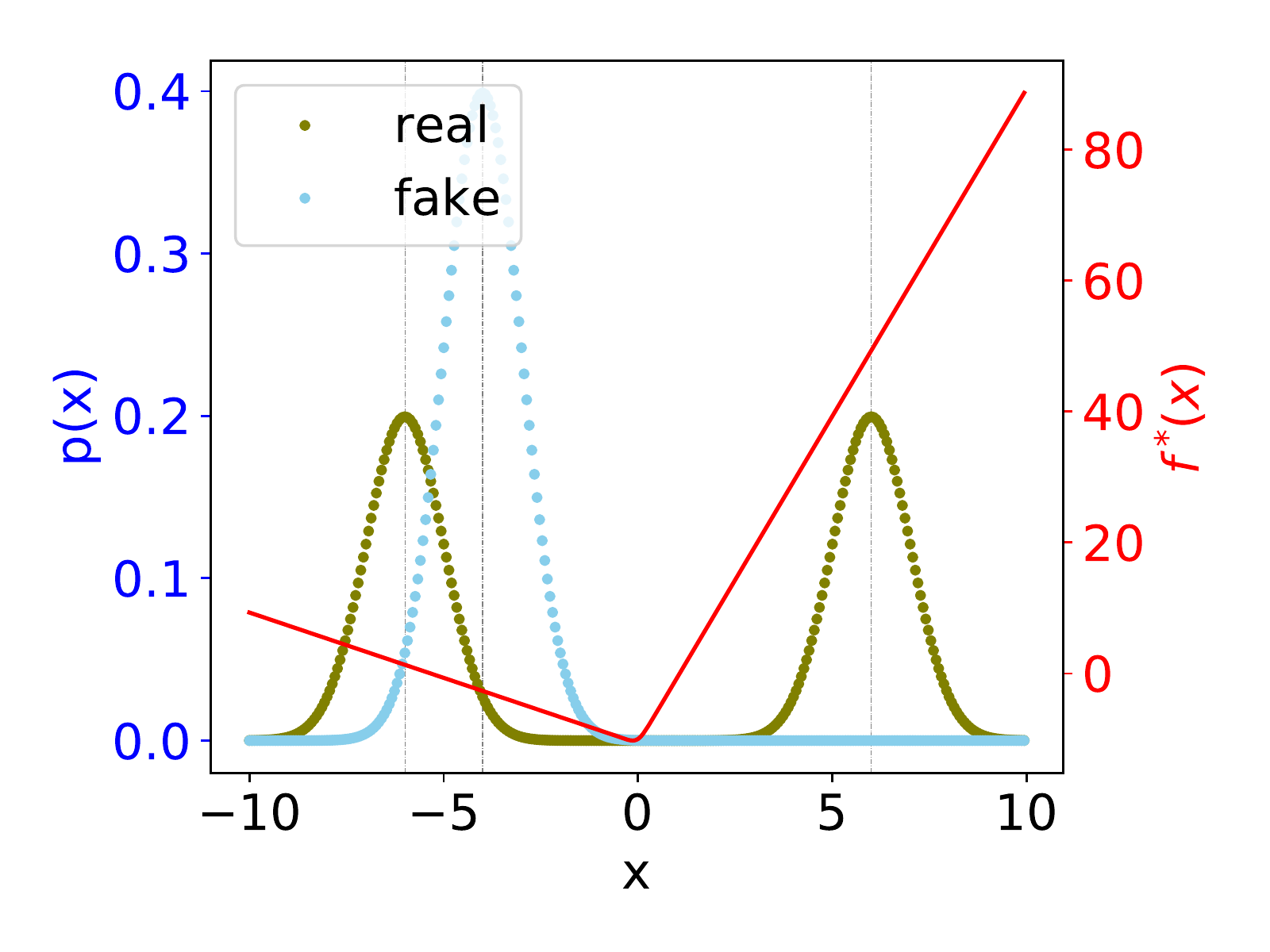}
    \vspace{-8pt}
    \caption{Original GAN}
    \label{fig2_originalgan}
\end{subfigure}	
\begin{subfigure}{0.33\linewidth}
    \centering
    \includegraphics[width=0.8\columnwidth]{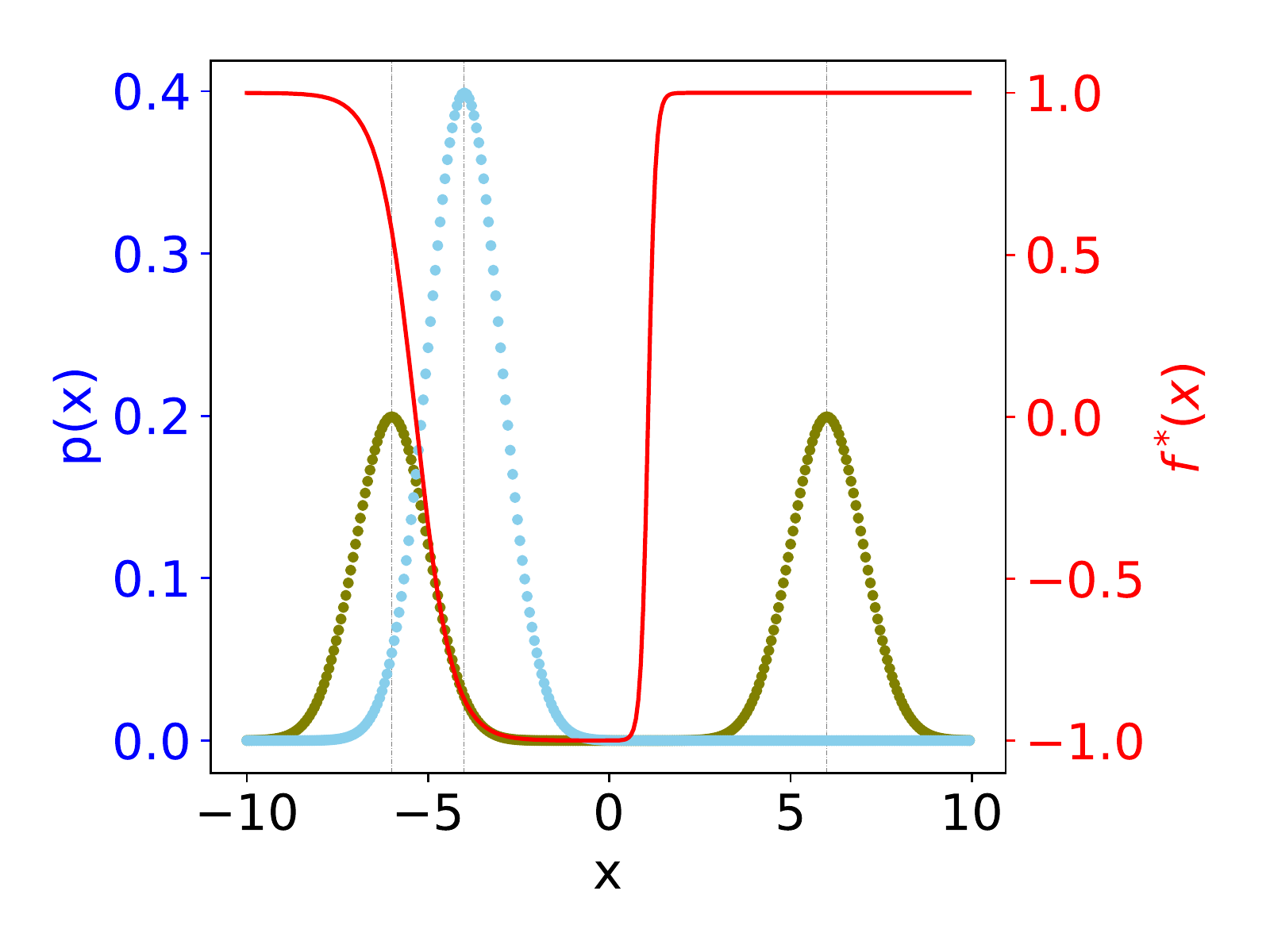}
    \vspace{-8pt}
    \caption{Least Square GAN}
    \label{fig2_lsgan}
\end{subfigure}
\begin{subfigure}{0.33\linewidth}
    \centering
    \includegraphics[width=0.8\columnwidth]{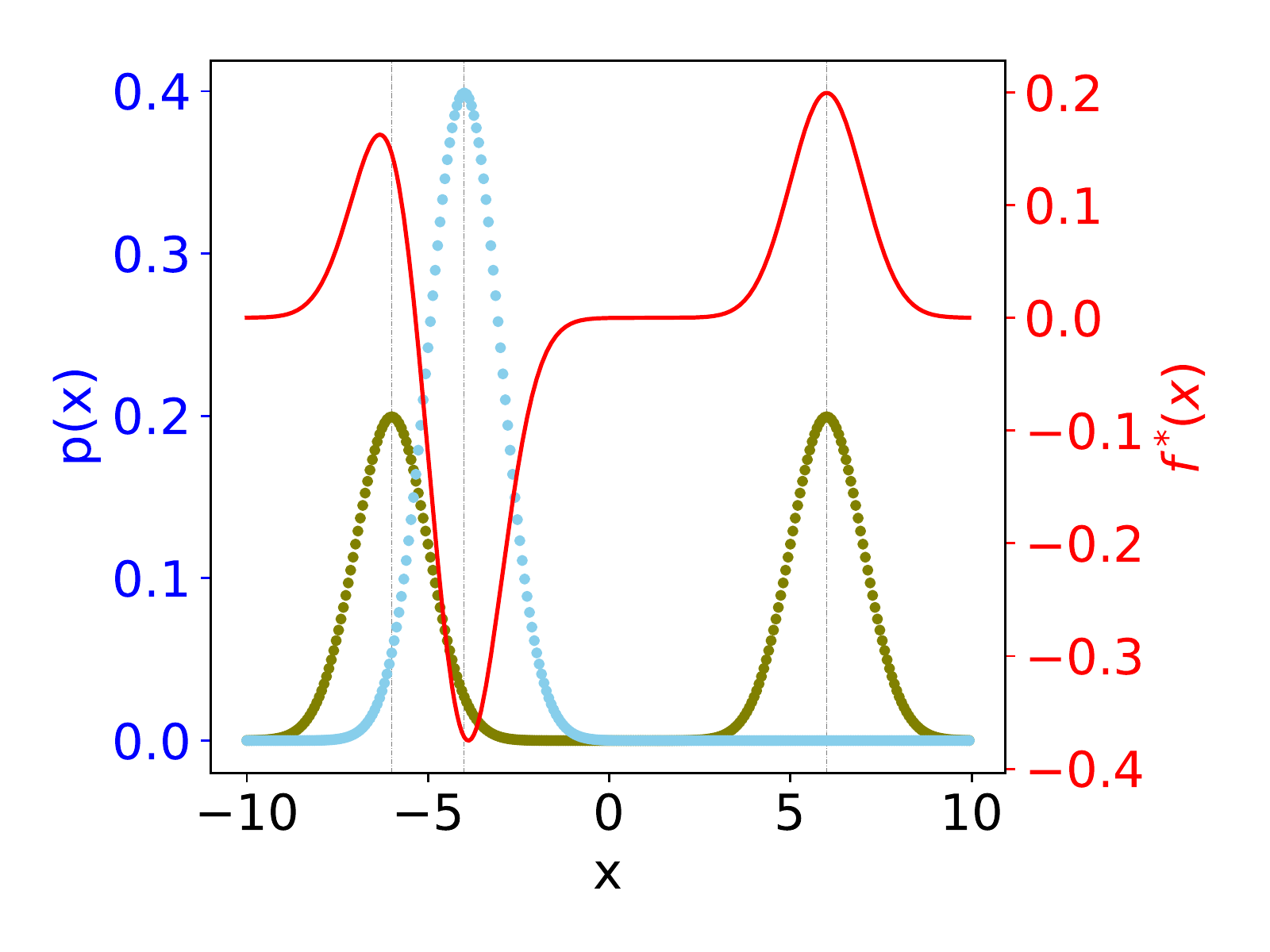}
    \vspace{-8pt}
    \caption{Fisher GAN with uniform $\mu$}
    \label{fig2_fishergan}
\end{subfigure}
\vspace{-5pt}
\caption{The source of Mode Collapse. In traditional GANs, $\ff(x)$ is a function of the local densities $P_g(x)$ and $P_r(x)$. Given $\ff(x)$ is an increasing function of $P_r(x)$ and decreasing function of $P_g(x)$, when fake samples get close to a mode of the $P_r$, $\nabla_{\!x} \ff(x)$ move them towards the mode.} % And being inaccessible to the global status, it sticks there.}, and  is also local
\label{figure2}
\vspace{-5pt}
\end{figure}

\section{Overlapping case: the cause of mode collapse} \label{sec_overlapped}
%Towards Comprehensive Understanding}
%\weinan{this section title should be refined.}\zhiming{indeed, how is it now}

%The Problem of $\nabla_{\!x} \ff(x)$ When $P_g$ and $P_r$ Are Overlapped} 
% \subsection{From disjoint case to overlapping case} 

In Section~\ref{sec_failue_cause}, we discuss the problem of $\ff(x)$ and $\nabla_{\!x} \ff(x)$ in the case where $P_g$ and $P_r$ are disjoint. In this section, we extend our discussion to the overlapping case. 
In the disjoint case, we argue that ``$\ff(x)$ on $P_g$ does not reflect any information about the location of other points in $P_r$'' will lead to an unfeasible $\nabla_{\!x} \ff(x)$ and thus non-convergence. In the overlapping and continuous case, things are actually different, $\ff(x)$ around each point is also defined, and its gradient $\nabla_{\!x} \ff(x)$ now reflects the local variation of $\ff(x)$.

For most traditional GANs, $\ff(x)$ mainly reflects the local information about the density $P_g(x)$ and $P_r(x)$. However, it is worth noting that $\ff(x)$ is usually an increasing function with respect to $P_r(x)$ while a decreasing function with respect to $P_g(x)$. For instance, $\ff(x)$ in the original GAN is $\log {P_r(x)}/{P_g(x)}$. 
Optimizing the generator according $\nabla_{\!x} \ff(x)$ will move sample $x$ towards the direction of increasing $\ff(x)$. Because $\ff(x)$ positively correlates with $P_r(x)$ and negatively correlated with $P_g(x)$, it in sense means $x$ is becoming more real. %\weinan{please refine this sentence. I cannot understand.}\zhiming{refined. it is not strict. actually ``$\ff(x)$ increases'' means increasing of $\log \frac{P_r(x)}{P_g(x)}$ and $\frac{\alpha \cdot P_g(x)+\beta \cdot P_r(x)}{P_g(x)+P_r(x)}$ in the original GAN and Least-Squares GAN respectively.}
However, such a local greedy strategy turns out to be a fundamental cause of mode collapse.

%as we will next show,  cannot guarantee convergence and
%Fisher GAN are slight different, where $\ff(x)$ is proportional to $(P_r(x)\!-\!P_g(x))/\mu(x)$. It holds an extra $\mu(x)$ inside $\ff(x)$. Depending on the selection of $\mu(x)$, it actually could be even worse. Saying $x$ is in one dimension and $P_r(x)\!-\!P_g(x)$ is increasing while $(P_r(x)\!-\!P_g(x))/\mu(x)$ is decreasing; $\nabla_{\!x} \ff(x)$ by the definition will point towards the increasing direction of $(P_r(x)\!-\!P_g(x))/\mu(x)$, however, it turns out to be the decreasing direction of $P_r(x)\!-\!P_g(x)$. 

% \subsection{The cause of mode collapse: the locality of $\ff(x)$ and $\nabla$} \label{sec_mode_collapse}

Mode collapse is a notorious problem in GANs' training, which refers to the phenomenon that the generator only learns to produce part of $P_r$. Many literatures try to study the source of mode collapse \citep{mode_gan,unrolled_gan,kodali2017convergence,arora2017generalization} and measure the degree of mode collapse \citep{acgan,arora2017gans}.

The most recognized cause of mode collapse is that, if the generator is much stronger than the discriminator, it may learn to only produce the sample(s) in the local or global maximum of $f(x)$ for the current discriminator. This argument is true for most of GAN models. However, from our perspective on $\ff(x)$ and its gradient, there actually exists a much more fundamental cause of mode collapse, i.e., the locality of $\ff(x)$ in traditional GANs and the locality of gradient operator $\nabla$. 

In traditional GANs, $\ff(x)$ is a function of local densities $P_g(x)$ and $P_r(x)$, which is local, and the gradient operator $\nabla$ is also a local operator. As the result, $\nabla_{\!x} \ff(x)$ only reflects its local variations and cannot capture the statistic of $P_r$ and $P_g$ that is far from itself. If $\ff(x)$ in the surrounding area of $x$ is well-defined, $\nabla_{\!x} \ff(x)$ will move $x$ towards the \textbf{nearby} location where the value of $\ff(x)$ is higher. It does not take the global status into account.

The typical result is that when fake samples get close to a mode of the $P_r$, they move towards the mode and get stuck there (due to the locality). Assume $P_r$ consists of two Gaussian distributions (A and B) that are distant from each other, while the current $P_g$ is uniformly distributed over its support and close to real Gaussian A. In this case, $\nabla_{\!x} f(x)$ of all fake samples will point towards the center of Gaussian A. If $P_g$ is a Gaussian with the same standard deviation as Gaussian A, $\nabla_{\!x} f(x)$ in original GAN and Least-Square GAN shows almost identical behaviors, which is illustrated in Figure~\ref{figure2}. In Fisher GAN, if $\mu(x)$ is uniform, the case is even worse: a large amount of points that are relatively far from Gaussian A will move away from A (but the direction is not necessarily towards B, though in our 1-D case it is). %these points who are relatively near to the real Gaussian A will move towards A, but another
%
%\vspace{-1pt}
This observation again supports our argument that ``a well-defined distance metric does not necessarily guarantee the convergence'', and the validity of $\nabla_{\!x} \ff(x)$ is still necessary even if $P_g$ and $P_r$ is continuous and overlapped. %As far as we know, this is the first work that provides a clear explanation on this cause of mode collapse. 

%The gradient problems of $\nabla_{\!x} \ff(x)$ in traditional GANs are actually not limited to these above-mentioned. We have extended discussions on non-vanishing gradient at convergence and Type-\RNum{1} and Type-\RNum{2} gradient vanishing in the Appendix~\ref{}, i.e.~the issues regarding the gradient scale, in the main paper, we more focused on $\nabla_{\!x} \ff(x)$ and its direction.

%\section{Discussion and Contribution Clarification}
%\label{sec_5}

\section{Extended Discussions}

\subsection{The relation between Lipschitz-continuity and Wasserstein distance} \label{lip_stronger}

Most literature presents the dual form of Wasserstein distance with the Lipschitz-continuity condition. However, it is worth noticing that the Lipschitz-continuity condition is actually stronger than the necessary one in the dual form of Wasserstein distance.
Recall that in the dual form of Wasserstein distance, the constraint can be more compactly written as (introduced in Section~\ref{sec_w_also_suffer} and proved in Appendix~\ref{app_dual_form})
\vspace{-3pt}
\begin{equation}
\begin{aligned}
f(x) - f(y) \leq d(x, y), \,\, \forall x \sim P_r, \forall y \sim P_g.
\end{aligned}
\label{eq_w_dual_form_1_st}
\end{equation}
However, it is usually written as 1-Lipschitz continuous, which is
\begin{equation}
\begin{aligned}
f(x) - f(y) \leq d(x, y), \forall x, \forall y.
\end{aligned}
\label{eq_lip_simp}
\end{equation}
The key difference is that the constraint in Eq.~(\ref{eq_w_dual_form_1_st}) restricts the range of $x$ and $y$, but Lipschitz-continuity condition (Eq.~(\ref{eq_lip_simp})) does not have the restriction on the range, thus the latter is the sufficient condition of the former one. 
It is also worth noticing that, though Lipschitz-continuity condition is stronger than the compact one, it does not affect the final solution (Appendix~\ref{app_dual_form}). In other words, Lipschitz-continuity condition is a safe extension of the compact constraint. And if the supports of $P_g$ and $P_r$ are the entire space, Eq.~(\ref{eq_w_dual_form_1_st}) and Eq.~(\ref{eq_lip_simp}) are actually identical; in such condition, Wasserstein distance in its dual form always works. \textbf{However, $P_g$ and $P_r$ are usually disjoint in GANs}. Therefore, \textbf{using the strong Lipschitz-continuity condition is necessary} to ensure the validity of the dual form of Wasserstein distance in $\nabla_{\!x} \ff(x)$-based updating, and the constraint in Eq.~(\ref{eq_w_dual_form_1_st}) is not enough as shown in Section~\ref{sec_w_also_suffer}. 

\subsection{Explanation on the empirical success of traditional GANs} \label{sec_success}

Though traditional GANs does not have any guarantee on its convergence, it has already achieved its great success. The reason is that having no guarantee does not mean it cannot converge. It turns out extensive parameter-tuning actually increases the probability of the convergence.

As shown in Appendix~\ref{hyper_para}, hyper-parameters are important in influencing the value surface of $\ff(x)$. Some typical settings (e.g., simplified neural network architecture, relu or leaky relu activation, relatively high learning rate, Adam optimizer, etc.) tend to form a relatively smooth value surface (e.g., monotonically increasing from $P_g$ to $P_r$), making $\nabla_{\!x} \ff(x)$ much more meaningful. That is, one can find these settings, where $\nabla_{\!x} \ff(x)$ or $\nabla_{\!x} f(x)$ is more favourable, to enable traditional GANs to work. %, i.e., implicitly restricting the function space $\mathcal{F}$
In opposite, we have tried highly-nonlinear activation such as swish \citep{ramachandran2018searching} in the discriminator. It turns out traditional GANs are very likely to fail. In contrast, our proposed Lipschitz-continuity condition based GANs are compatible with highly-nonlinear activation. 
Another important empirical technique is to delicately balance the generator and the discriminator or limit the capacity of the discriminator. This is to avoid the fatal optimal $\ff(x)$. 
All these could possibly make traditional GANs work. \textbf{However, the consequence is that these GANs are very sensitive to hyper-parameters and hard to use.}

\section{Experiments} \label{sec_exp}

In this section, we present the experiment results on our proposed objectives for GANs. The anonymous code is provided at \url{http://bit.ly/2Kvbkje}.

\subsection{Verifying the objective family and its gradient $\nabla_{\!x} \ff(x)$}

\begin{figure}[thbp]
\vspace{-10pt}
\begin{minipage}{.5\textwidth}
\begin{subfigure}{0.49\linewidth}
\includegraphics[width=0.99\columnwidth]{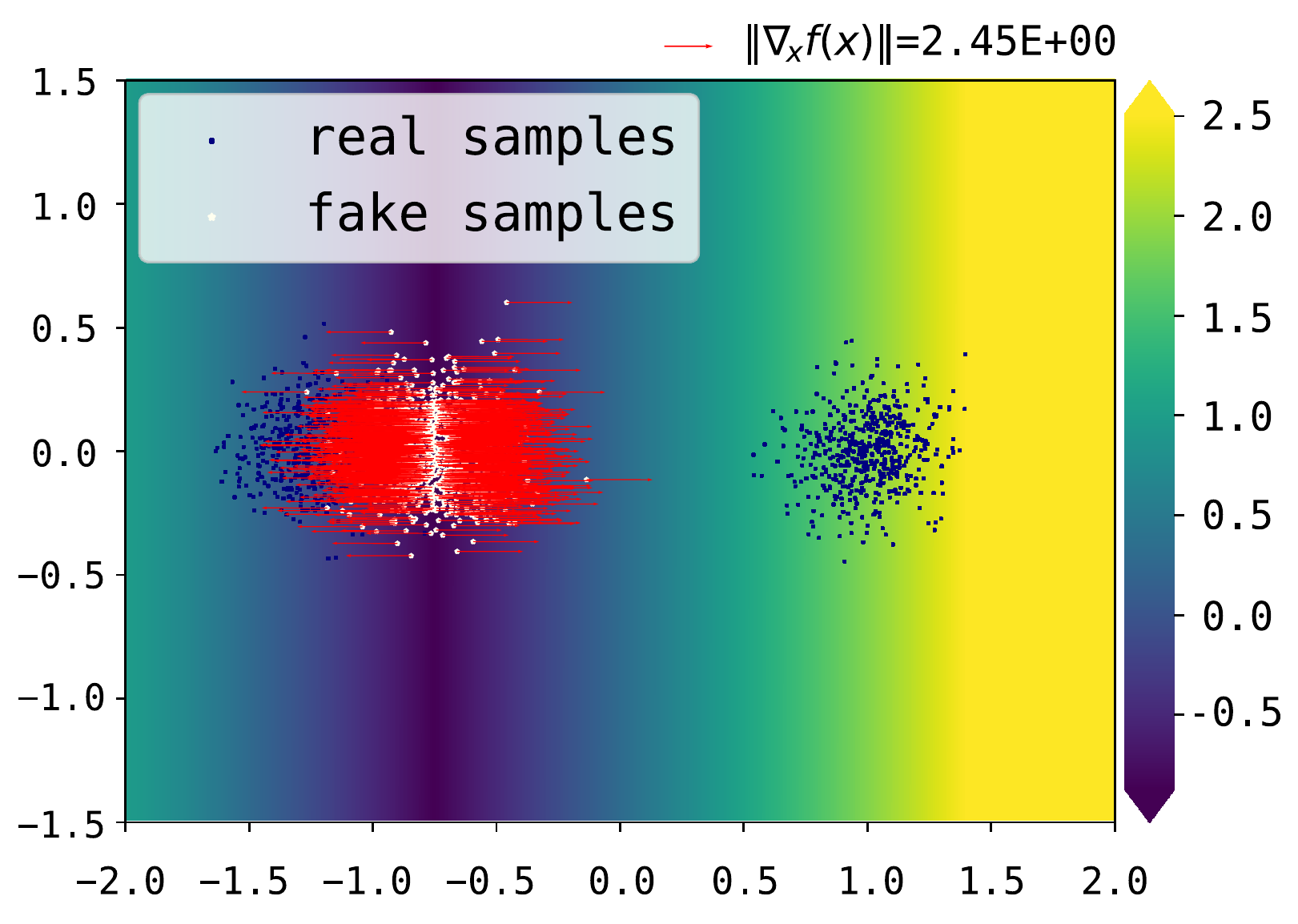}
\vspace{-17pt}
\caption{$x$}
\label{a}
\end{subfigure}	
\begin{subfigure}{0.49\linewidth}
\includegraphics[width=0.99\columnwidth]{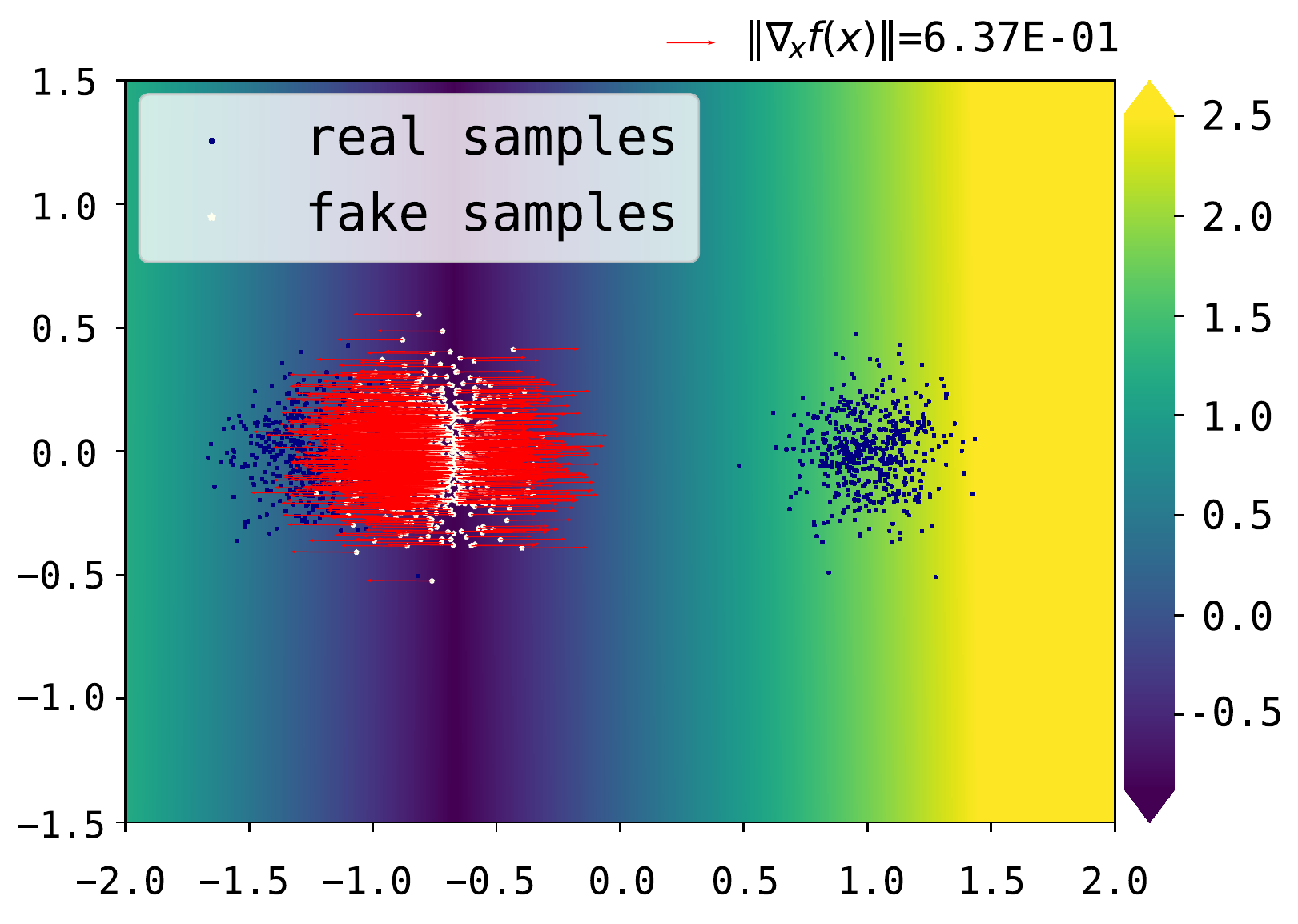}
\vspace{-17pt}
\caption{$-\log(\sigma(-x))$}
\label{b}
\end{subfigure}	
\begin{subfigure}{0.49\linewidth}
\includegraphics[width=0.99\columnwidth]{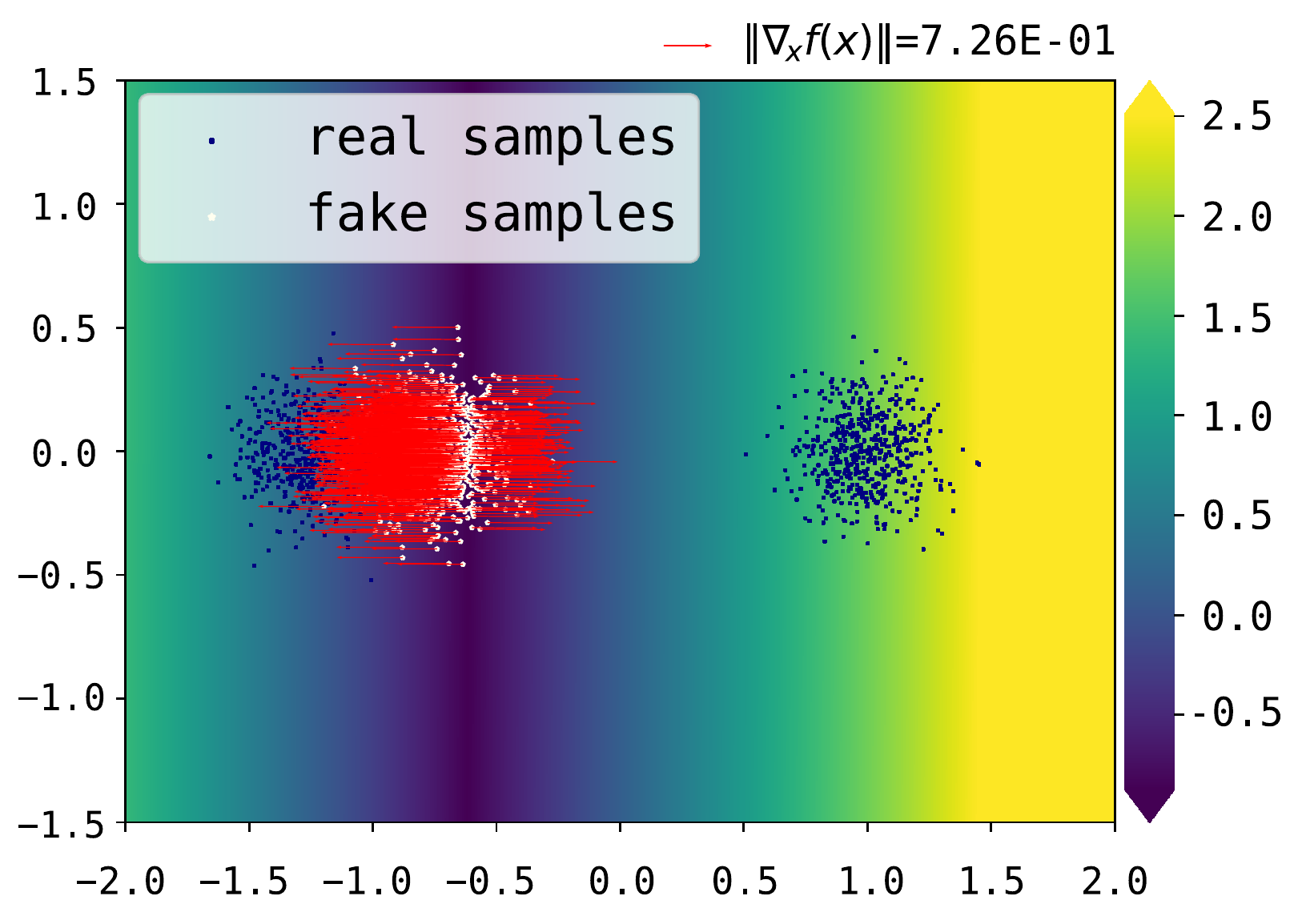}
\vspace{-17pt}
\caption{$x+\sqrt{x^2+1}$}
\label{c}
\end{subfigure}	
\begin{subfigure}{0.49\linewidth}
\includegraphics[width=0.99\columnwidth]{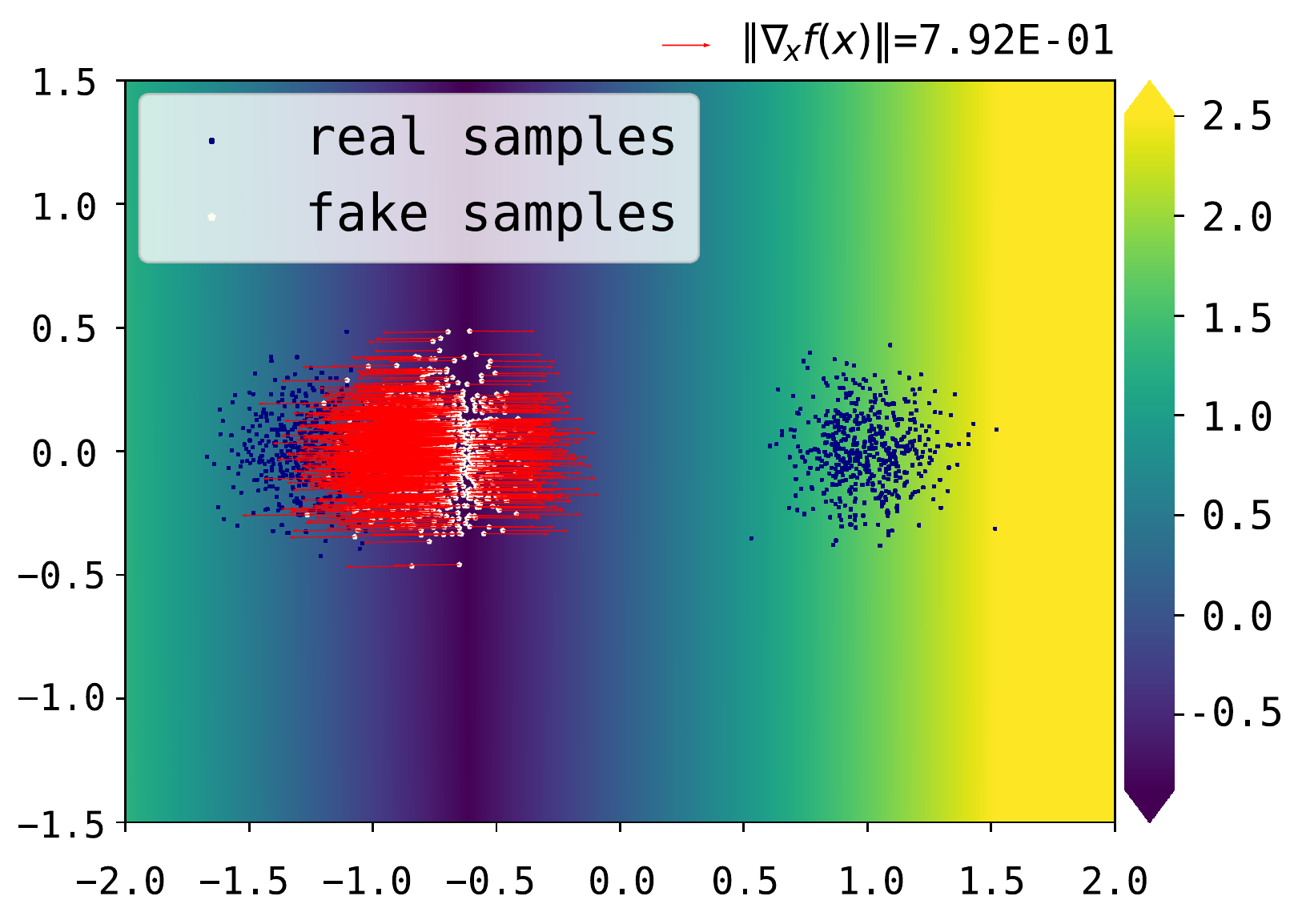}
\vspace{-17pt}
\caption{$\exp(x)$}
\label{d}
\end{subfigure}	
\caption{Verifying the objective family}
\label{figure3}
\end{minipage}
\begin{minipage}{.5\textwidth}
\includegraphics[width=0.97\linewidth]{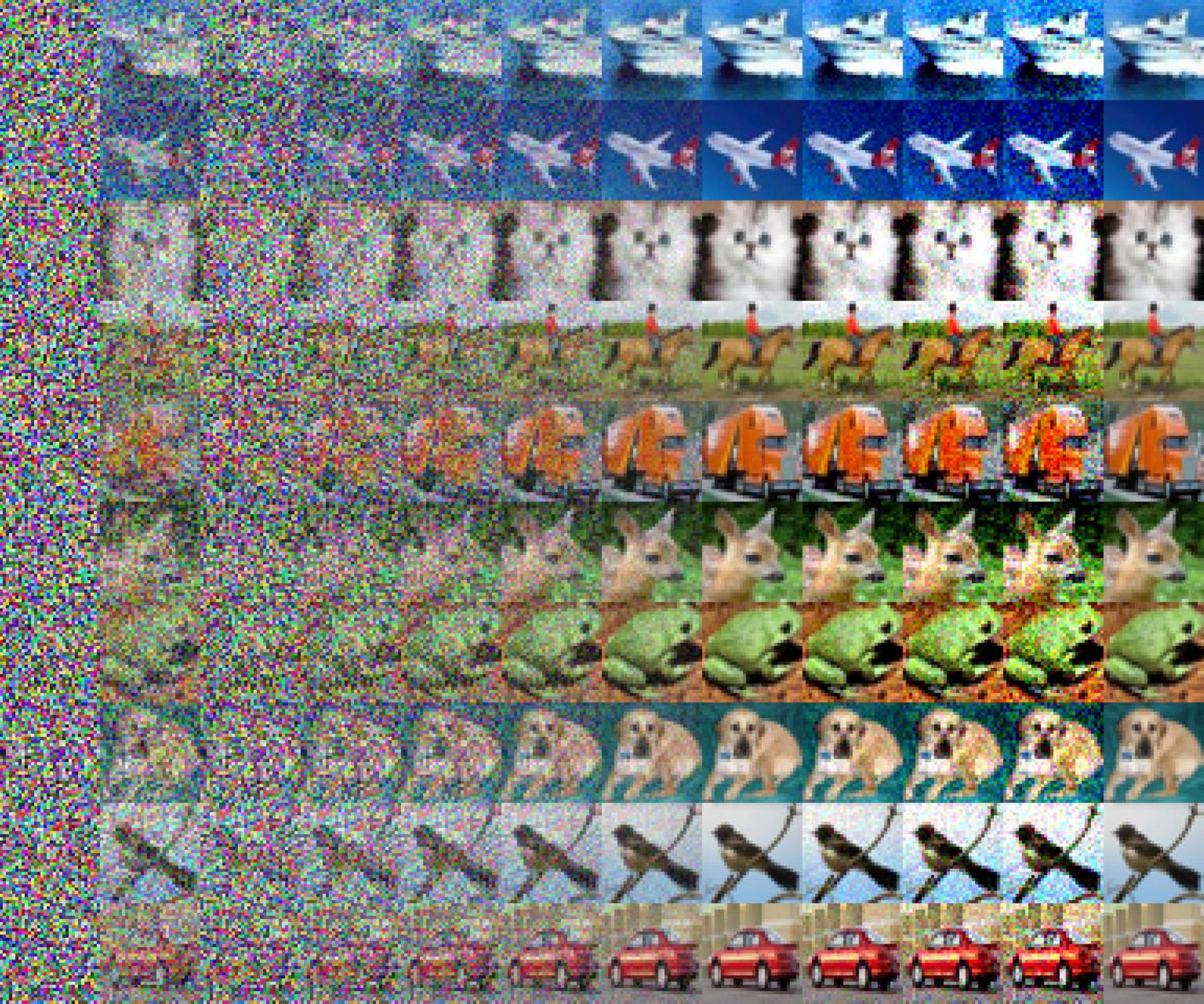}
\caption{$\nabla_{\!x} \ff(x)$ gradation with CIFAR-10}
\label{figure4}
\end{minipage}
\end{figure}

We verify a set of $\phi$ and $\varphi$ satisfying Eq.~(\ref{eq_solvable}): (a) $\phi(x)=\varphi(-x)=x$; (b) $\phi(x)=\varphi(-x)=-\log(\sigma(-x))$; (c) $\phi(x)=\varphi(-x)=x+\sqrt{x^2+1}$; (d) $\phi(x)=\varphi(-x)=\exp(x)$. As shown in Figure~\ref{figure3}, the gradient of each generated sample is towards a real sample. 

We further verify $\nabla_{\!x} \ff(x)$ with the real-world data, using \textbf{ten CIFAR-10 images} as $P_r$ and \textbf{ten noise images} as $P_g$ to make the solving of $\ff(x)$ feasible. The result is shown in Figure~\ref{figure4}, where The leftmost in each row are the $x \tsim P_g$ and the second are their gradient $\nabla_{\!x} f(x)$. The interior are $x+\epsilon\cdot\nabla_{\!x} f(x)$ with increasing $\epsilon$, which will pass through a real sample, and the rightmost are the nearest $y \tsim P_r$. This result visually demonstrates that the gradient of a generated sample is towards the direction of one real sample. Note that the final results of this experiment keep almost identical when varying the loss metric $\phi(x)$ and $\varphi(x)$ in the family. %Because, in this setting, $P_g$ and $P_r$ are totally disjoint, and according our theorem, each $x\tsim P_g$ will point towards some real sample $y\tsim P_r$. %only reason is that for all objectives that satisfying the Eq.~(\ref{eq_solvable}), it holds for 

%In this experiments, we use keep track of the maximum gradient of $f(x)$ and directly penalize the Lipschitz constant via $\max\{\lVert\nabla_{\!x} f(x)\rVert\}^2$, which we call the \textbf{maxgp}.

% \begin{figure}
% \centering
% \begin{subfigure}{0.25\linewidth}
%     \centering
%     \includegraphics[width=0.99\columnwidth]{figures/wgan_logit.pdf}
%     \vspace{-17pt}
%     \caption{$x$}
%     \label{fig5_a}
% \end{subfigure}	
% \begin{subfigure}{0.25\linewidth}
%     \centering
%     \includegraphics[width=0.99\columnwidth]{figures/wgans_logit.pdf}
%     \vspace{-17pt}
%     \caption{$-\log(\sigma(-x))$}
%     \label{fig5_b}
% \end{subfigure}	
% \vspace{-5pt}
% \caption{$\ff(x)$ in new objective is more stable.}
% \label{figure5}
% \end{figure}

\begin{figure}[t]
\vspace{-5pt}
\begin{minipage}{.6\textwidth}
\centering
\begin{subfigure}{0.49\linewidth}
    \includegraphics[width=0.99\columnwidth]{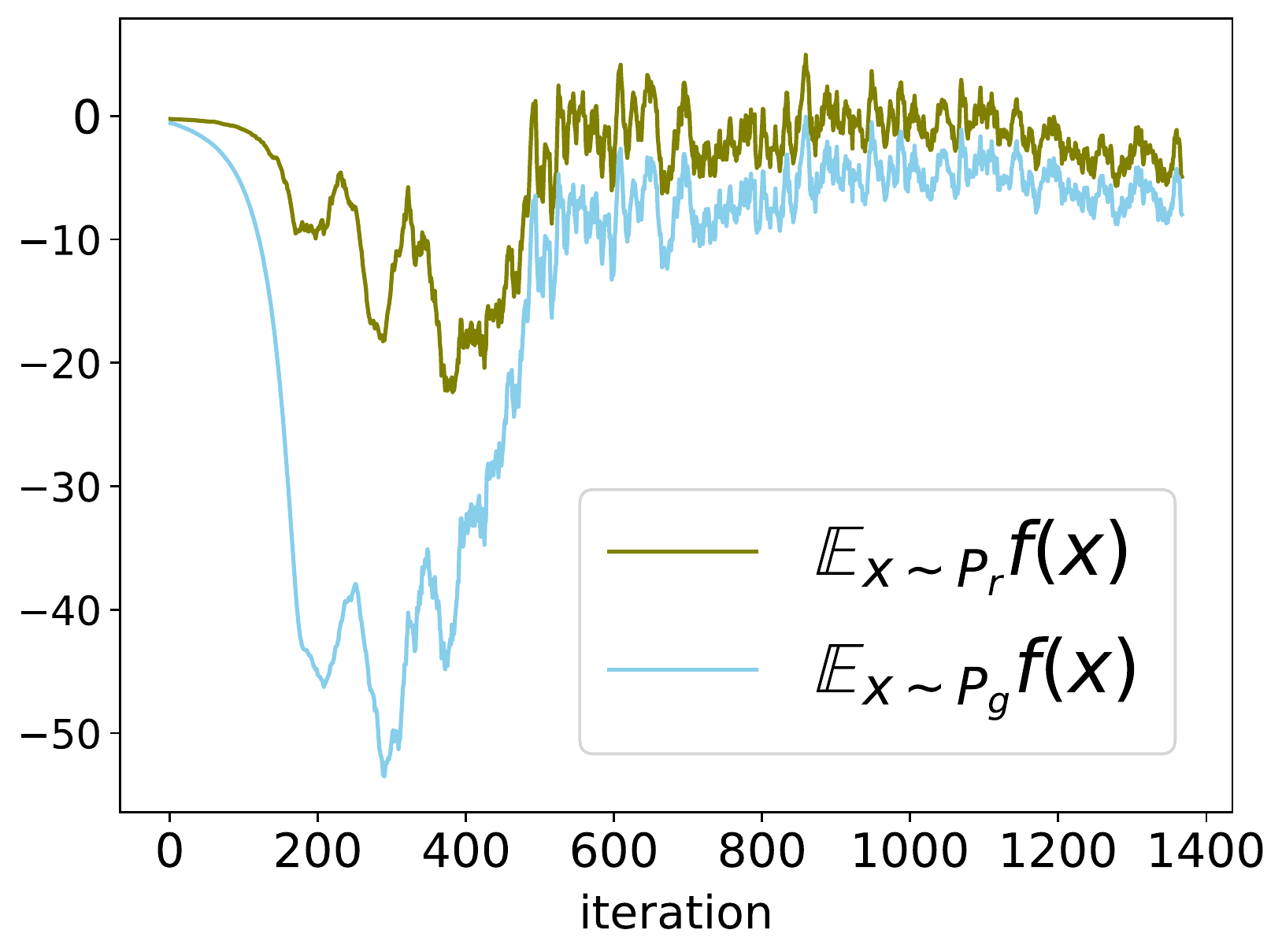}
    \vspace{-17pt}
    \caption{$x$}
    \label{fig5_a}
\end{subfigure}	
\begin{subfigure}{0.49\linewidth}
    \includegraphics[width=0.99\columnwidth]{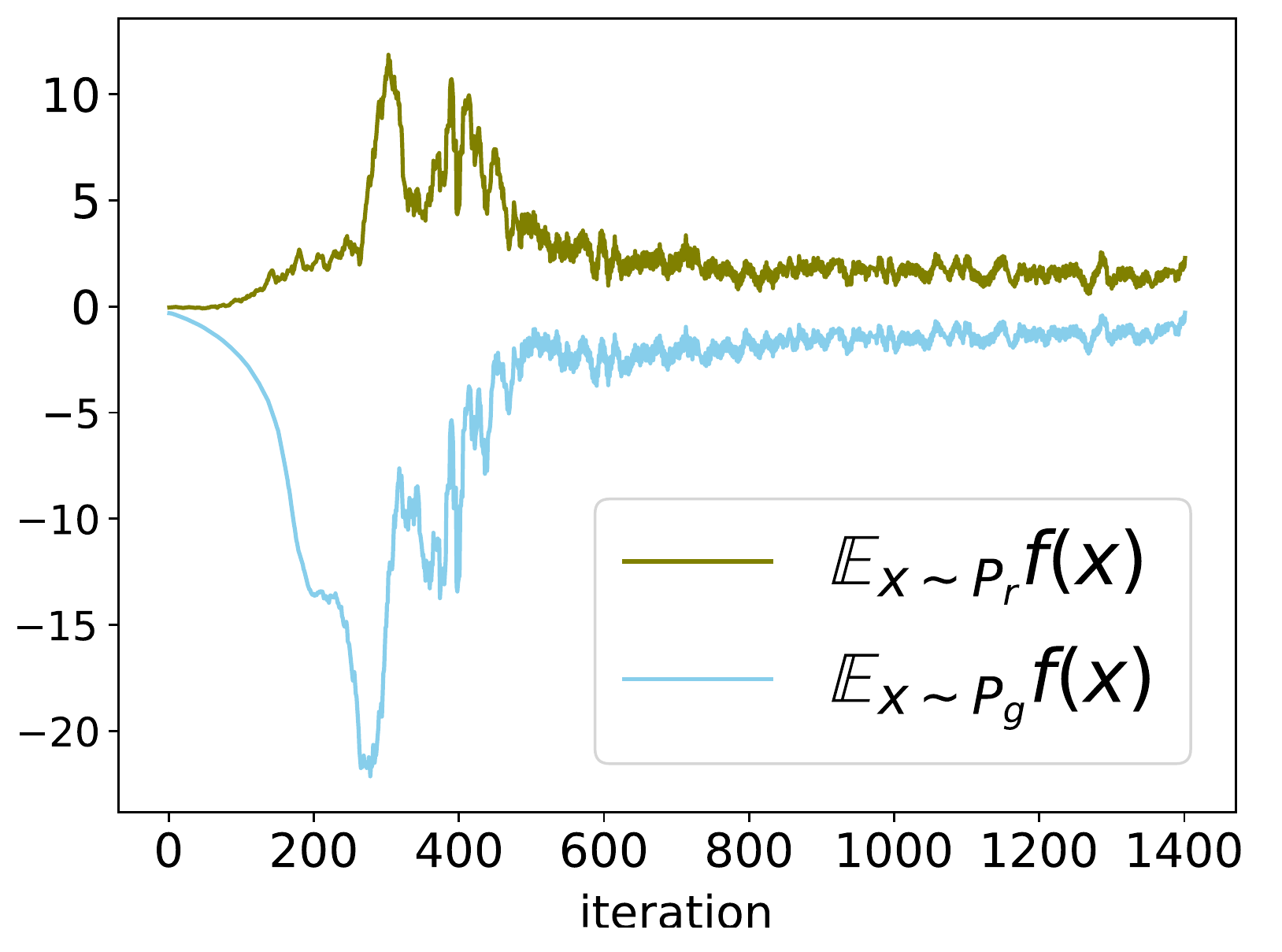}
    \vspace{-17pt}
    \caption{$-\log(\sigma(-x))$}
    \label{fig5_b}
\end{subfigure}	
\vspace{-5pt}
\caption{$\ff(x)$ in new objective is more stable.}
\label{figure5}
\end{minipage}
\begin{minipage}{.4\textwidth}	
    \vspace{-5pt}
    \centering
    \includegraphics[width=0.85\columnwidth]{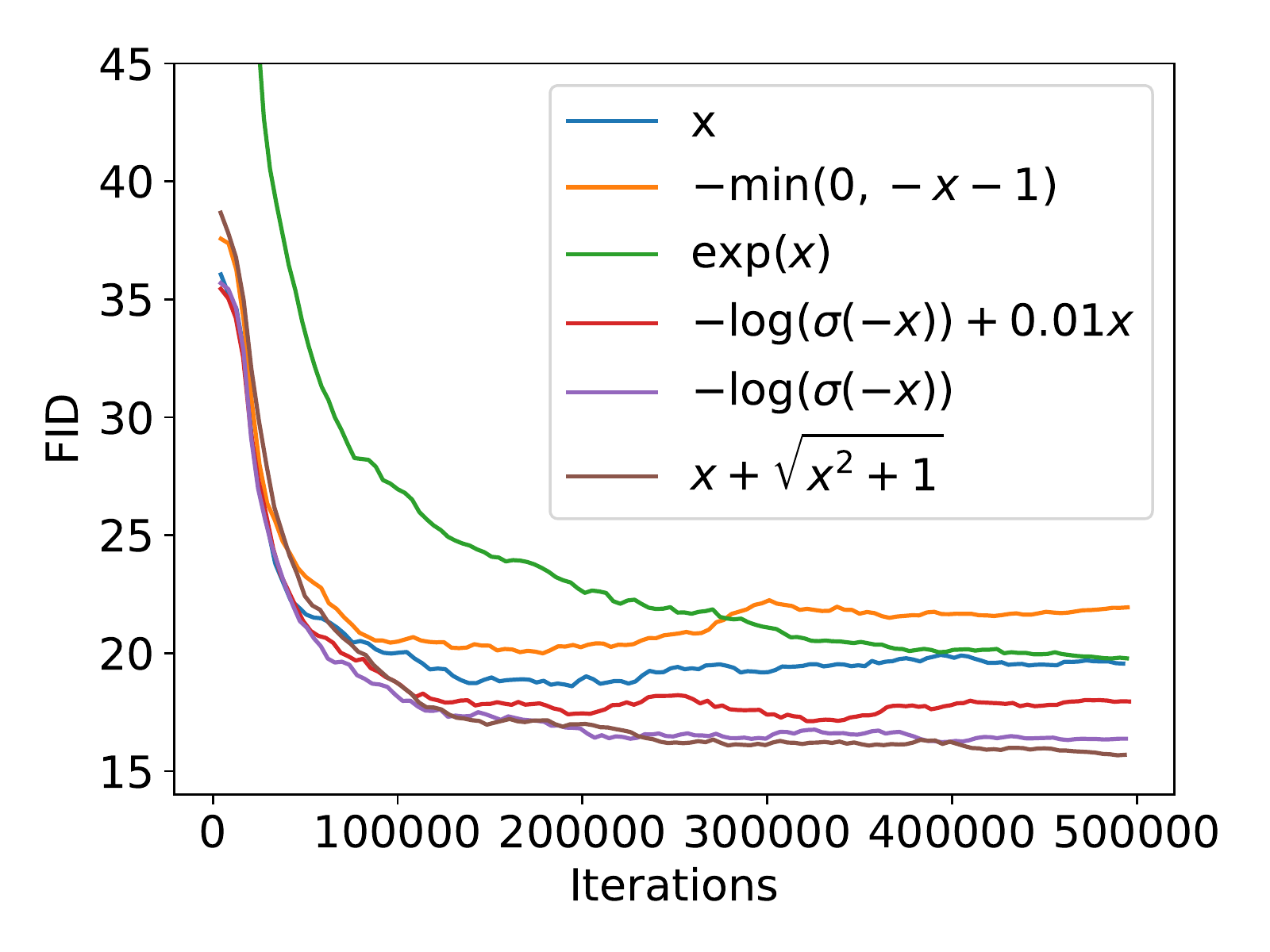}
    \vspace{-5pt}
    \caption{Training curves on CIFAR-10.}
    \label{figure6}
\end{minipage}
\vspace{-10pt}
\end{figure}        

\subsection{Stabilizing $\ff(x)$ with new Objectives}

Wasserstein distance is a special case in our proposed family of objectives where $\phi''(x)=\varphi''(x)=0$. As a result, $\ff(x)$ under the Wasserstein distance objective where $\phi(x)=\varphi(-x)=x$ has a free offset, which means given a $\ff(x)$, $\ff(x)+b$ with any $b \in \mathbb{R}$ is also an optimal. In practice, this behaves as an oscillatory $f(x)$ during training. Any other instance of our new proposed objectives does not have this problem. We illustrate this practical difference in Figure~\ref{figure5}. 

\subsection{Benchmark on unsupervised image generation tasks}
\vspace{-3pt}
\begin{table}[h]
\centering
\captionof{table}{Quantitative comparisons on unsupervised image generation tasks.}
\vspace{-5pt}
\resizebox{0.96\textwidth}{!}{
\begin{tabular}{|c|c c|c c|c c|}
\hline 
 \multirow{2}{1.2cm}{Objective}             & \multicolumn{2}{|c|}{CIFAR-10} & \multicolumn{2}{c}{Tiny ImageNet} & \multicolumn{2}{|c|}{Oxford 102 Flower}\\
 \cline{2-7}
 &FID         & IS & FID & IS  & FID* & IS* \\
 \hline
$-\min(0, -x-1)$            & $21.58\pm0.21$    & $7.43\pm0.04$     &$16.22\pm0.33$     & $\bf8.58\pm0.08$     &$9.72\pm0.51$     &$21.91\pm 0.18$ \\ 
\hline 
$x$                         & $19.64\pm0.23$    & $7.66\pm0.03$     &$18.81\pm 0.58$    &$8.20\pm 0.05$     &$9.74\pm0.63$        &$21.66\pm0.22$\\
\hline 
\rowcolor{LightCyan} $-\log(\sigma(-x))$         & $16.36\pm0.09$    &$\bf8.49\pm0.11$     &$\bf15.94\pm0.33$     &$8.42\pm0.04$     &$9.40\pm 0.49$        &$21.82\pm0.11$\\
\hline 
\rowcolor{LightCyan} $x+\sqrt{x^2+1}$            & $\bf15.76\pm0.13$ & $8.04\pm0.04$  &$16.83\pm0.41 $   &$8.35\pm0.09 $    &$\bf9.16 \pm 0.52$        & $\bf21.96\pm0.19$ \\ 
\hline 
$\exp(x)$                   & $19.82\pm0.13$    & $7.79\pm0.03$     &$20.45\pm0.15$   & $8.06\pm0.05$    &$9.90\pm0.72$        & $21.91\pm0.22 $ \\
\hline 
$-\log(\sigma(-x))+0.01x$   & $18.32\pm0.15$    & $7.75\pm0.04$     &$16.09\pm0.23$   & $8.47\pm0.10$    &$9.50\pm0.39$        & $21.91\pm0.20$ \\
\hline 
%$-\log(\sigma(-x))+x$  & $19.20\pm0.29$  & $7.75\pm0.05$ \\
%\hline 
\end{tabular}
}
\centering
\label{table2}
\vspace{-3pt}
\end{table}

Finally, we fix $\psi(x)=-x$ in the generator's objective and compare various objectives on unsupervised image generation tasks. The results of Inception Score \citep{improved_gan} and Frechet Inception Distance \citep{two_time_scale_gan} are presented in Table~\ref{table2}. We also include the hinge loss $\phi(x)=\varphi(-x)=-\min(0, -x-1)$ which used in \citep{sngan}. We use a classifier on Oxford 102 Flower Dataset for the evaluation of FID and Inception Score for results on Oxford 102.

The gradient of $\exp(x)$ varies significantly and we find it requires a small learning rate to avoid explosion. The objectives $x + \sqrt{x^2 + 1}$ and $-\log(\sigma(-x))$ %and $-\log(\sigma(-x))+ 0.01x$ 
achieve the best performances. This is probably because they have bounded gradient and reduce the gradient of well-identified points towards zero, which enables the discriminator to pay more attention to these ill-identified. Hinge loss $-\min(0, -x-1)$ does not lie in our proposed objective family and turns out to be unstable and performs unsatisfactory in same cases. %According to the experiments, we suggest using $-\log(\sigma(-x))+\alpha x$ with a small $\alpha$, which has non-vanishing gradient and may be more easy-to-use. 
We also plot the training curve in terms of FID in Figure~\ref{figure6}. 

Due to page limitation, we leave the details, visual results and more experiments in the Appendix. 
 %$\phi(x)=\varphi(-x)=-\log(\sigma(-x))+x$ which is the combination of the original GAN objective and Wasserstein distance, and 
%$x$ and $-\log(\sigma(-x))+x$ lie in the middle. 

\vspace{-5pt}
\section{Conclusion}
\vspace{-5pt}

%\zhiming{yuxuan, please handle the \bf warnings and errors in ``log''}

In this paper we have shown that the fundamental cause of failure in training of GANs stems from the unreliable $\nabla_{\!x} \ff(x)$. Specifically, when $P_g$ and $P_r$ are disjoint, $\nabla_{\!x} \ff(x)$ for fake sample $x \tsim P_g$ tells nothing about $P_r$, making it impossible for $P_g$ to converge to $P_r$. We have further demonstrated that even Wasserstein distance in a more compact dual form (still is equivalent to Wasserstein distance and can properly measure the distance between distributions) also suffers from the same problem when $P_g$ and $P_r$ are disjoint. This implies that ``whether a distance metric can properly measure the distance'' does not yet touch the key of non-convergence of GANs. We have highlighted in this paper that a well-defined distance metric %, or more generally, ``$P_g = P_r$ is the optimum'', 
does not necessarily guarantee the convergence of GANs because $\nabla_{\!x} \ff(x)$ %which the generator's update based on 
can be meaningless. 
Therefore, if we update the generator based on $\nabla_{\!x} \ff(x)$, we need to pay more attention on the design of $\ff(x)$. Furthermore, to address the aforementioned problem, we have proposed the Lipschitz-continuity condition as a general solution to make $\nabla_{\!x} \ff(x)$ reliable and ensure the convergence of GANs, which works well with a large family of GAN objectives. In addition, we have shown that in the overlapping case, $\nabla_{\!x} \ff(x)$ is also problematic which turns out to be an intrinsic cause of mode collapse in traditional GANs.

% \iffalse

% \subsection{Remarks}

\textbf{Remark 1:} It is worth noticing that $\dloss$ in our formulation is not derived from any well-established distance metric; it is derived based on Lipschitz-continuity condition. As we have shown that a well-established distance metric does not necessarily ensure the convergence, we hope our trial could shed light on the new direction of GANs. 

\textbf{Remark 2:} Though the objective of generator is not the focus of this paper, our analysis indicates that the minimax in terms of $\psi$ in Eq.~(\ref{eq_gan_formulation}) is not essential, because it only influences the scale of the gradient. Nevertheless, the function $\psi$ does influence the updating of the generator and we leave the detailed investigation as future work. 

%Another example that has a theoretically meaningful $\nabla_{\!x} \ff(x)$ is Coulomb GAN \citep{coulombgan}, which is also derived neither from minimax game nor from well-defined distance metric. 

% \fi

%\vspace{-2pt}
\section{Related work} 

% \subsection{Clarification: relation to Wasserstein GAN}

%This work is substantially different from Wasserstein GAN \citep{wgan}. Though the final solution in Wasserstein GAN is sound, its main argument is off the point. 
The main argument in Wasserstein GAN \citep{wgan} for the benefit of Wasserstein distance is that it can properly measure the distance between two distributions no matter whether their supports are disjoint. %, i.e., Wasserstein distance is a good distance metric. 
However, according to our analysis, %in Section~\ref{sec_not_enough}, 
a proper distance metric does not necessarily ensure the convergence of GAN and the Lipschitz-continuity condition in Wasserstein GAN is crucial for ensuring its convergence. More specifically, we have shown that Wasserstein distance in the dual form with compacted constraint also cannot provide meaningful gradient through $\nabla_{\!x} \ff(x)$. % Though it can properly measure the distance no matter whether their supports are disjoint, as we showed in , it also does not guarantee the convergence. 
%the final solution in Wasserstein GAN is sound as we prove. %We argue in this paper that a well-defined distance metric, including Wasserstein distance, is not enough for ensure the convergence of GANs and prove that Lipschitz-continuity condition is a general solution to non-convergence problem. 

In addition, we have shown that Lipschitz-continuity condition is able to ensure the convergence of GANs for a family of GAN objectives, which is not restricted to Wasserstein distance. For example, Lipschitz-continuity condition is also introduced to original GAN in \citep{sngan,kodali2017convergence} and shows improvements on the quality of generated samples. As a matter of fact, the original GAN objective $\phi(x)=\varphi(-x)=-\log(\sigma(-x))$ is another instance in our proposed family. Thus our analysis explains why and how it works. 

\cite{fedus2017many} also argued that divergence is not the primary guide of the training of GANs and pointed out that the gradient does not necessarily related to the divergence. However, they tended to believe that original GAN with non-saturating generator objective can somehow work. As we have proved before, given the optimal $\ff$, the original GAN has no guarantee on its convergence. And we argue that practical work scenarios benefit from parameter-tuning.
%Given that their arguments are mainly supported via experiments, we believe our conclusion that derived from theoretically analysis is more sound. %Towards their finding of ``gradient penalty improves original GANs'', our theory is more general and explicitly tells under what condition and how it helps. % which is contradict to our conclusion.

Some work study the suboptimal $f(x)$ \citep{mescheder2017numerics,mescheder2018training,arora2017generalization}, which is another important direction for understanding GANs theoretically. While the behaviors of suboptimal can be slightly different, we think the optimal $\ff(x)$ should well-behave in the first place. %\weinan{what do you mean by `at the first place'?}\zhiming{at first?in the first place?}

Researchers also found that applying Lipschitz-continuity condition to the generator also benefits the quality of generated samples \citep{zhang2018self, odena2018generator}. In addition, researchers also investigated implementation of Lipschitz-continuity condition in GANs \citep{wgangp,wganlp,sngan}. However, this branch of related work is out of scope of the discussion in this paper.

%\paragraph{Clarification} Because there does exist a strong connection between our work and the previous work Wasserstein GAN \citep{wgan}, we here further clear several points to make it less confusing: (\rnum{1}) we show that there exists a broad family of GANs objectives that combines with Lipschitz-continuity condition guaranteeing the convergence while Wasserstein distance is only one of its instances; (\rnum{2}) Lipschitz-continuity condition is stronger than the necessary one in the dual form of Wasserstein distance; (\rnum{3}) we show that the pure Wasserstein distance without the unnecessary stronger Lipschitz-continuity condition will also fail; (\rnum{4}) our contribution also include the investigation on the cause of failure in traditional GANs, in terms of $\ff(x)$ and $\nabla_{\!x} \ff(x)$. 

\bibliography{reference}
\bibliographystyle{iclr2019}
\newpage
\appendix

\section{Experiments: the influence of hyper-Parameters}
\label{hyper_para}
The value surface of traditional GANs is highly depended on the network and training hyper-parameters. 
We here plot the value surface of Least-Square GAN with various hyper-parameter settings, to give directly impression on how these parameters influence GANs training. Not very strictly, but our empirical code is: (\rnum{1}) a low-capacity network tends to learn a simple surface; (\rnum{2}) SGD tends to learn a more complex surface than ADAM; (\rnum{3}) large learning rate tends to learns a simpler surface than small learning rate; (\rnum{4}) highly nonlinear activation function tends to result in more complex value surface.

Though hyper-parameters tuning could possibly make traditional GANs work, it also makes these GANs hard to use, sensitive to hyper-parameters and easily broken.

\vspace{-0pt}
\begin{figure}[!htbp]
	\begin{subfigure}{0.33\linewidth}
		\vspace{-0pt}
		\centering
		\includegraphics[width=0.90\columnwidth]{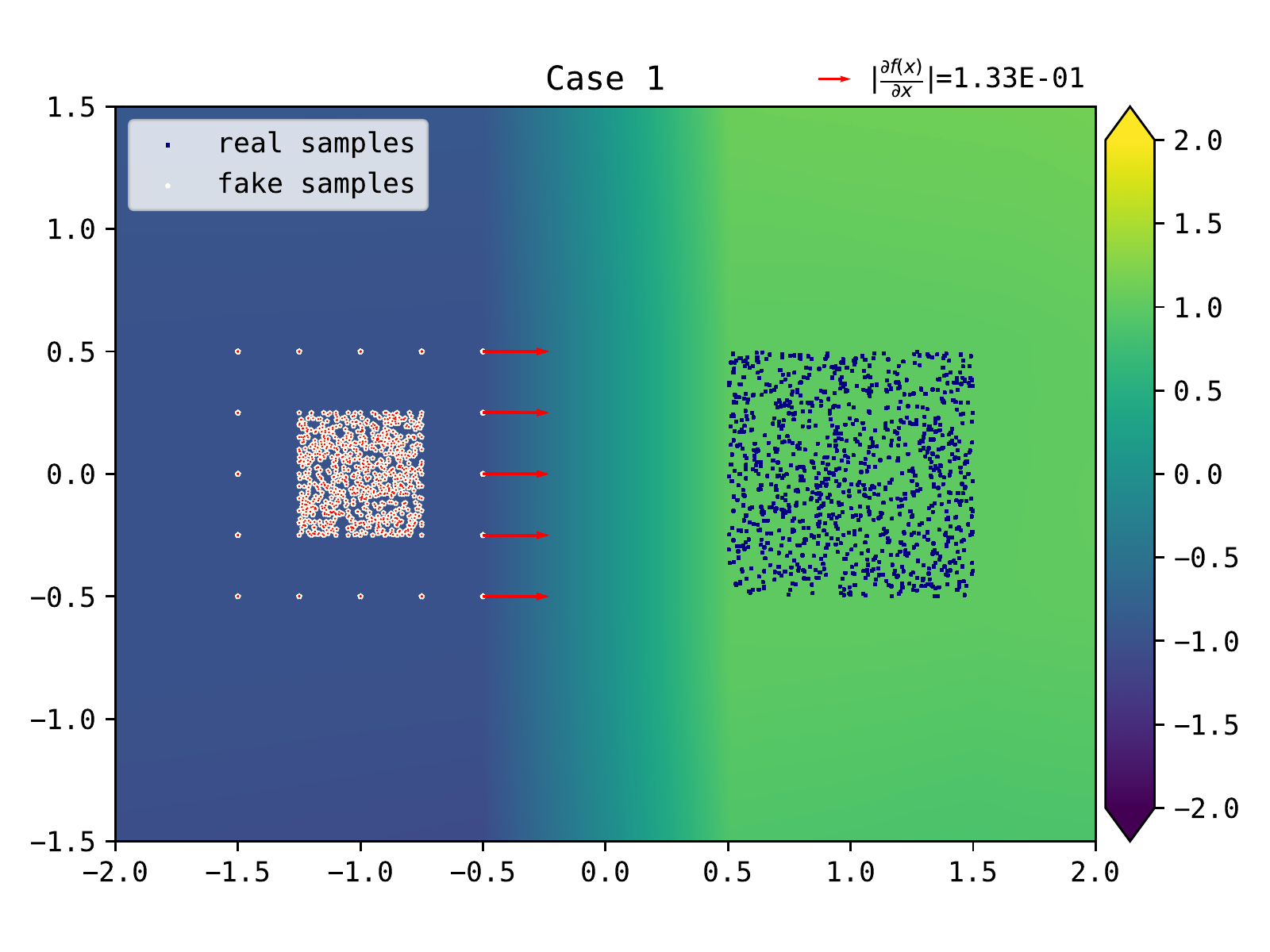}
		%\vspace{-10pt}
		%\caption{Activation: RELU; Optimizer: ADAM; Small LR}
		\label{fig_case1_lsgan_adam_1e-2_relu_1024*1_toy}
	\end{subfigure}
	\begin{subfigure}{0.33\linewidth}
		\vspace{-0pt}
		\centering
		\includegraphics[width=0.90\columnwidth]{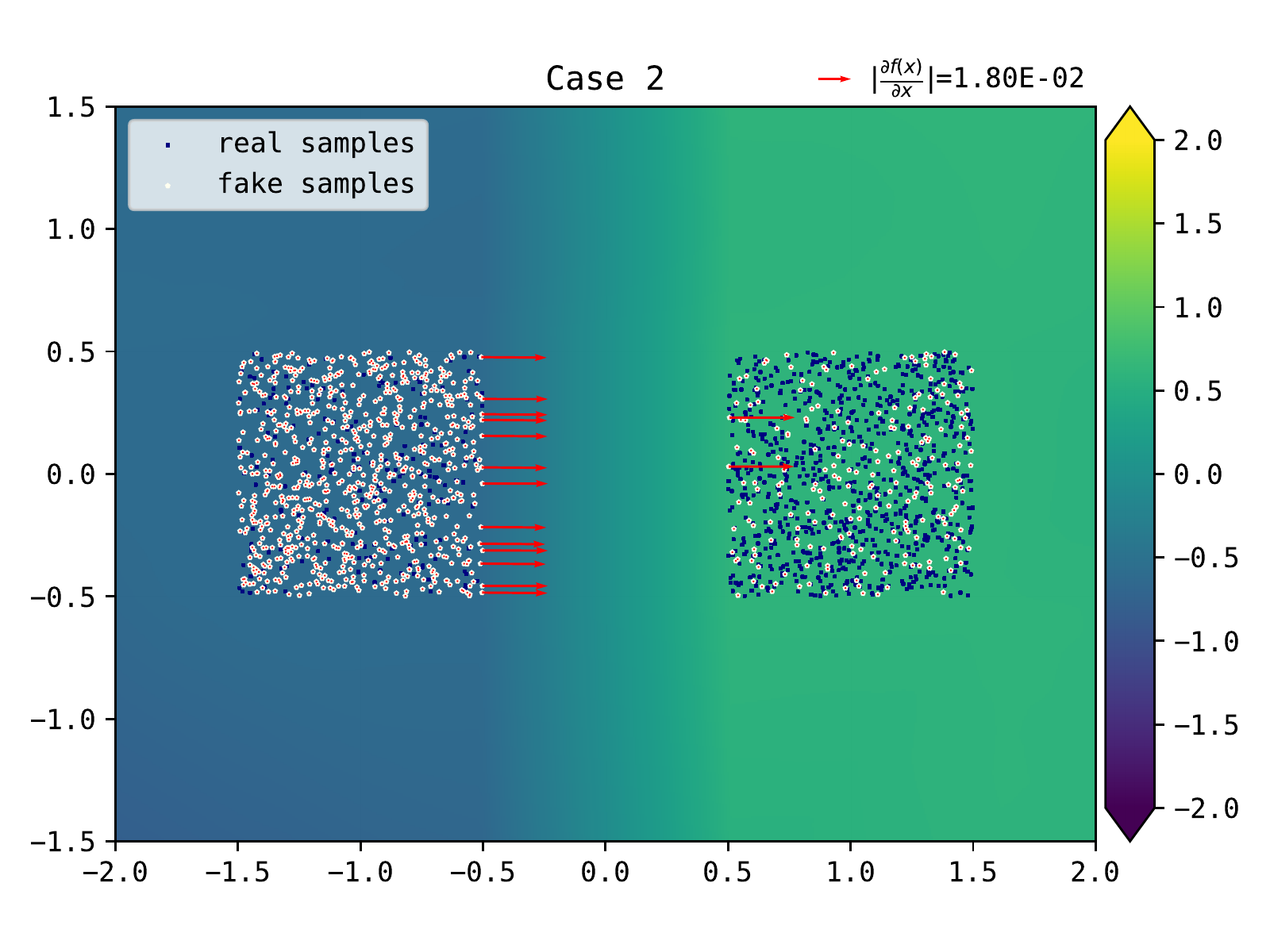}
		%\vspace{-10pt}
		%\caption{Activation: RELU; Optimizer: ADAM; Small LR}
		\label{fig_case2_lsgan_adam_1e-2_relu_1024*1_toy}
	\end{subfigure}
	\begin{subfigure}{0.33\linewidth}
		\vspace{-0pt}
		\centering
		\includegraphics[width=0.90\columnwidth]{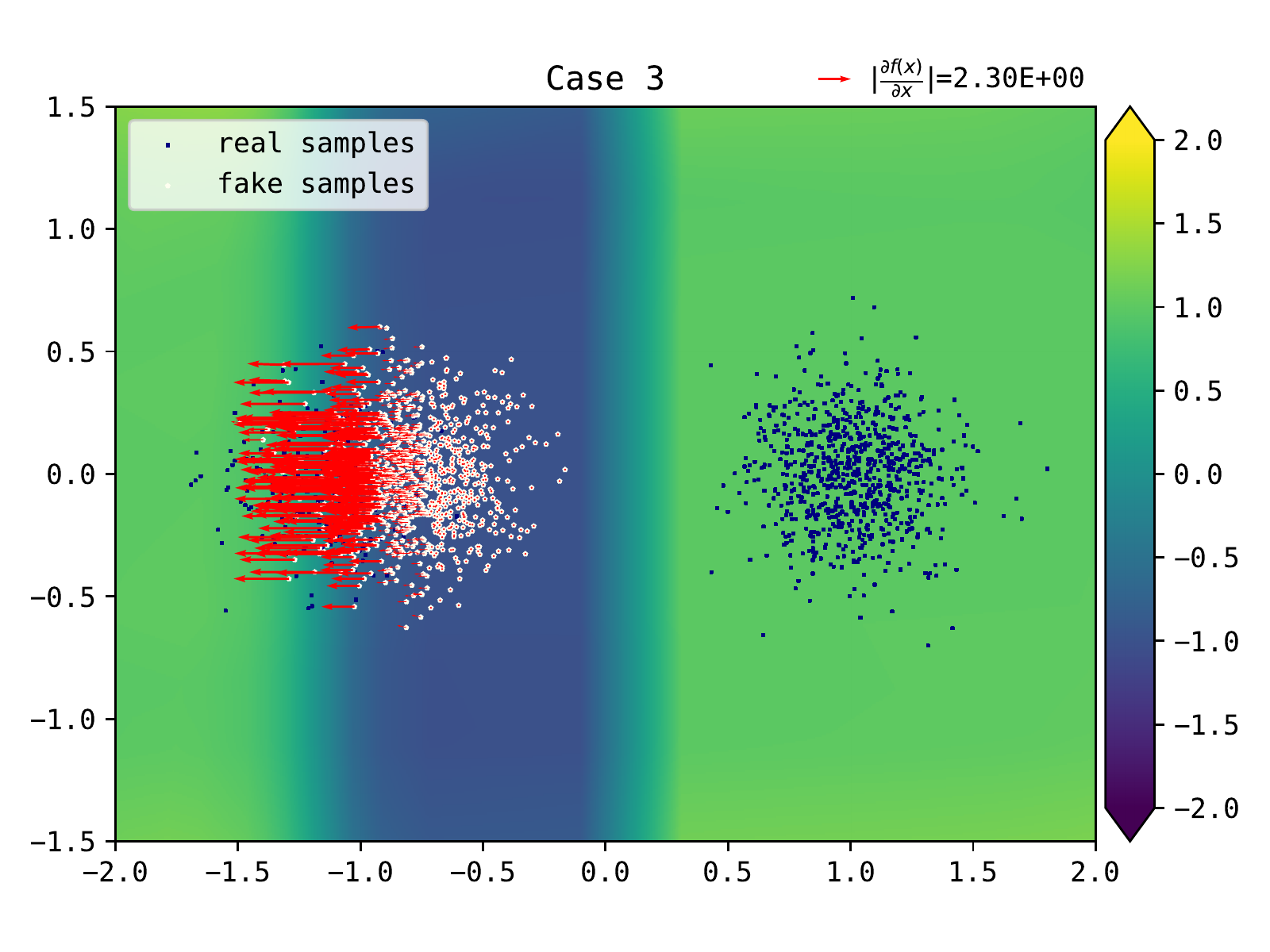}
		%\vspace{-10pt}
		%\caption{Activation: RELU; Optimizer: ADAM; Small LR}
		\label{fig_case3_lsgan_adam_1e-2_relu_1024*1_toy}
	\end{subfigure}
	\caption{ADAM optimizer with lr=1e-2, beta1=0.0, beta2=0.9. MLP with RELU activations, \#hidden units=1024, \#layers=1.}
\end{figure}
\vspace{-10pt}
\begin{figure}[!htbp]
	\vspace{-0pt}
	\begin{subfigure}{0.33\linewidth}
		\vspace{-0pt}
		\centering
		\includegraphics[width=0.90\columnwidth]{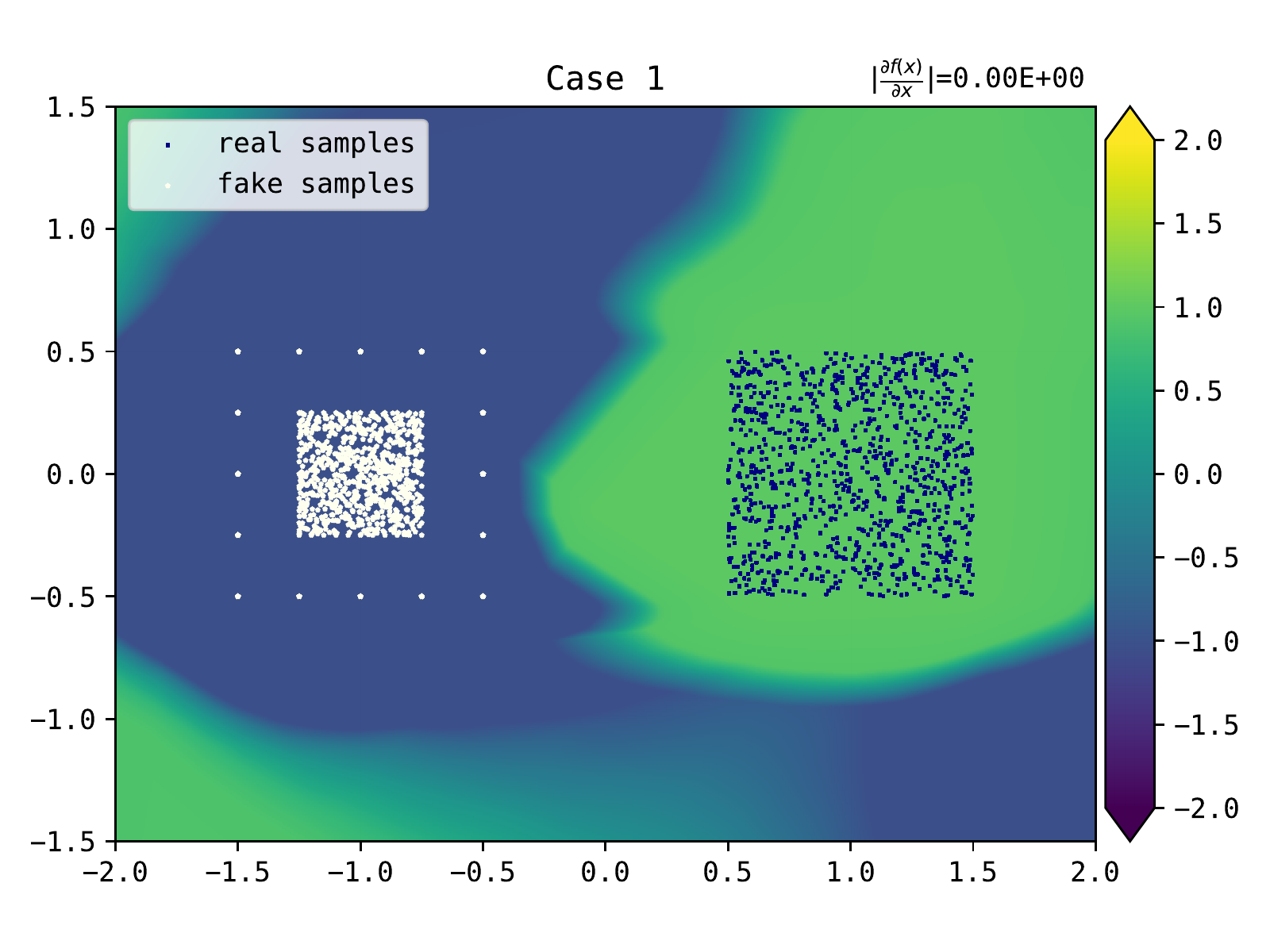}
		%\vspace{-10pt}
		%\caption{Activation: RELU; Optimizer: ADAM; Small LR}
		\label{fig_case1_lsgan_adam_1e-2_relu_1024*4_toy}
	\end{subfigure}
	\begin{subfigure}{0.33\linewidth}
		\vspace{-0pt}
		\centering
		\includegraphics[width=0.90\columnwidth]{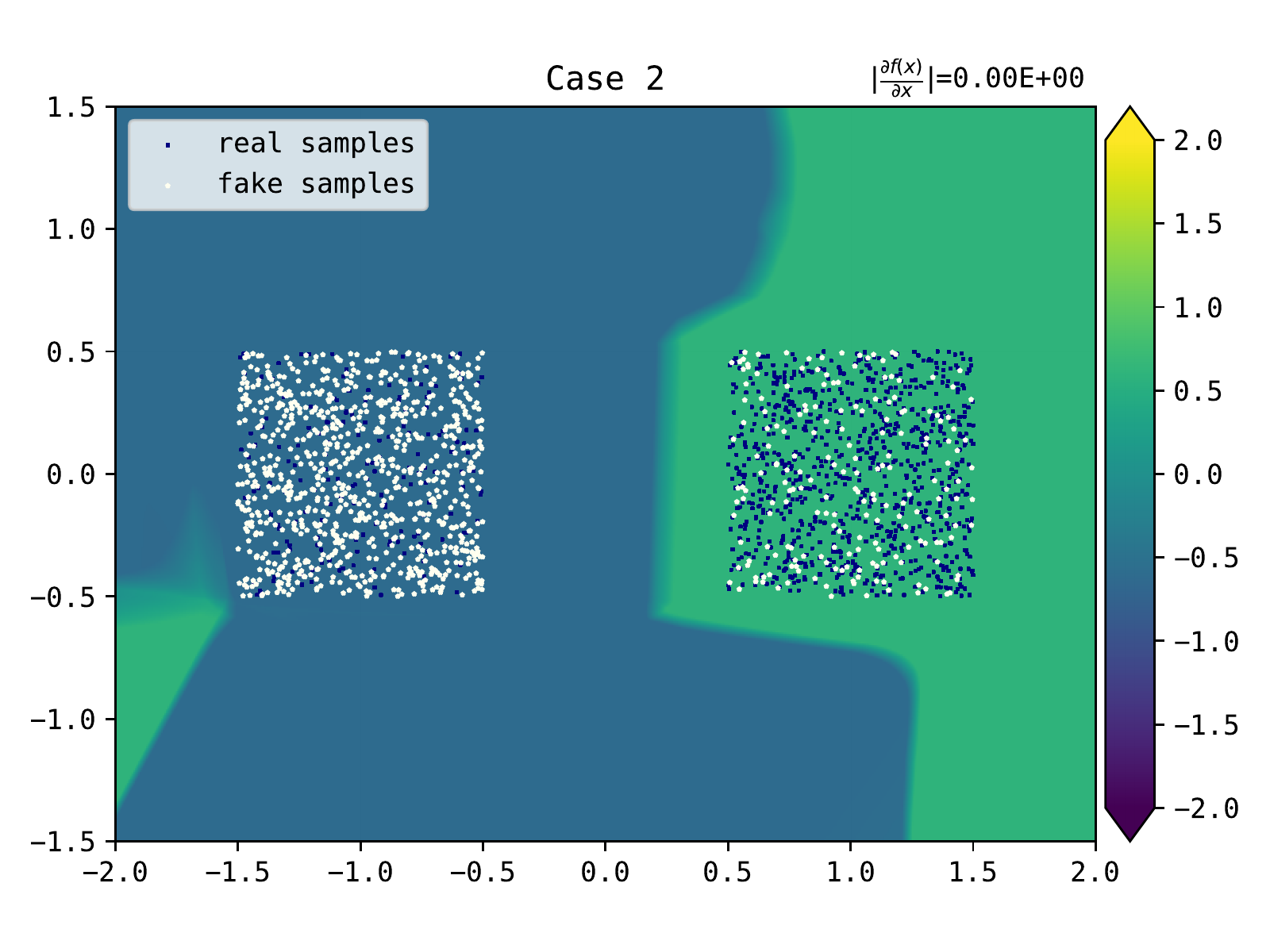}
		%\vspace{-10pt}
		%\caption{Activation: RELU; Optimizer: ADAM; Small LR}
		\label{fig_case2_lsgan_adam_1e-2_relu_1024*4_toy}
	\end{subfigure}
	\begin{subfigure}{0.33\linewidth}
		\vspace{-0pt}
		\centering
		\includegraphics[width=0.90\columnwidth]{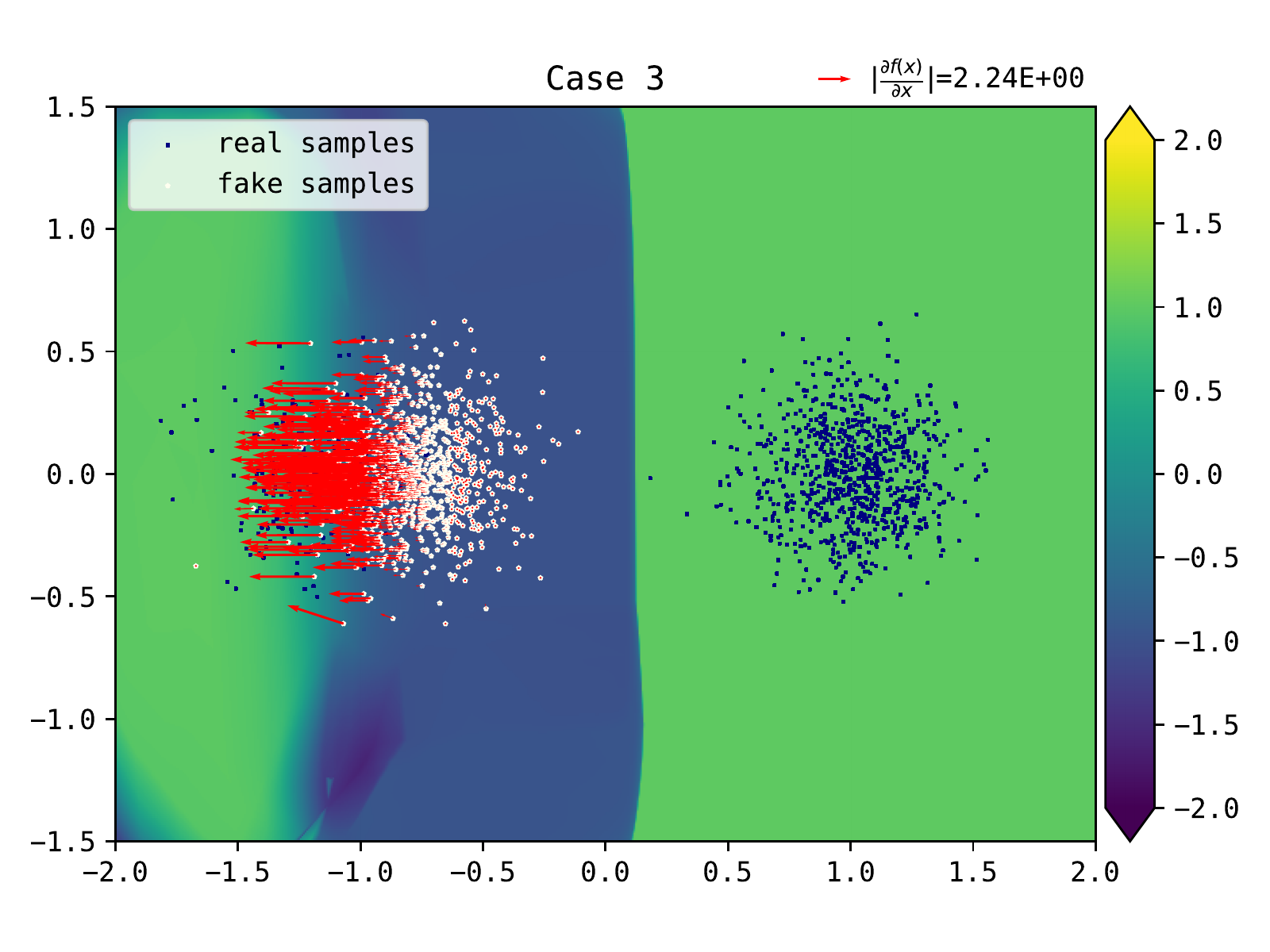}
		%\vspace{-10pt}
		%\caption{Activation: RELU; Optimizer: ADAM; Small LR}
		\label{fig_case3_lsgan_adam_1e-2_relu_1024*4_toy}
	\end{subfigure}
	\caption{ADAM optimizer with lr=1e-2, beta1=0.0, beta2=0.9. MLP with RELU activations, \#hidden units=1024, \#layers=4.}
\end{figure}
\vspace{-10pt}
\begin{figure}[!htbp]
	\vspace{-0pt}
	\begin{subfigure}{0.33\linewidth}
		\vspace{-0pt}
		\centering
		\includegraphics[width=0.90\columnwidth]{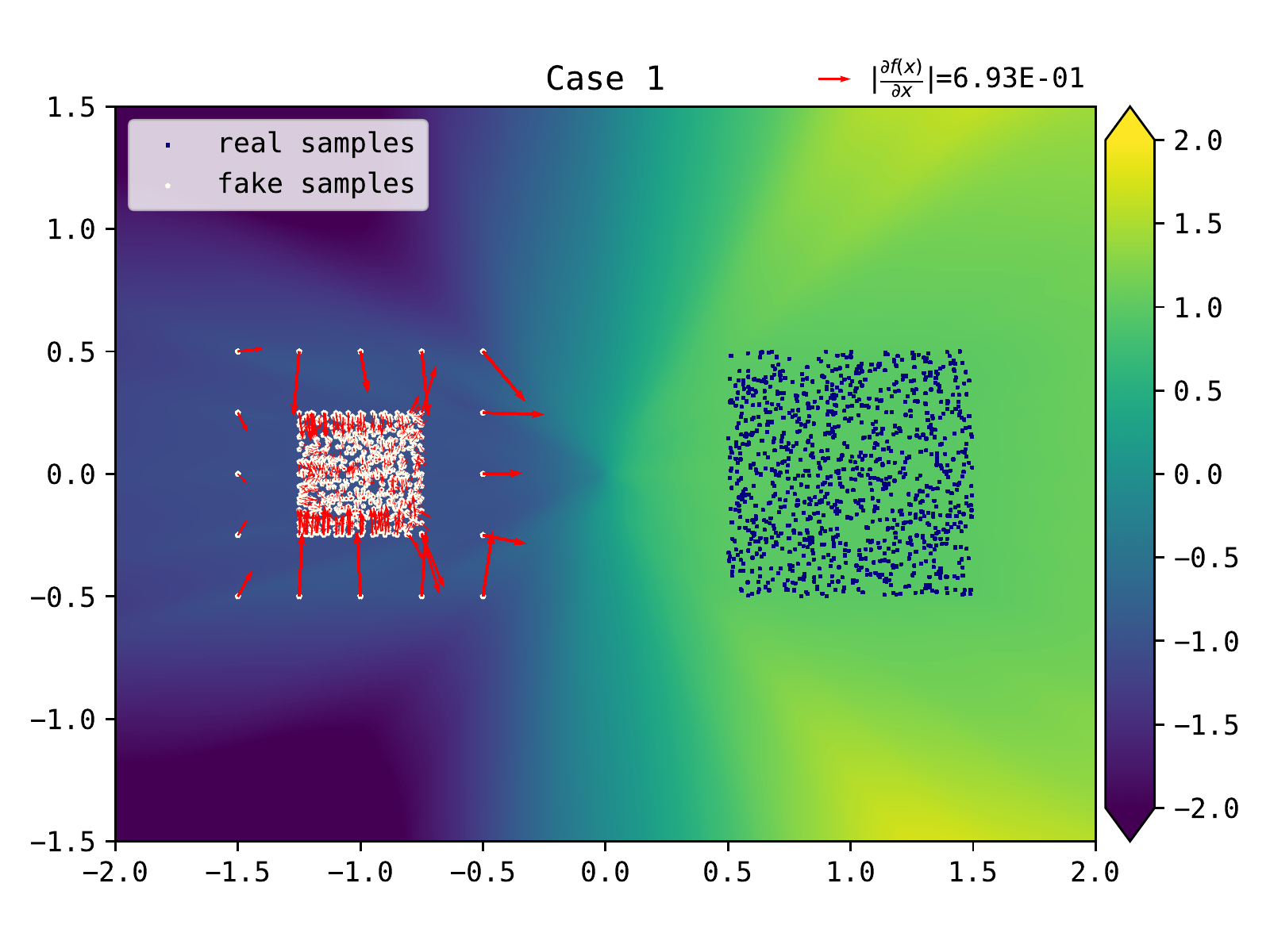}
		%\vspace{-10pt}
		%\caption{Activation: RELU; Optimizer: ADAM; Small LR}
		\label{fig_case1_lsgan_adam_1e-5_relu_1024*4_toy}
	\end{subfigure}
	\begin{subfigure}{0.33\linewidth}
		\vspace{-0pt}
		\centering
		\includegraphics[width=0.90\columnwidth]{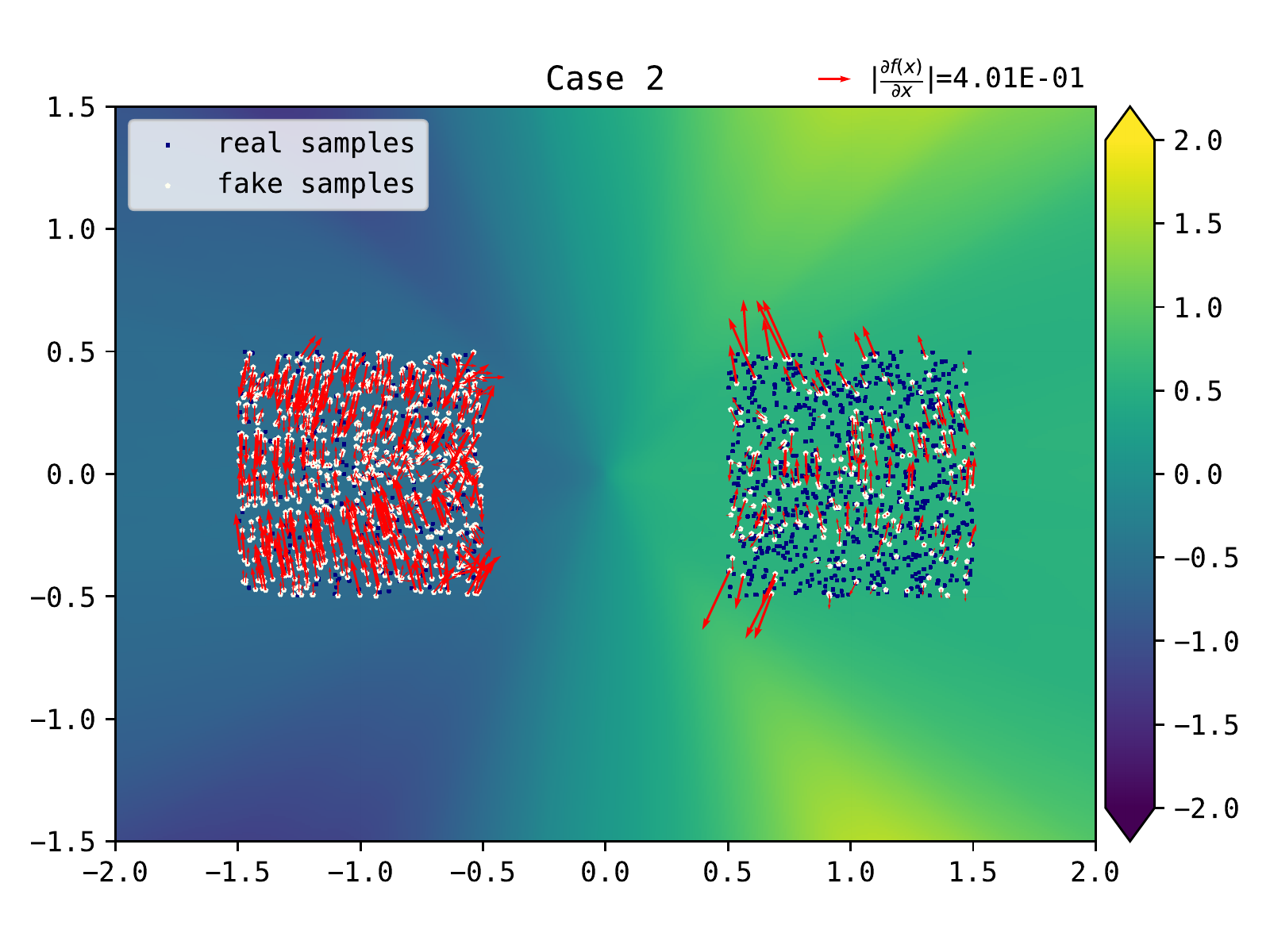}
		%\vspace{-10pt}
		%\caption{Activation: RELU; Optimizer: ADAM; Small LR}
		\label{fig_case2_lsgan_adam_1e-5_relu_1024*4_toy}
	\end{subfigure}
	\begin{subfigure}{0.33\linewidth}
		\vspace{-0pt}
		\centering
		\includegraphics[width=0.90\columnwidth]{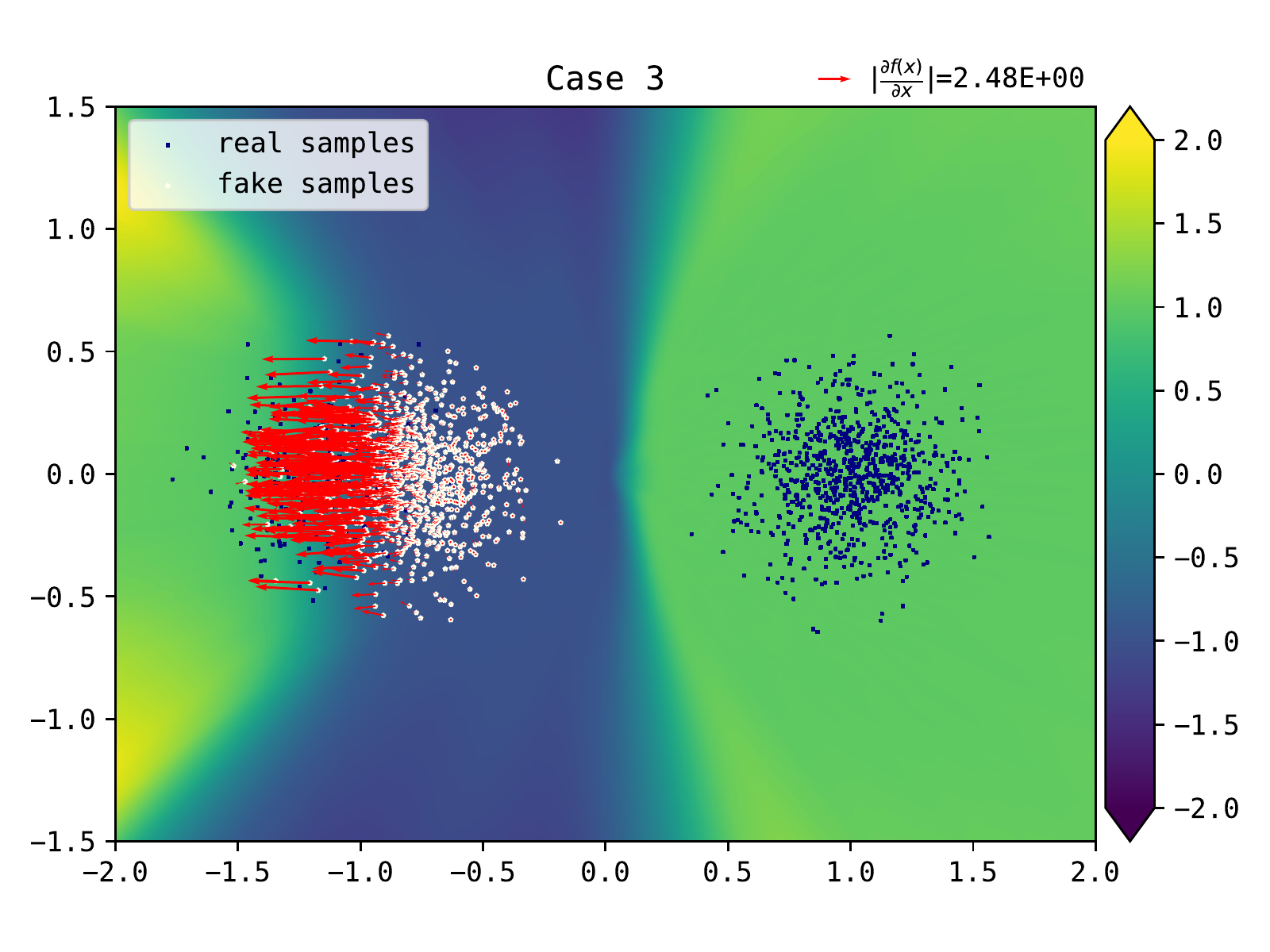}
		%\vspace{-10pt}
		%\caption{Activation: RELU; Optimizer: ADAM; Small LR}
		\label{fig_case3_lsgan_adam_1e-5_relu_1024*4_toy}
	\end{subfigure}
	\caption{ADAM optimizer with lr=1e-5, beta1=0.0, beta2=0.9. MLP with RELU activations, \#hidden units=1024, \#layers=4.}
\end{figure}
\vspace{-20pt}
\begin{figure}[!htbp]
	\vspace{-0pt}
	\begin{subfigure}{0.33\linewidth}
		\vspace{-0pt}
		\centering
		\includegraphics[width=0.90\columnwidth]{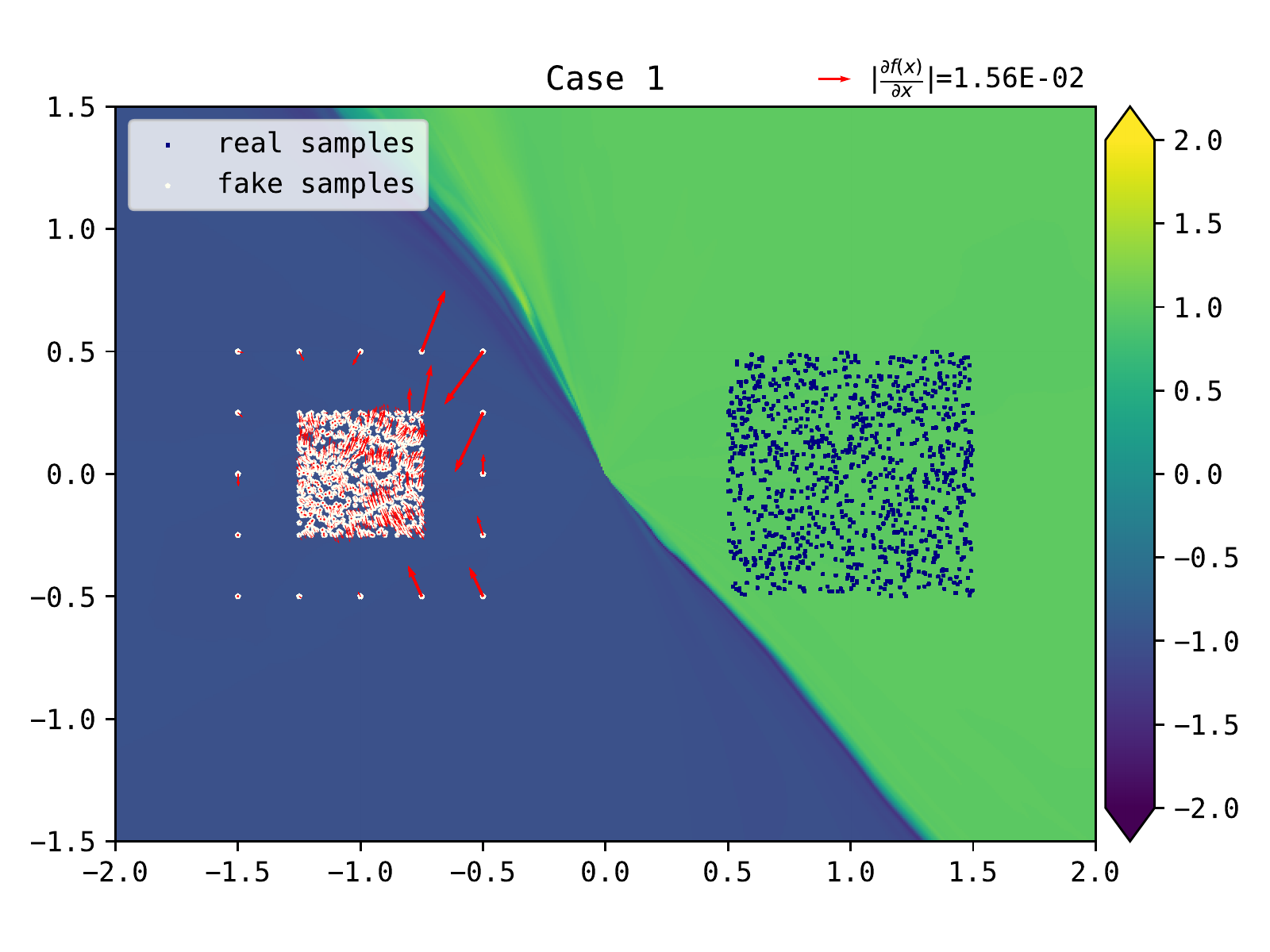}
		%\vspace{-10pt}
		%\caption{Activation: RELU; Optimizer: ADAM; Small LR}
		\label{fig_case1_lsgan_sgd_1e-3_relu_128*64_toy}
	\end{subfigure}
	\begin{subfigure}{0.33\linewidth}
		\vspace{-0pt}
		\centering
		\includegraphics[width=0.90\columnwidth]{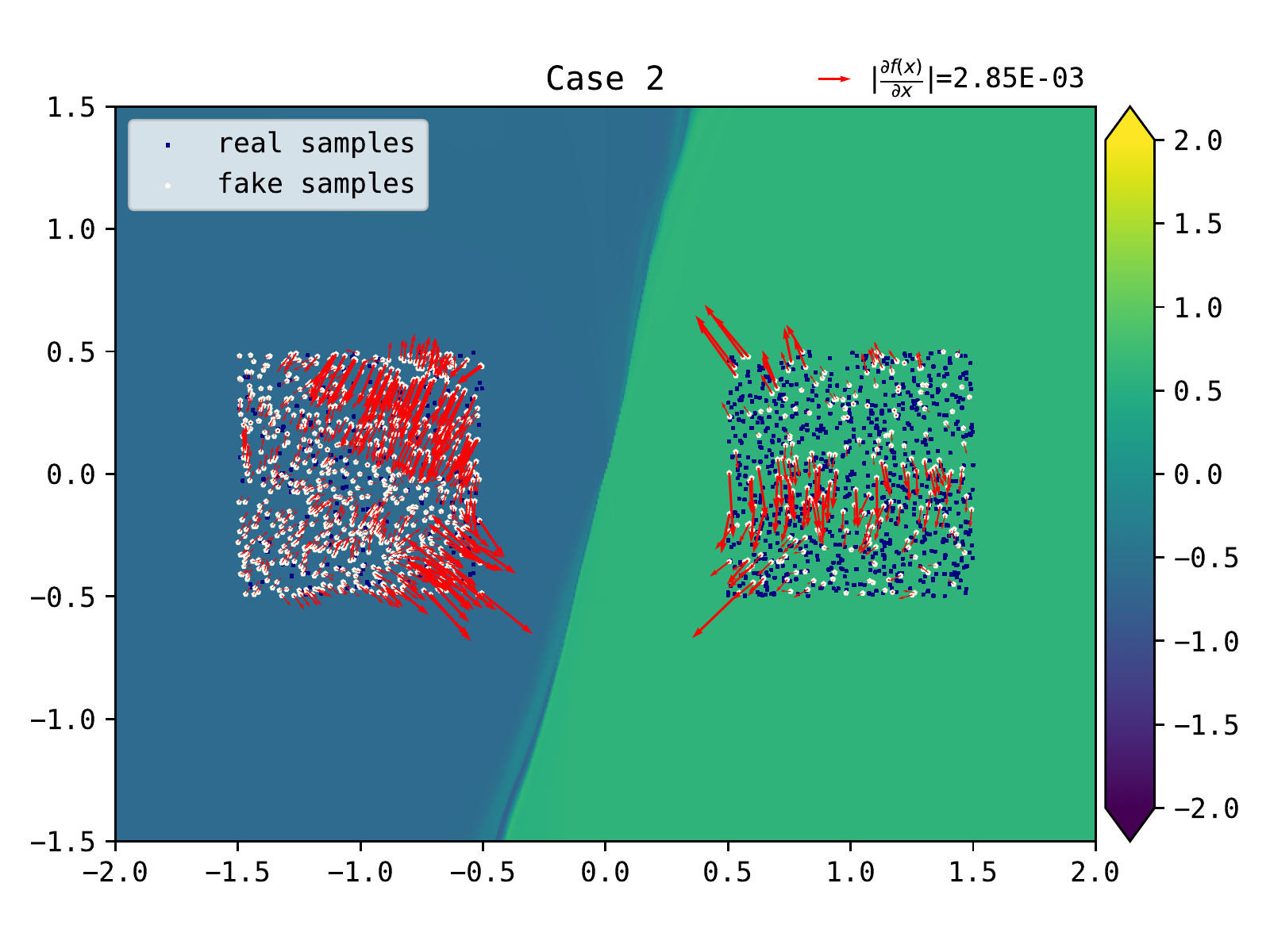}
		%\vspace{-10pt}
		%\caption{Activation: RELU; Optimizer: ADAM; Small LR}
		\label{fig_case2_lsgan_sgd_1e-3_relu_128*64_toy}
	\end{subfigure}
	\begin{subfigure}{0.33\linewidth}
		\vspace{-0pt}
		\centering
		\includegraphics[width=0.90\columnwidth]{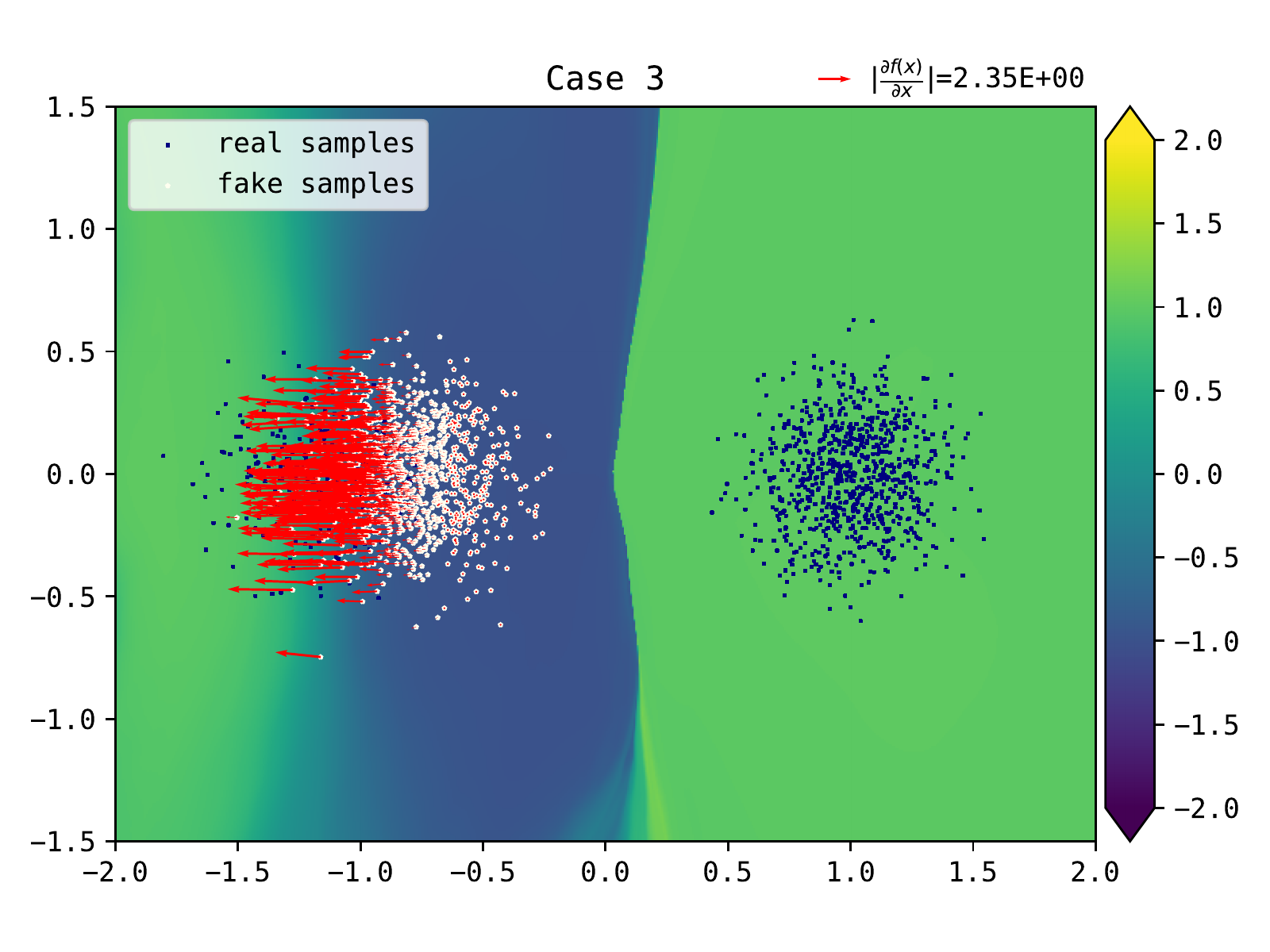}
		%\vspace{-10pt}
		%\caption{Activation: RELU; Optimizer: ADAM; Small LR}
		\label{fig_case3_lsgan_sgd_1e-3_relu_128*64_toy}
	\end{subfigure}
	\caption{SGD optimizer with lr=1e-3. MLP with SELU activations, \#hidden units=128, \#layers=64.}
\end{figure}
\vspace{-10pt}
\begin{figure}[!htbp]
	\vspace{-0pt}
	\begin{subfigure}{0.33\linewidth}
		\vspace{-0pt}
		\centering
		\includegraphics[width=0.90\columnwidth]{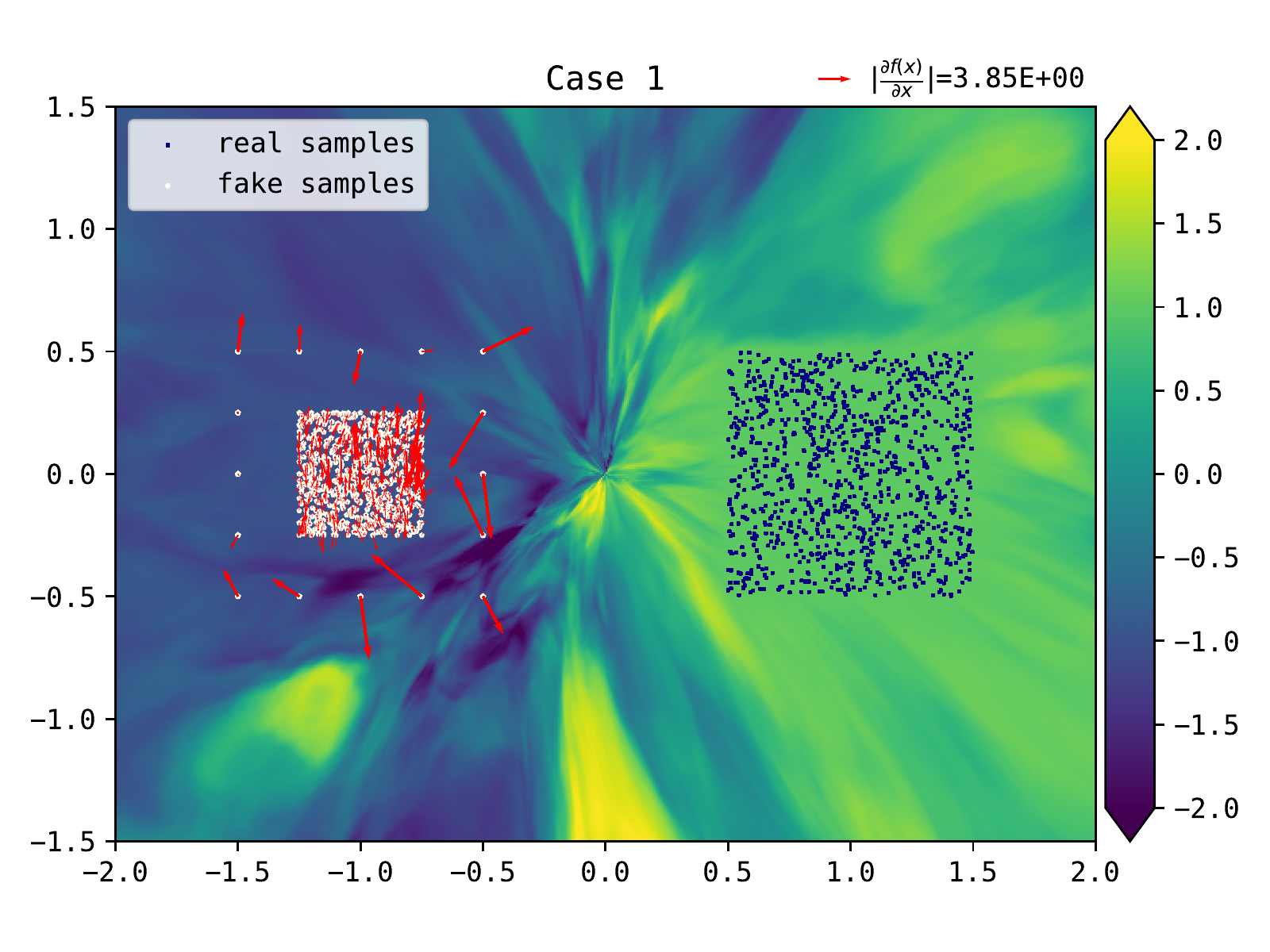}
		%\vspace{-10pt}
		%\caption{Activation: RELU; Optimizer: ADAM; Small LR}
		\label{fig_case1_lsgan_sgd_1e-4_relu_128*64_toy}
	\end{subfigure}
	\begin{subfigure}{0.33\linewidth}
		\vspace{-0pt}
		\centering
		\includegraphics[width=0.90\columnwidth]{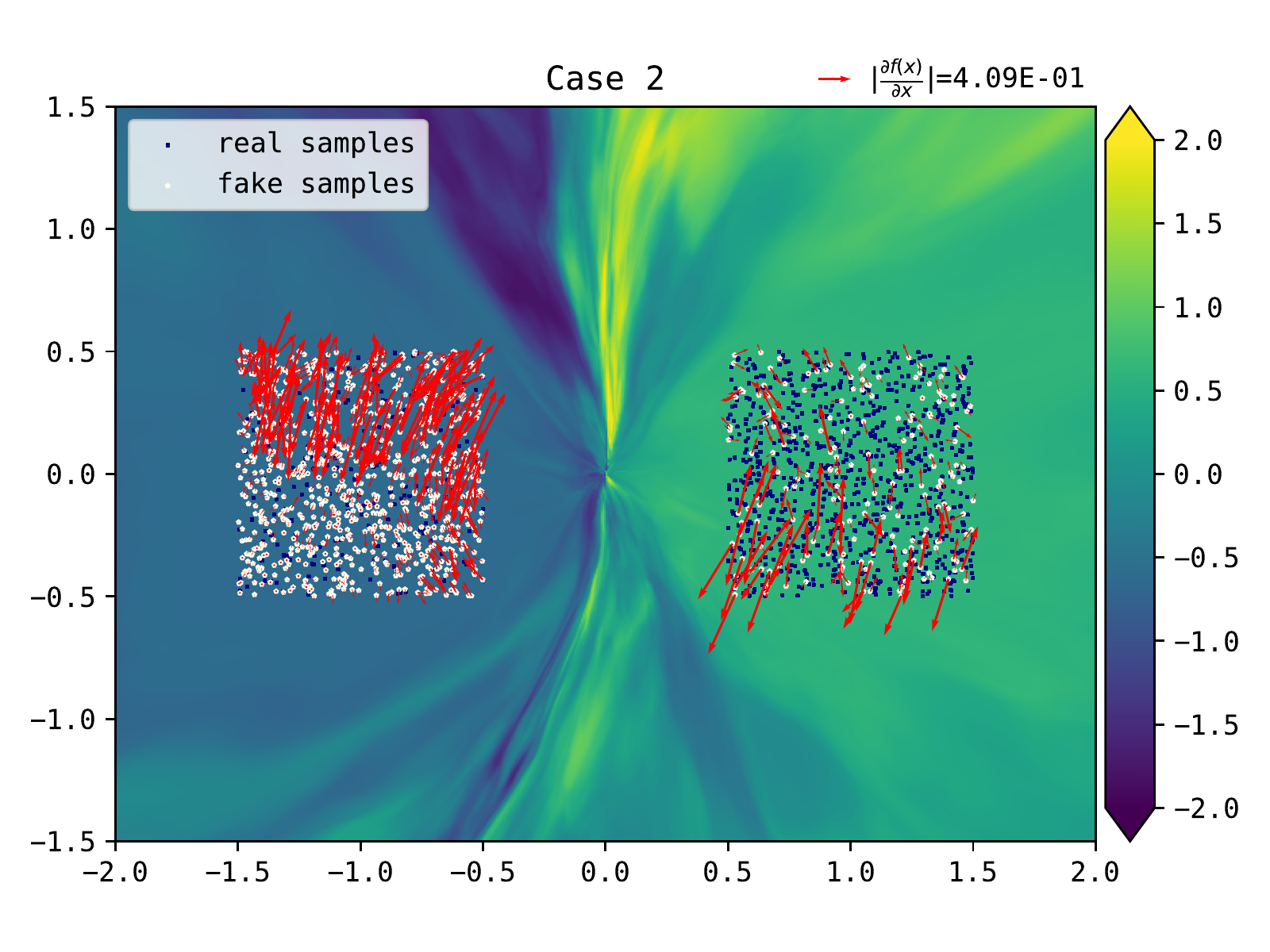}
		%\vspace{-10pt}
		%\caption{Activation: RELU; Optimizer: ADAM; Small LR}
		\label{fig_case2_lsgan_sgd_1e-4_relu_128*64_toy}
	\end{subfigure}
	\begin{subfigure}{0.33\linewidth}
		\vspace{-0pt}
		\centering
		\includegraphics[width=0.90\columnwidth]{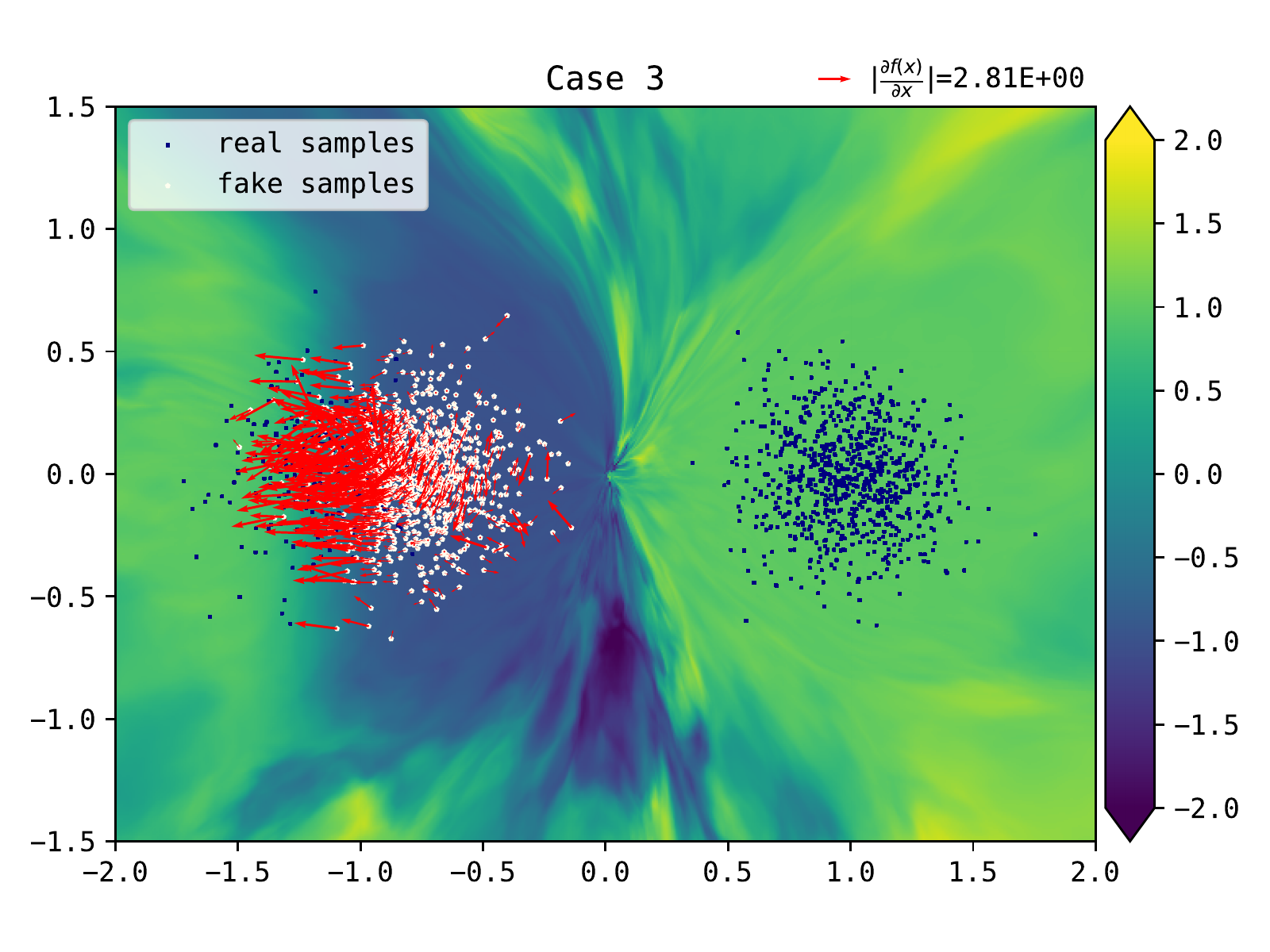}
		%\vspace{-10pt}
		%\caption{Activation: RELU; Optimizer: ADAM; Small LR}
		\label{fig_case3_lsgan_sgd_1e-4_relu_128*64_toy}
	\end{subfigure}
	\caption{SGD optimizer with lr=1e-4. MLP with SELU activations, \#hidden units=128, \#layers=64.}
\end{figure}

\section{Various  $\phi(x)$  and  $\varphi(x)$ that satisfies Eq.~\ref{eq_solvable}}

For Lipschitz-continuity condition based GANs, $\phi(x)$ and $\varphi(x)$ are required to satisfy Eq.~\ref{eq_solvable}. Eq.~(\ref{eq_solvable}) is actually quite general and there exists many other instances, e.g., $\phi(x)=\varphi(-x)=x$, $\phi(x)=\varphi(-x)=-\log(\sigma(-x))$, $\phi(x)=\varphi(-x)=x+\sqrt{x^2+1}$, $\phi(x)=\varphi(-x)=\exp(x)$, etc. We plot these instances of $\phi(x)$ and $\varphi(x)$ in Figure \ref{fig_function_curve}. 

Generally, it is feasible to set $\phi(x)=\varphi(-x)$. Note that rescaling and offsetting along the axes are trivial operation to found more $\phi(x)$ and $\varphi(x)$ within a function class, and linear combination of two or more $\phi(x)$ or $\varphi(x)$ from different function classes also keep satisfying Eq.~\ref{eq_solvable}. 

\begin{figure}[!htbp]
	\centering
	\vspace{-0pt}
	\includegraphics[width=1.0\columnwidth]{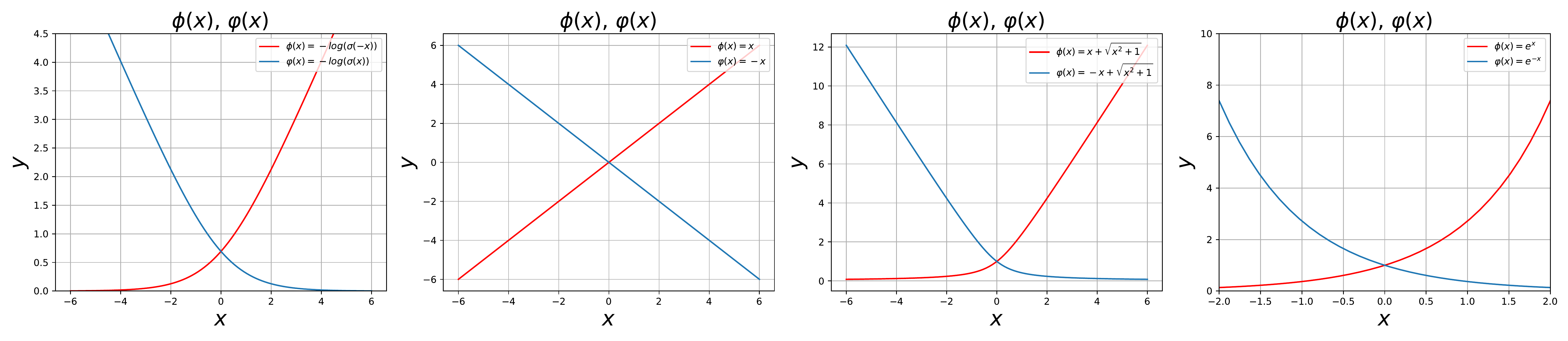}
	\vspace{-8pt}
	\caption{Various $\phi(x)$ and $\varphi(x)$ that satisfies Eq.~\ref{eq_solvable}.}
	\label{fig_function_curve} 
\end{figure}

\section{Generated images and Training Curves}
Training curves on Tiny ImageNet are plotted in Figure \ref{fig:Tinycurve}. And comparisons on the visual results among different objectives are also provided in Figure \ref{oxford}, Figure \ref{cifar10} and Figure \ref{imagenet}.    
% \begin{figure}[t]
% 	\centering
% 	\vspace{-30pt}
% 	\centering
% 	\includegraphics[width=0.80\columnwidth]{figures/training_curve_tinyscore_fid.pdf}
% 	\vspace{-5pt}
% 	\caption{FID training curves of different objective functions on Tiny Imagenet Dataset.}\label{fig:Tinycurve}
% \end{figure}
% \newpage
\begin{figure}[!htbp]
	\centering
	\begin{subfigure}{0.45\linewidth}
		\centering
		\includegraphics[width=0.95\columnwidth]{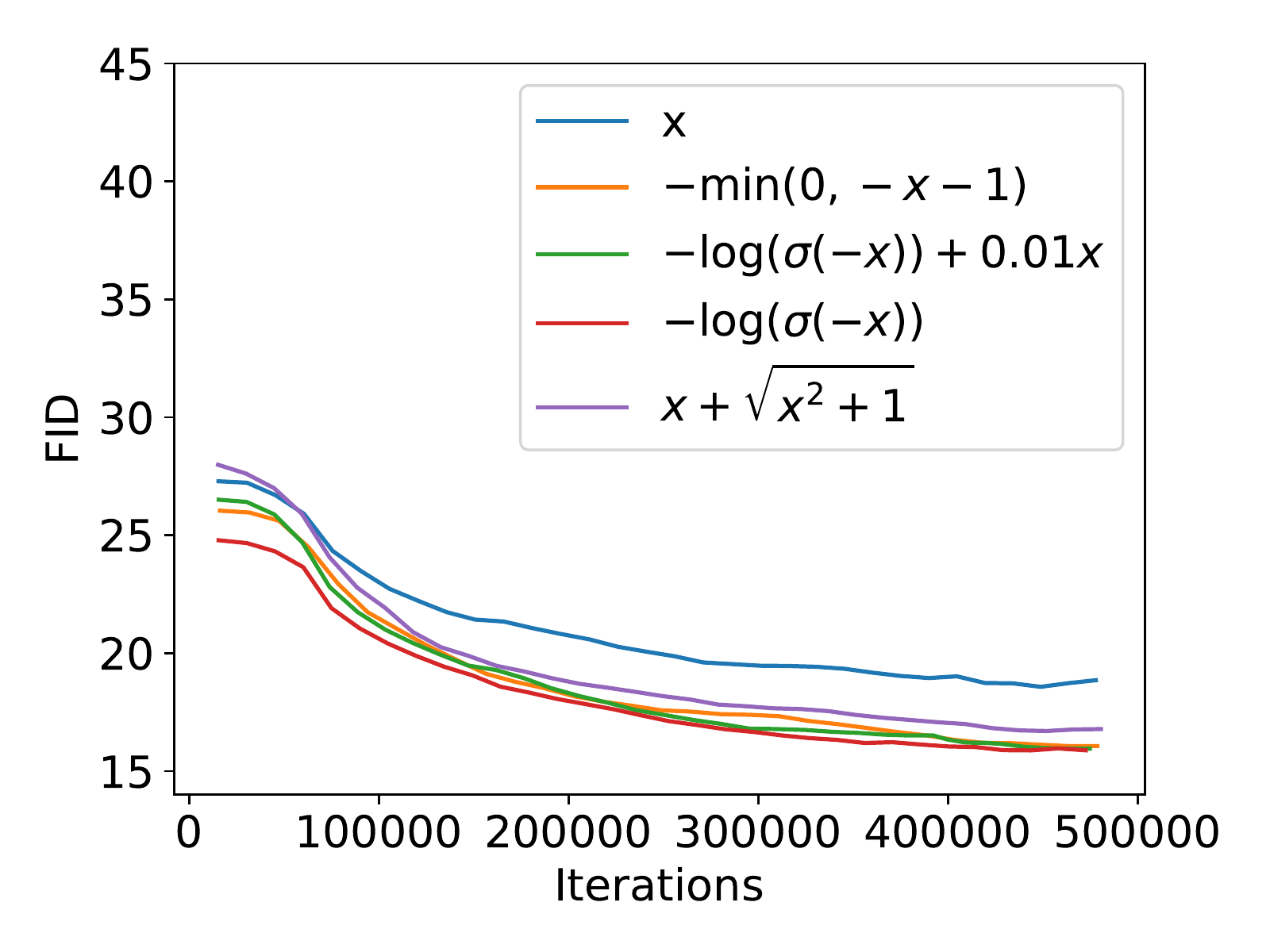}
		\vspace{-3pt}
		\caption{FID training curve}
		\label{fig:Tinycurvefid}
	\end{subfigure}
	\begin{subfigure}{0.45\linewidth}
		\centering	
		\includegraphics[width=0.95\columnwidth]{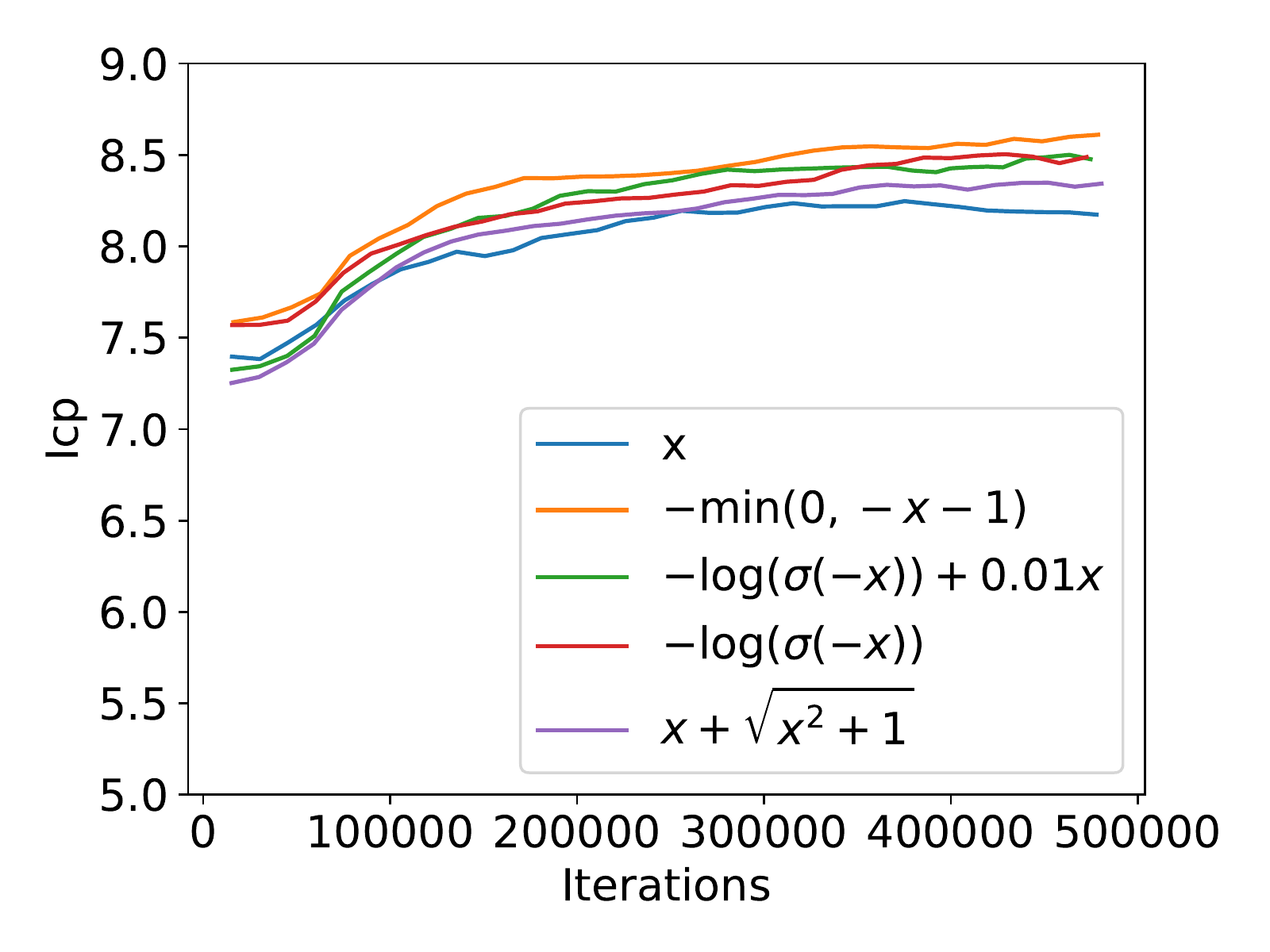}
		\vspace{-3pt}
		\caption{Inception Score training curve}
		\label{fig:Tinycurveicp}
	\end{subfigure}	
	\vspace{-5pt}
	\caption{FID and ICP (Inception Score) training curves of different objectives on Tiny ImageNet.}
	\label{fig:Tinycurve}
	\vspace{-3pt}
\end{figure}

\begin{figure}[!htbp]
	\centering
% 	\vspace{-20pt}
	\begin{subfigure}{0.45\linewidth}
		\centering
		\includegraphics[width=0.95\columnwidth]{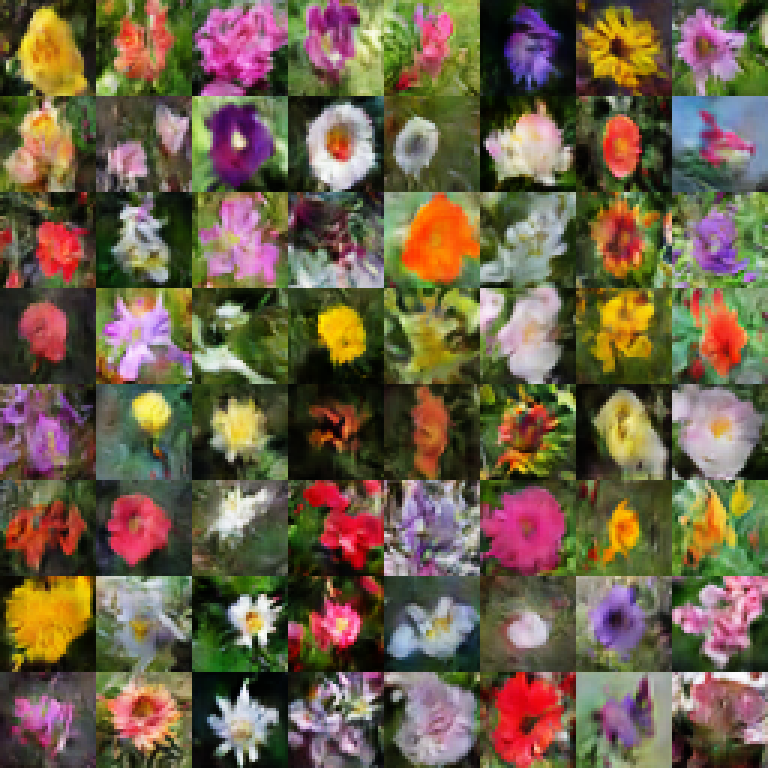}
% 		\vspace{-3pt}
		\caption{$-\log(\sigma(-x))$}
	\end{subfigure}
	\begin{subfigure}{0.45\linewidth}
		\centering	
		\includegraphics[width=0.95\columnwidth]{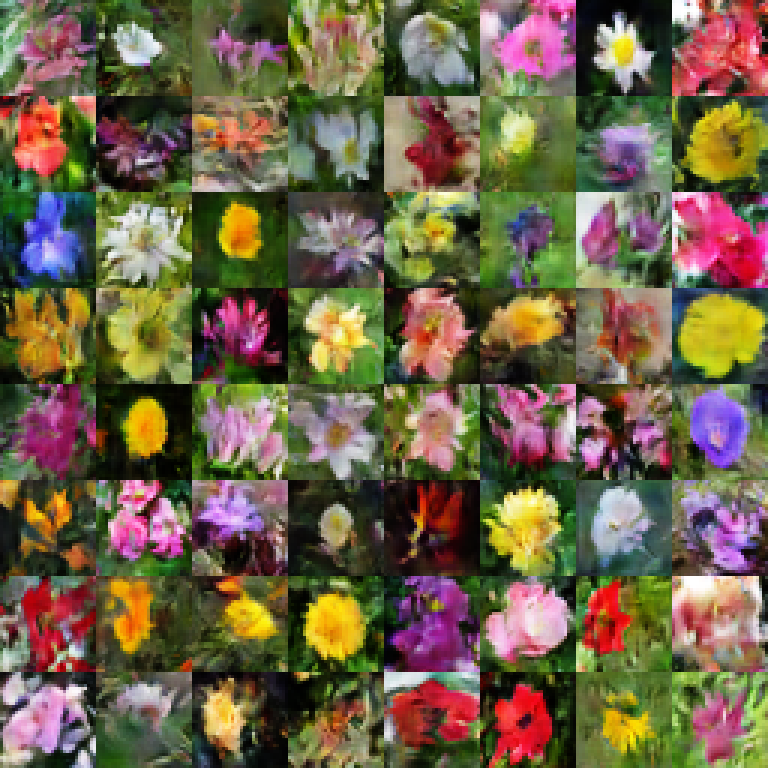}
% 		\vspace{-3pt}
		\caption{$-\log(\sigma(-x))+0.01x$}
	\end{subfigure}	
% 	\vspace{-5pt}
	\begin{subfigure}{0.45\linewidth}
	    \centering
	    \includegraphics[width=0.95\columnwidth]{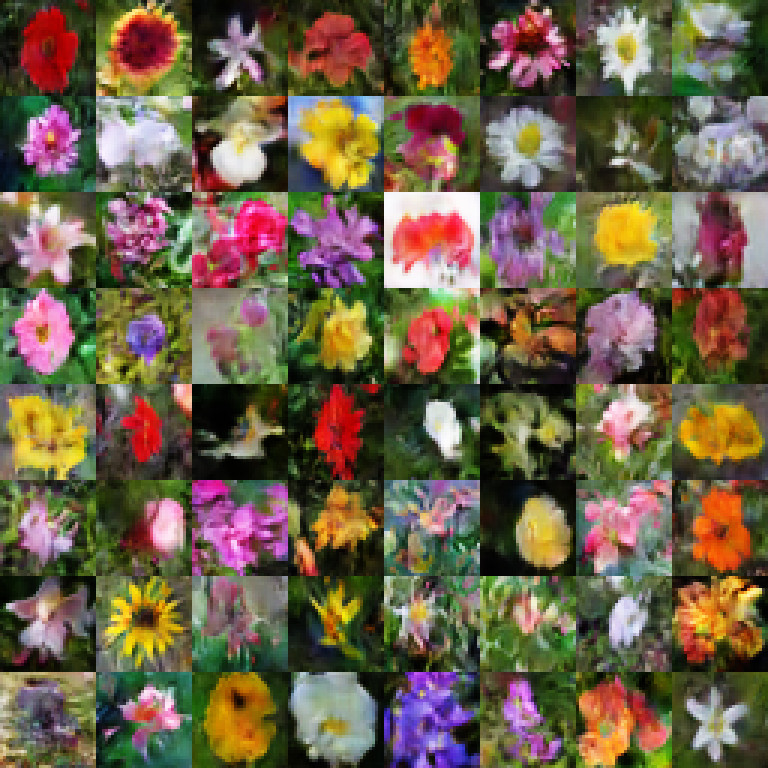}
	   % \vspace{-3pt}
	    \caption{$x$}
	\end{subfigure}
	\begin{subfigure}{0.45\linewidth}
		\centering	
		\includegraphics[width=0.95\columnwidth]{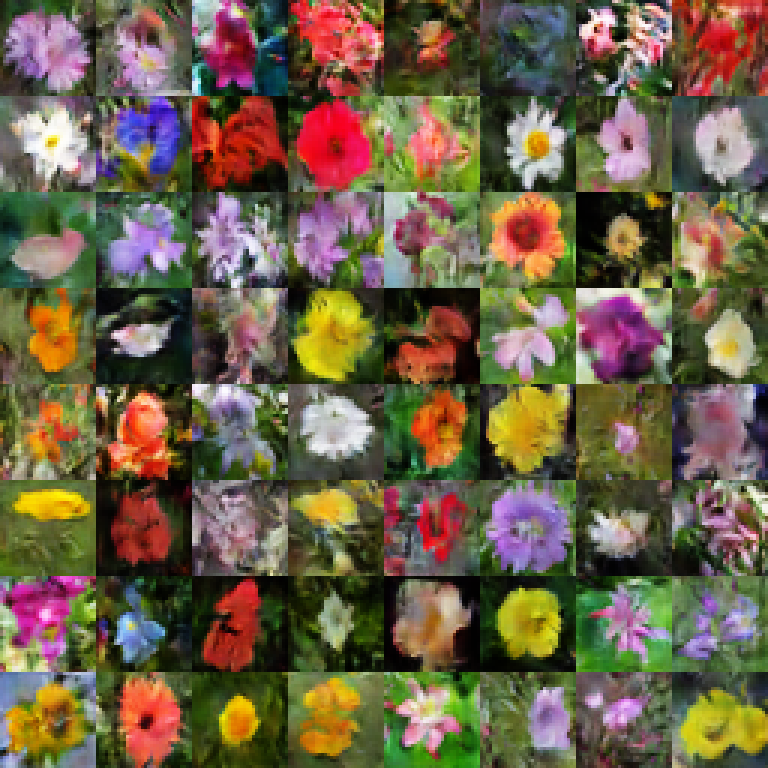}
% 		\vspace{-3pt}
		\caption{$x+\sqrt{x^2+1}$}
	\end{subfigure}	
% 	\vspace{-5pt}
	\begin{subfigure}{0.45\linewidth}
    	\centering
	    \includegraphics[width=0.95\columnwidth]{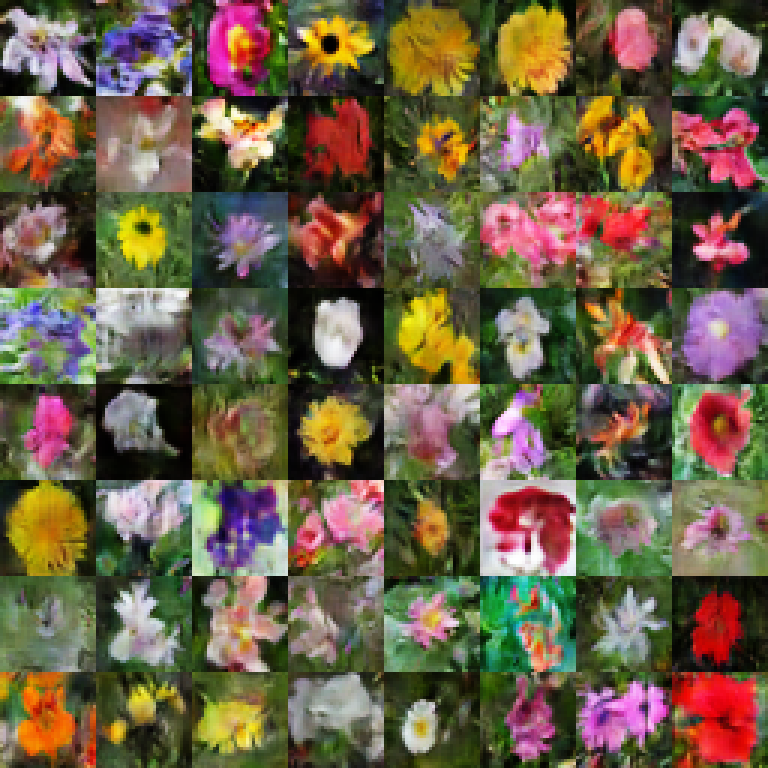}
	   % \vspace{-3pt}
	    \caption{$\exp(x)$}
	\end{subfigure}
	\begin{subfigure}{0.45\linewidth}
		\centering	
		\includegraphics[width=0.95\columnwidth]{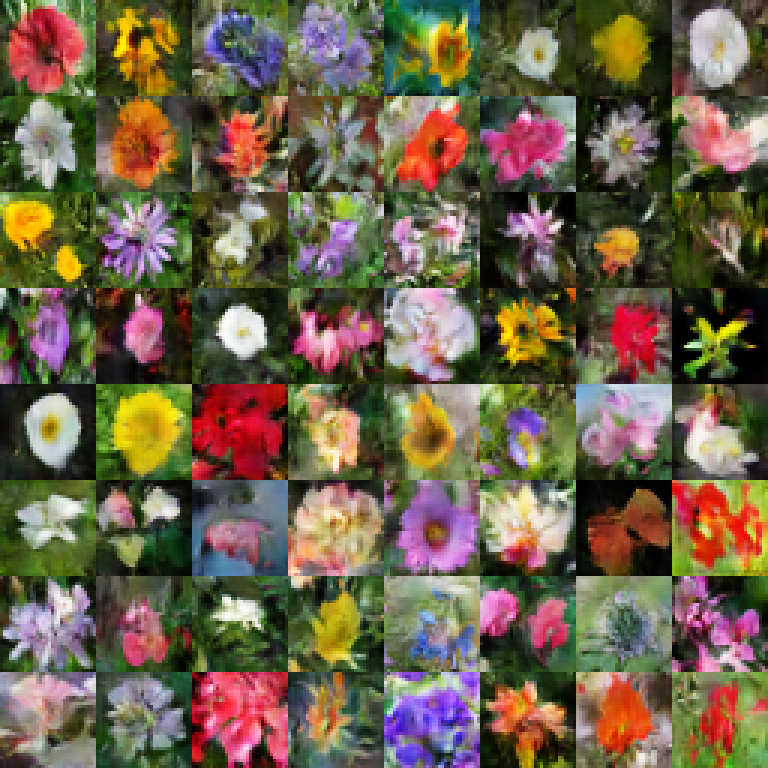}
% 		\vspace{-3pt}
		\caption{$-\min(0, -x-1)$}
	\end{subfigure}	
	\caption{Random Samples of Lipschitz GAN trained of different objectives on Oxford 102.}
	\label{oxford}

% 	\vspace{-3pt}
\end{figure}

\begin{figure}[!htbp]
	\centering
% 	\vspace{-20pt}
	\begin{subfigure}{0.45\linewidth}
		\centering
		\includegraphics[width=0.95\columnwidth]{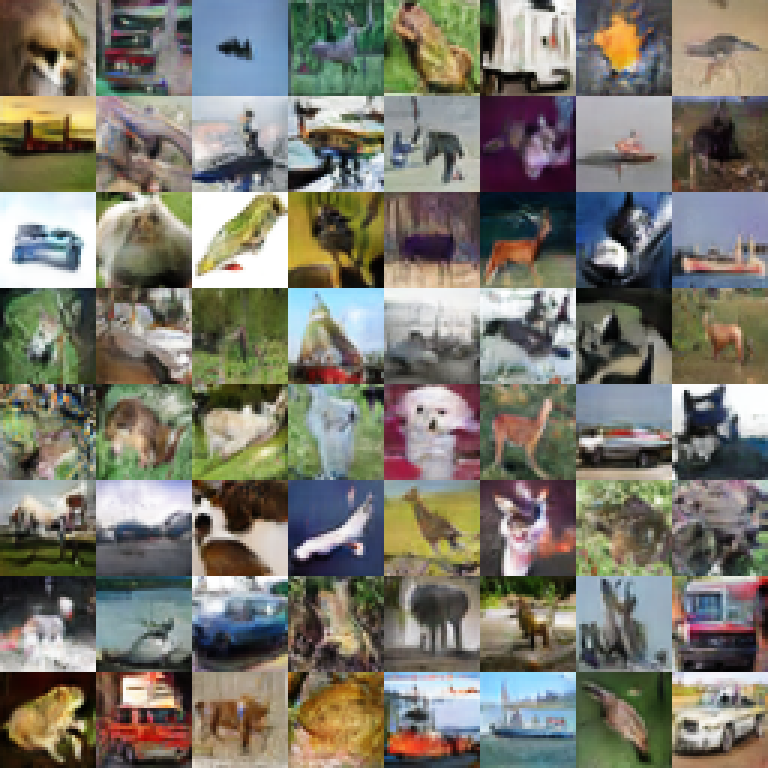}
% 		\vspace{-3pt}
		\caption{$-\log(\sigma(-x))$}
	\end{subfigure}
	\begin{subfigure}{0.45\linewidth}
		\centering	
		\includegraphics[width=0.95\columnwidth]{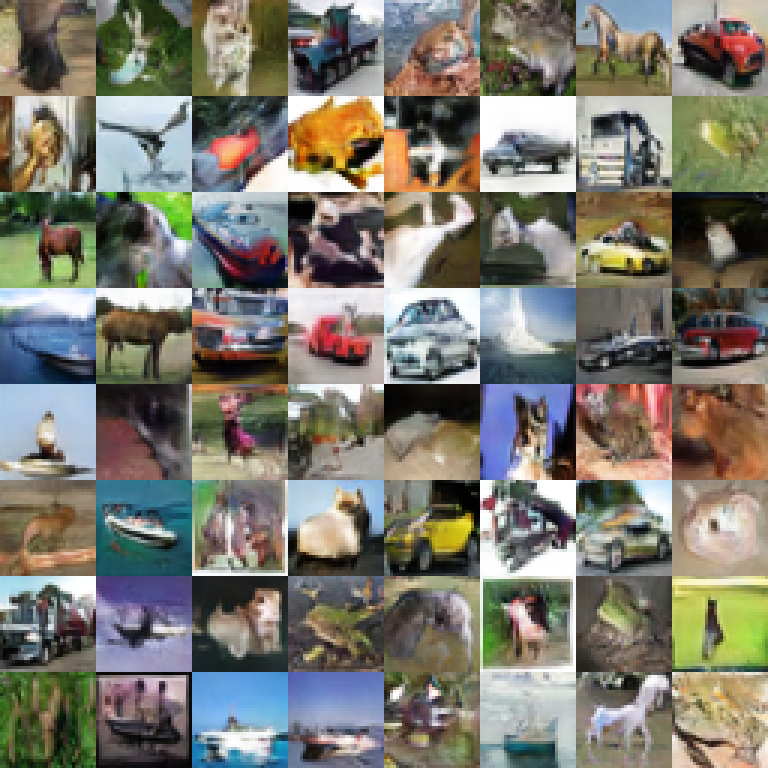}
% 		\vspace{-3pt}
		\caption{$-\log(\sigma(-x))+0.01x$}
	\end{subfigure}	
% 	\vspace{-5pt}
	\begin{subfigure}{0.45\linewidth}
	    \centering
	    \includegraphics[width=0.95\columnwidth]{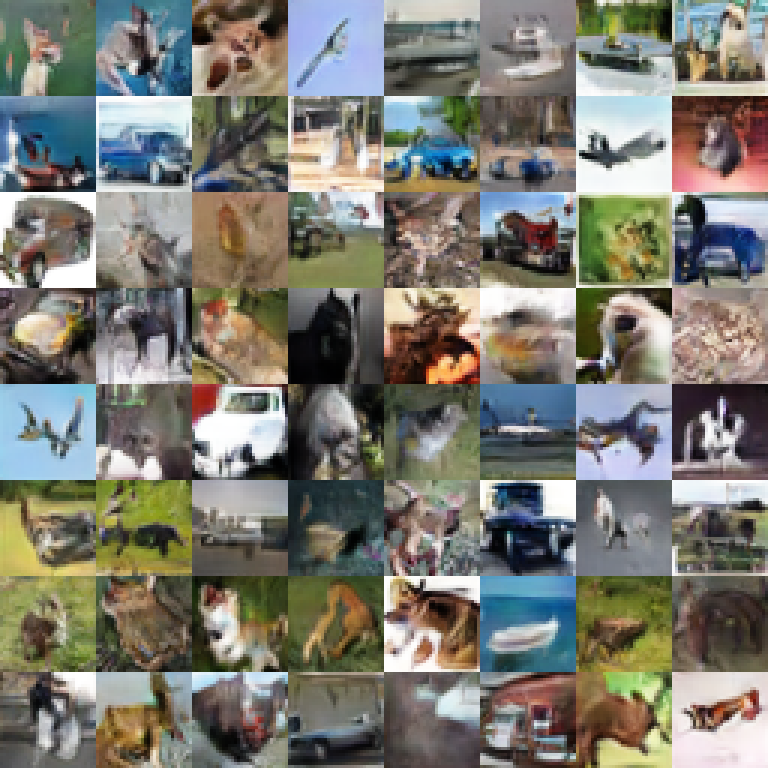}
	   % \vspace{-3pt}
	    \caption{$x$}
	\end{subfigure}
	\begin{subfigure}{0.45\linewidth}
		\centering	
		\includegraphics[width=0.95\columnwidth]{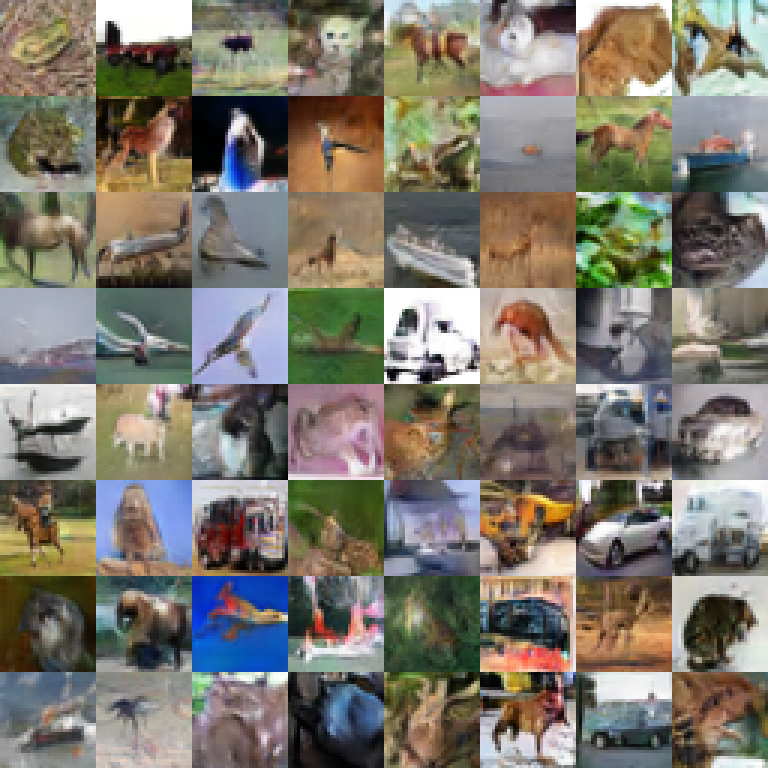}
% 		\vspace{-3pt}
		\caption{$x+\sqrt{x^2+1}$}
	\end{subfigure}	
% 	\vspace{-5pt}
	\begin{subfigure}{0.45\linewidth}
    	\centering
	    \includegraphics[width=0.95\columnwidth]{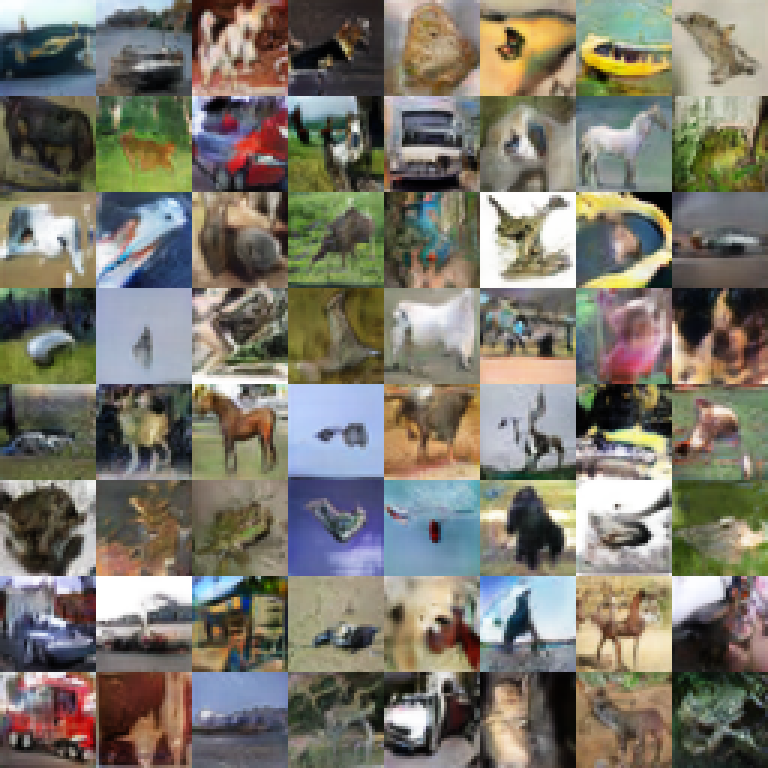}
	   % \vspace{-3pt}
	    \caption{$\exp(x)$}
	\end{subfigure}
	\begin{subfigure}{0.45\linewidth}
		\centering	
		\includegraphics[width=0.95\columnwidth]{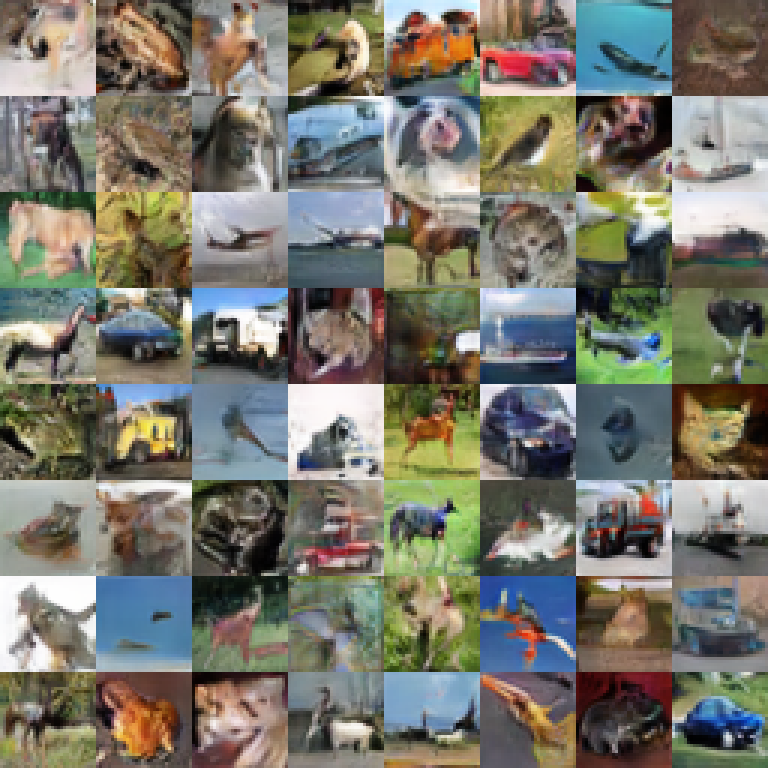}
% 		\vspace{-3pt}
		\caption{$-\min(0, -x-1)$}
	\end{subfigure}	
	\caption{Random Samples of Lipschitz GAN trained of different objectives on Cifar-10.}
	\label{cifar10}
% 	\vspace{-3pt}
\end{figure}

\begin{figure}[!htbp]
	\centering
% 	\vspace{-20pt}
	\begin{subfigure}{0.45\linewidth}
		\centering
		\includegraphics[width=0.95\columnwidth]{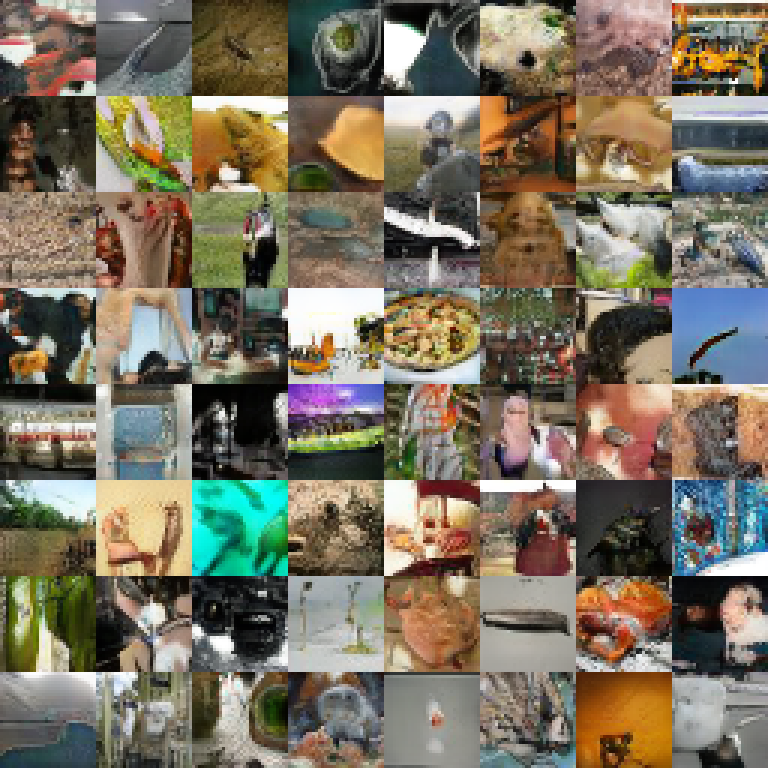}
% 		\vspace{-3pt}
		\caption{$-\log(\sigma(-x))$}
	\end{subfigure}
	\begin{subfigure}{0.45\linewidth}
		\centering	
		\includegraphics[width=0.95\columnwidth]{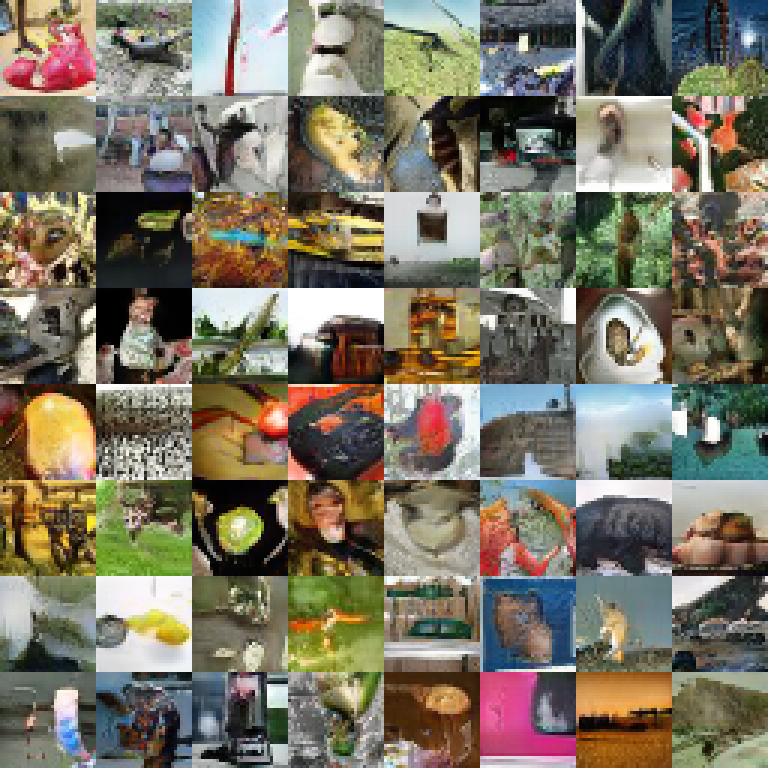}
% 		\vspace{-3pt}
		\caption{$-\log(\sigma(-x))+0.01x$}
	\end{subfigure}	
% 	\vspace{-5pt}
	\begin{subfigure}{0.45\linewidth}
	    \centering
	    \includegraphics[width=0.95\columnwidth]{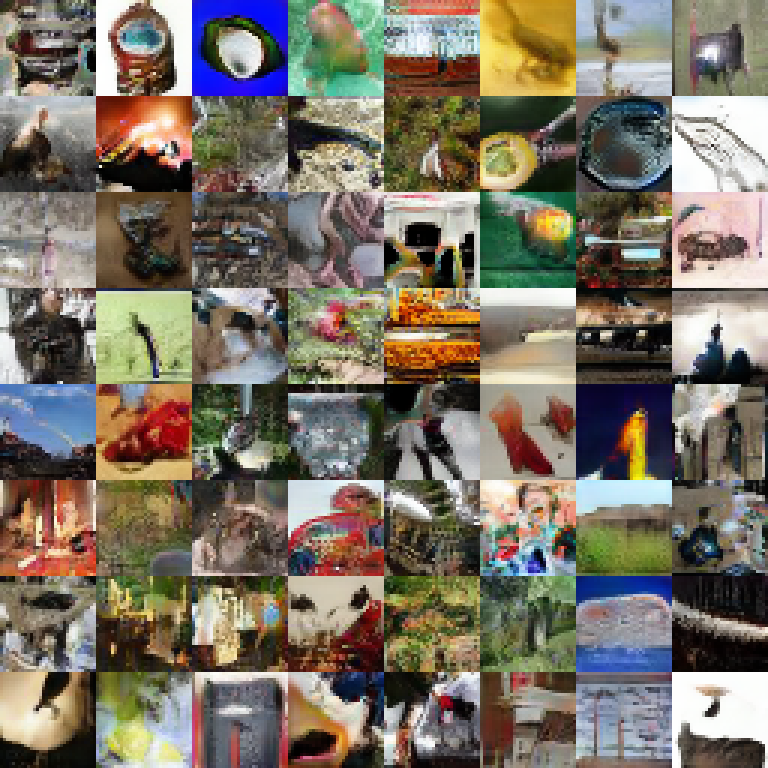}
	   % \vspace{-3pt}
	    \caption{$x$}
	\end{subfigure}
	\begin{subfigure}{0.45\linewidth}
		\centering	
		\includegraphics[width=0.95\columnwidth]{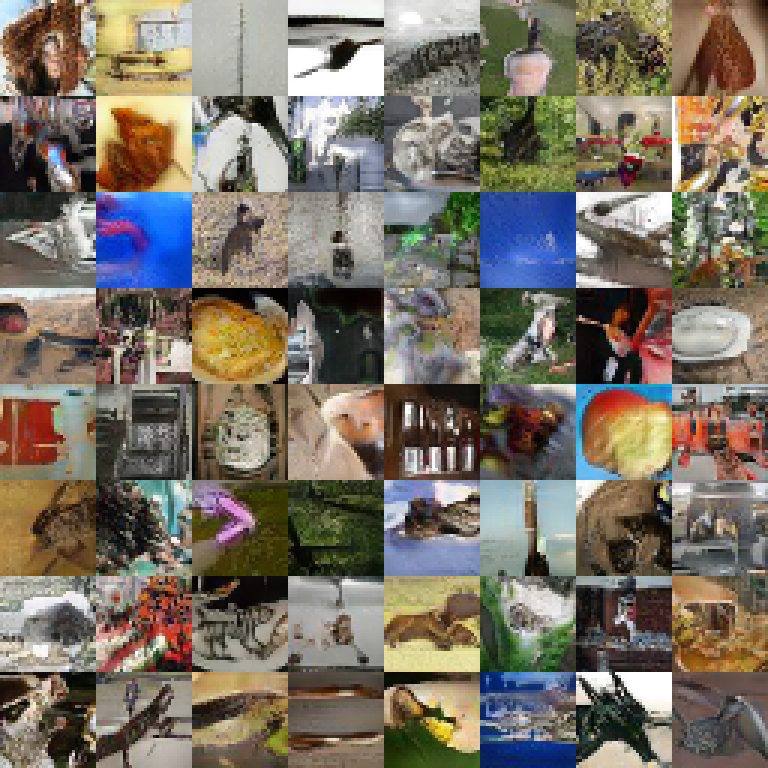}
% 		\vspace{-3pt}
		\caption{$x+\sqrt{x^2+1}$}
	\end{subfigure}	
% 	\vspace{-5pt}
	\begin{subfigure}{0.45\linewidth}
    	\centering
	    \includegraphics[width=0.95\columnwidth]{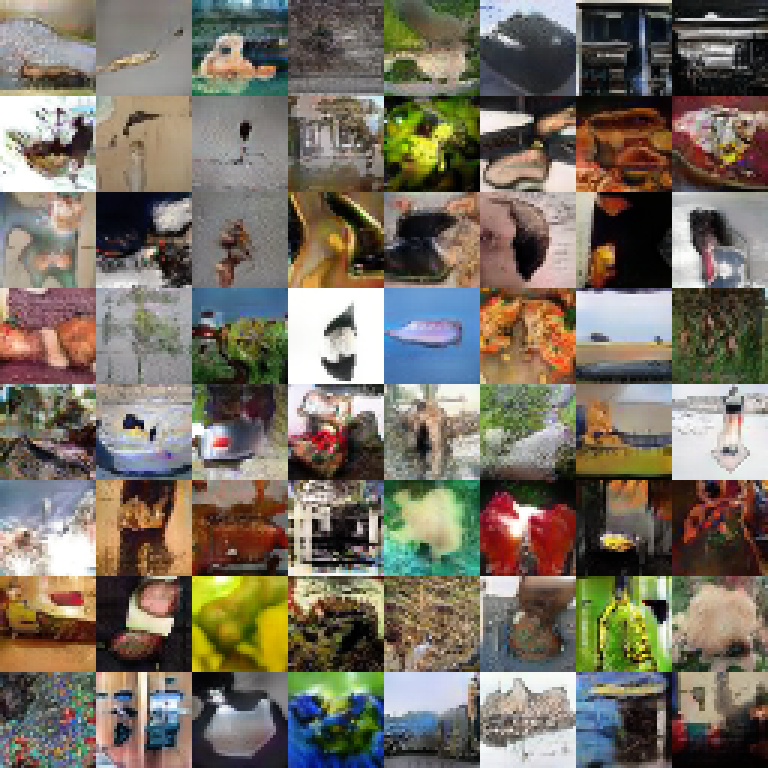}
	   % \vspace{-3pt}
	    \caption{$\exp(x)$}
	\end{subfigure}
	\begin{subfigure}{0.45\linewidth}
		\centering	
		\includegraphics[width=0.95\columnwidth]{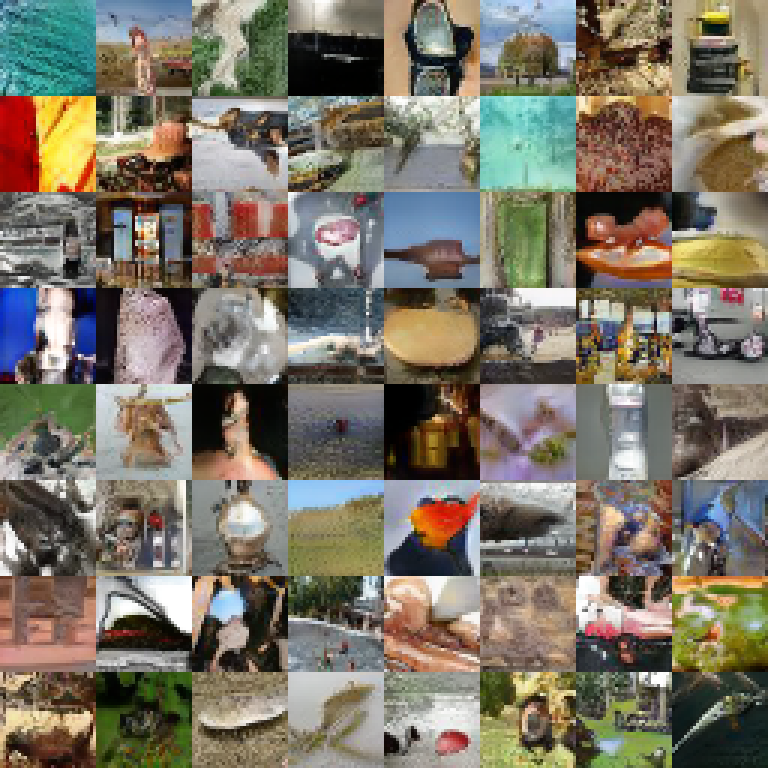}
% 		\vspace{-3pt}
		\caption{$-\min(0, -x-1)$}
	\end{subfigure}	
	\caption{Random Samples of Lipschitz GAN trained of different objectives on Tiny Imagenet.}
	\label{imagenet}
% 	\vspace{-3pt}
\end{figure}

\begin{figure}[!htbp]
\centering
\includegraphics[width=0.85\columnwidth]{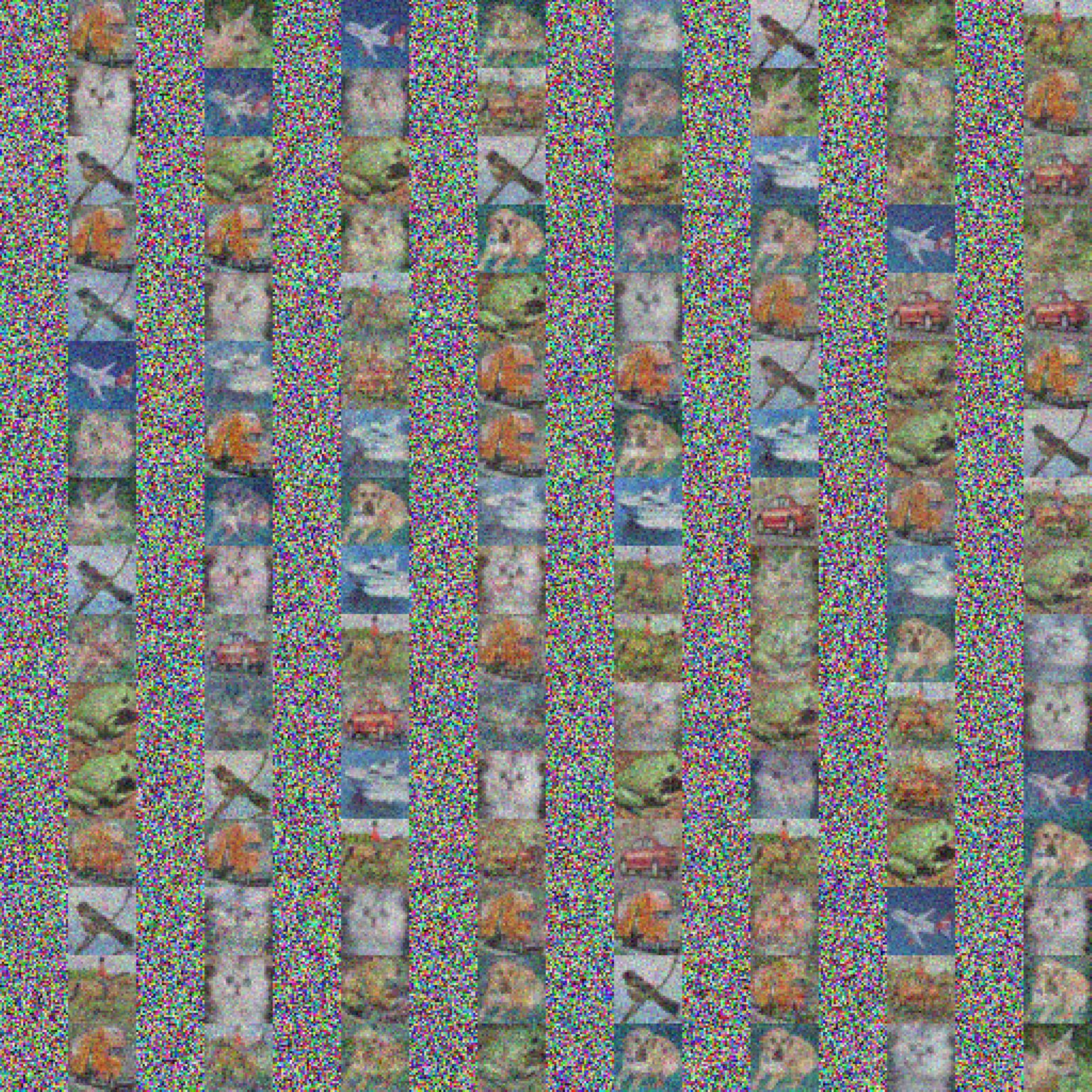}

\vspace{10pt}

\includegraphics[width=0.85\columnwidth]{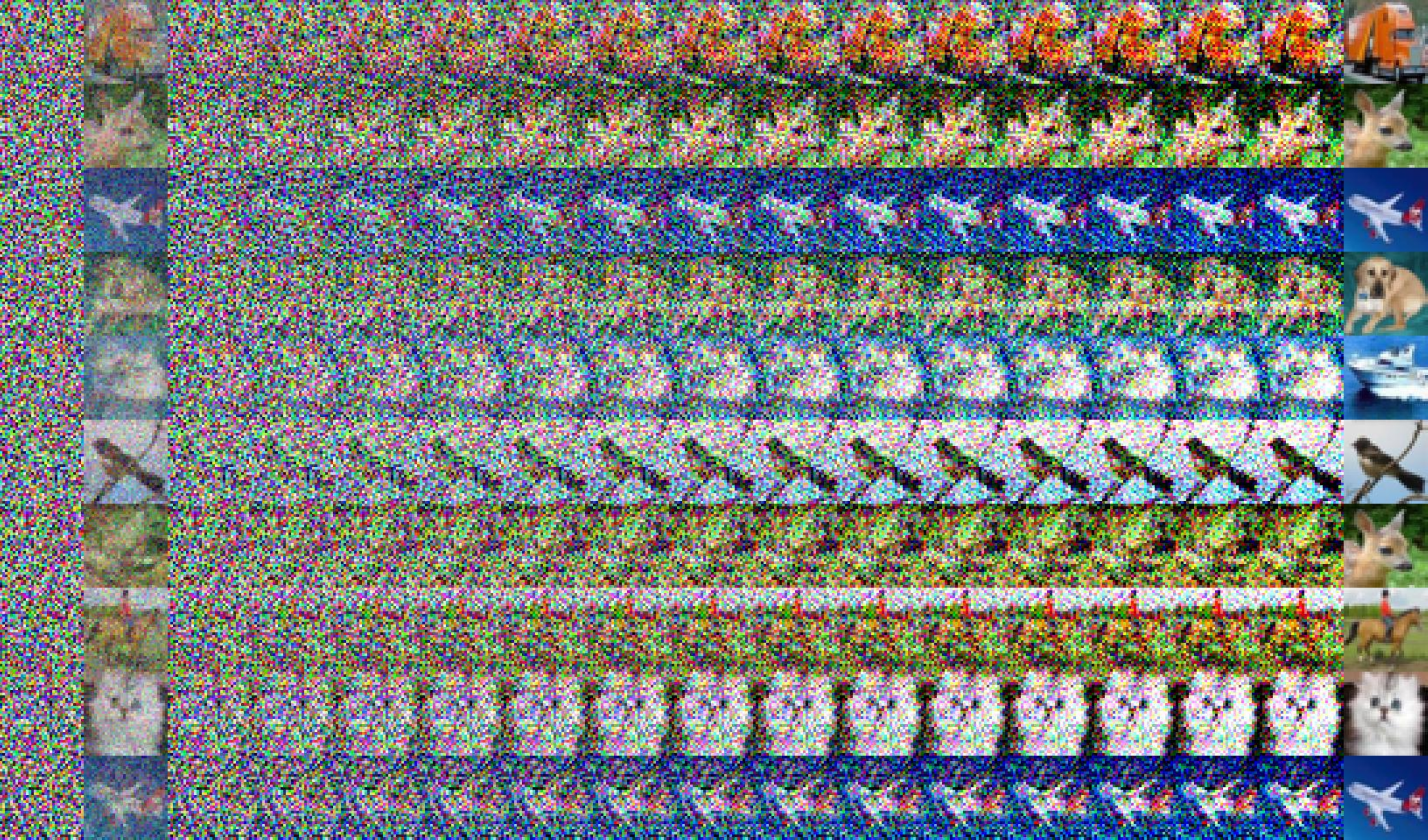}
\vspace{3pt}
\caption{The gradient of Lipschitz-continuity condition based GANs with real world data, where $P_r$ consists of ten images and $P_g$ is Gaussian noise. Up: Each odd column are $x \sim \pgen$ and the nearby column are their gradient $\nabla_{\!x} \ff(x)$. Down: the leftmost in each row are $x \sim \pgen$, the second are their gradients $\nabla_{\!x} \ff(x)$, the interior are $x+\epsilon\cdot\nabla_{\!x} \ff(x)$ with increasing $\epsilon$, and the rightmost are the nearest $y \sim \pdata$.} 
\label{fig_gradient_direction_continuous}
\end{figure}

%\newpage
\section{Proof of Theorem \ref{lemma} and the Necessity of Euclidean distance}
\label{lip_direction}
In this section, we delve deeply into the relationship between gradient properties and different norms in Lipschitz-continuity condition. We will prove Theorem \ref{lemma}, \emph{i.e.} Lipschitz continuity with $l_2$-norm (Euclidean Distance) can guarantee the gradient direction of $\nabla \ff(x)$, and at the same time, demonstrate that the other norms do not have this property. 
%
%A brief review on Lipschitz continuity: a function $f$ is k-Lipschitz over a set $\mathcal{S}$ with respect to a norm $\big\Vert . \big\Vert_p$, means that, for all $a$, $b \in \mathcal{S}$ there is $f(a) - f(b) \leq k\big\Vert a - b \big\Vert_p$. %Lemma \ref{lem_gradient_direction} states that the gradient direction of optimal function under Lipschitz-continuity condition with respect to $\big\Vert . \big\Vert_2$. 
To start with, we give the proof of Theorem \ref{lemma} in the following. 
\vspace{-2pt}
\begin{proof} $ $\newline$ $\newline
Let $(x, y)$ be such that $x\neq y $, and we define $x_t = x+ t\cdot (y-x)$ with $t \in [0,1]$. We claim that: if $f(x)$ is k-Lipschitz with respect to $\big\Vert . \big\Vert_p$ and $f(y)-f(x) = k\big\Vert x - y \big\Vert_p$, then $f(x_t) = f(x)+t\cdot k\big\Vert x - y \big\Vert_p$. 

As we know $f(x)$ is k-Lipschitz, with the property of norms, we have
%\small{
\begin{align}\label{eq:linear_interopation}
f(y)-f(x) &= f(y)-f(x_t)+ f(x_t)-f(x) \nonumber \\
&\leq f(y)-f(x_t)+k\Vert x_t-x\Vert_p %\nonumber  \\%&\small\text{---$f(x)$ is k-Lipschitz} \\
= f(y)-f(x_t)+t\cdot k\Vert x - y\Vert_p \nonumber \\%&\small\text{--- with the property of norms}\\
&\leq k\Vert y-x_t\Vert_p+t\cdot k\Vert x - y\Vert_p \nonumber %\\%&\small\text{--- $f(x)$ is k-Lipschitz} \\
= k \cdot (1-t)\Vert x - y\Vert_p+t \cdot k\Vert x - y\Vert_p \nonumber \\
& = k \Vert x - y\Vert_p. %&\small\text{--- with the property of norms}
\end{align}
%}
Given $f(x)-f(y) = k\Vert x - y \Vert_p$, it implies all the inequalities need to be equalities. Therefore, $f(x_t) = f(x)+t\cdot k\Vert x - y \Vert_p$. 

It is clear that: given $f(x)$ is k-Lipschitz with respect to $\Vert.\Vert_2$, if $f(x)$ is differentiable at $x_t$, then $\Vert \nabla f(x_t) \Vert_2 \leq k$.
With $f(x_t) = f(x)+t\cdot k\Vert x - y \Vert_2$, the directional derivative of $f(x)$ on the direction $v=\frac{y-x}{\Vert y-x \Vert_2}$ at $x_t$ is equal to $k$,
{
\small
\begin{align}
\tpdv{f(x_t)}{v} &= \lim\limits_{h\rightarrow0} \frac{f(x_t+hv)-f(x_t)}{h}=\lim\limits_{h\rightarrow0} \frac{f(x_t+h\frac{y-x}{\Vert y-x \Vert_2})-f(x_t)}{h} \nonumber\\ 
&=\lim\limits_{h\rightarrow0}\frac{f(x_{t+\frac{h}{\Vert y-x \Vert_2}})-f(x_t)}{h} =\lim\limits_{h\rightarrow0}\frac{\frac{h}{\Vert y-x \Vert_2}\cdot k\Vert y-x \Vert_2}{h}=k.
\end{align}
}
Note that $\Vert v \Vert_2 = \Vert \frac{y-x}{\Vert y -x \Vert_2} \Vert_2 = 1$, \emph{i.e.} $v$ is a unit vector. Now,
\begin{align}
k^2 =k\tpdv{f(x_t)}{v} =k\left<v,\nabla f(x_t)\right>= \left<kv, \nabla f(x_t) \right> \leq \Vert kv \Vert_2\Vert \nabla f(x_t) \Vert_2 = k^2.
\end{align}
As the equality holds only when $\nabla f(x_t) = kv = k\frac{y-x}{\Vert y -x \Vert_2}$, we prove that $\nabla f(x_t) = k\frac{y-x}{\Vert y -x \Vert_2}$.
\end{proof}
%$f(x)$ is k-Lipschitz with respect to $\Vert.\Vert_2$
Above proof utilizes the property that $\Vert \nabla f(x_t) \Vert_2 \leq k$, which is derived from that $f(x)$ is k-Lipschitz with respect to $\Vert.\Vert_2$. However, other norms do not hold this property. Specifically, according to the theory in \citep{OLOco}: if a convex and differentiable function $f$ is k-Lipschitz over $\mathcal{S}$ with respect to norm $\Vert.\Vert_p$, then the Lipschitz continuity actually implies a bound on the dual norm of gradients, \emph{i.e.} $\Vert \nabla f\Vert_q \leq k$. Here $\Vert.\Vert_q$ is the dual norm of $\Vert.\Vert_p$, which satisfies the equation that $\frac{1}{p}+\frac{1}{q} = 1$. 
As we could notice, a norm is equal to its dual norm if and only if $p = 2$. Switching to $l_p$-norm with $p\neq2$, it is actually bounding the $l_q$-norm of the gradients. However, bounding the $l_q$-norm of the gradients does not guarantee the gradient direction at fake samples point towards real samples. A counter-example is provided as follows. 

Consider a function $g(x,y)=x+y$ on $\mathbb{R}^2$. $\forall$ $p_1=(x_1,y_1)$, $p_2=(x_2,y_2)$, there is $g(p_1)-g(p_2) = g(x_1,y_1)-g(x_2,y_2) = (x_1-x_2)+(y_1-y_2) \leq |x_1-x_2|+ |y_1-y_2| = \Vert p_1-p_2\Vert_1 $, which means $g$ is a 1-Lipschitz function with respect to $l_1$-norm. According to above analysis, the dual norm of $\nabla g$ is bounded, \emph{i.e.} $\Vert \nabla g \Vert_\infty \leq 1$. Actually $\nabla g$ is equal to $(1,1)$ at every point in $\mathbb{R}^2$ with $\Vert \nabla g \Vert_\infty = 1$. Selecting two points $ A \!=\! (0,0)$ and $B \!=\! (2,1)$, we have $g(A)\!-\!g(B) \!=\! \Vert A\!-\!B\Vert_1$, however, $\nabla g(A)$ $\!=\! (1,1)$ is not pointing towards $B$.

\section{On the implementation of k-Lipschitz for GANs}

Typical techniques for enforcing k-Lipschitz includes: spectral normalization \citep{sngan}, gradient penalty \citep{wgangp}, and Lipschitz penalty \citep{wganlp}. Before moving into the detailed discussion of these methods, we would provide several important notes in the first place. %Here we would like to provide several important notes on enforcing of k-Lipschitz. %

\textbf{Firstly}, enforcing k-Lipschitz in the blending-region of $\pgen$ and $\pdata$ is actually sufficient. Define $B(\mu,\nu) = \{\hat{x}=x \cdot t + y \cdot (1-t) \mid x\!\sim\!\mu\!\land \!y\!\sim\! \nu \!\land \!t \in [0,1] \}$. It is clear that $f(x) \text{ is 1-Lipschitz in }$ $B(\mu,\nu)$ implies $f(x)\!-\!f(y) \leq d(x, y), \forall x \in \mu, \forall y \in \nu$. Thus, it is a sufficient constraint for Wasserstein distance in Eq.~\ref{eq_w_dual_form_1}. 
In fact, $f(x) \text{ is k-Lipschitz in } B(\pgen,\pdata)$ is also a sufficient condition for all properties described in Lipschitz-continuity condition based GANs (Section \ref{sec_lip}). 

\textbf{Secondly}, enforcing k-Lipschitz with regularization would provide a dynamic Lipschitz constant $k$.

\vspace{2pt}
\begin{theorem}\label{lem_lipschitz_constant}
With Wasserstein GAN objective, we have $\min_{f \in \mathcal{F}_\text{k-Lip}} \dloss(f) = k \cdot \min_{f \in \mathcal{F}_\text{1-Lip}} \dloss(f)$. 
\end{theorem} 
\vspace{-5pt}

Assuming we know and can control the Lipschitz constant $k$ of $f(x)$, by introducing a loss, saying square loss, on $k$ respecting to a constant $k_0$, the total loss of the discriminator (critic) becomes $J(k) \triangleq \min_{f \in \mathcal{F}_\text{k-Lip}} \dloss(f) + \lambda \cdot (k-k_0)^2$. 
With Lemma \ref{lem_lipschitz_constant}, let $\alpha\!=\!-\min_{f \in \mathcal{F}_\text{1-Lip}} \dloss(f)$, then $J(k)\!=\!-k\cdot\alpha\!+\!\lambda\cdot (k\!-\!k_0)^2$, and $J(k)$ achieves its minimum when $k\!=\!\frac{\alpha}{2\lambda}\!+\!k_0$. When $\alpha$ goes to zero, \emph{i.e.} $\pgen$ converges to $\pdata$, the optimal $k$ decreases. And when $\pgen\!=\!\pdata$, we have $\alpha\!=\!0$ and optimal $k\!=\!k_0$. We choose $k_0\!=\!0$ in our experiments. The similar analysis applies to Lipschitz-continuity condition based GANs and we use $\lambda\cdot k^2$ to enforcing k-Lipschitz for general Lipschitz-continuity condition based GANs. 

\textbf{For practical methods}, though spectral normalization \citep{sngan} recently demonstrates their excellent results in training GANs, spectral normalization is an absolute constraint for Lipschitz over the entire space, i.e., constricting the maximum gradient of the entire space, which is unnecessary. On the other side, we also notice both penalty methods proposed in \citep{wgangp} and \citep{wganlp} are not the exactly implementing the Lipschitz continuity condition, because it does not simply penalty the maximum gradient, but penalties all gradients towards 1, or penalties all these greater than one towards 1.

We found in our experiments that the existing methods including spectral normalization \citep{sngan}, gradient penalty \citep{wgangp}, and Lipschitz penalty \citep{wganlp} all fail to converge to the optimal $\ff(x)$ in many of our synthetic experiments. We thus developed a new method for enforcing k-Lipschitz and we found in our experiments that the new method stably converges to the optimal $\ff(x)$. 

\begin{figure}[!htbp]
	\centering
	\begin{subfigure}{0.45\linewidth}
		\centering
		\includegraphics[width=0.95\columnwidth]{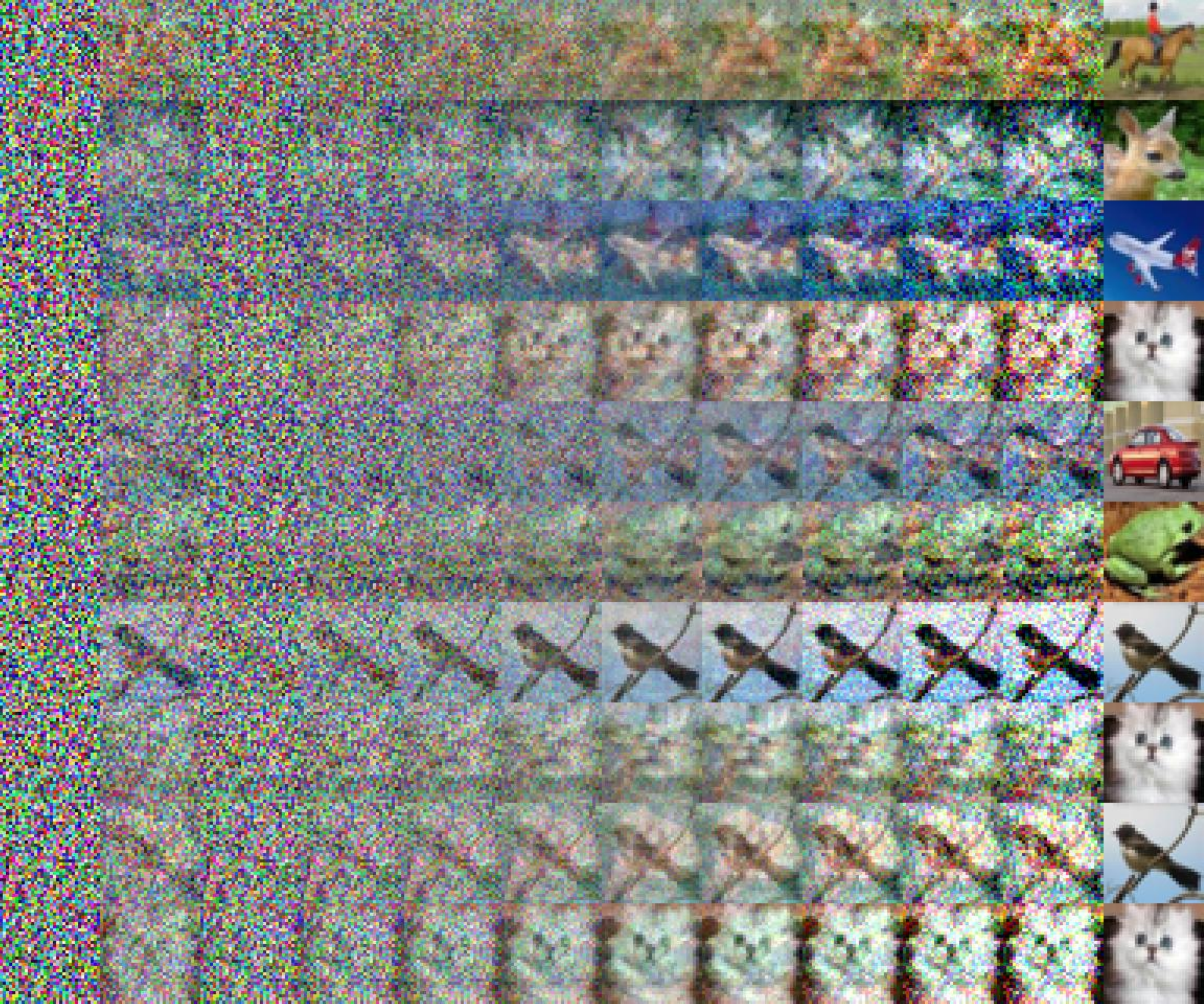}
		\vspace{-3pt}
		\caption{Gradient Penalty}
		\label{fig_gp}
	\end{subfigure}
	\begin{subfigure}{0.45\linewidth}
		\centering	
		\includegraphics[width=0.95\columnwidth]{figures/Case_5-10++_wgans_gp1_0_adam_1e-4_mlpdense256x16_cifar10.pdf}
		\vspace{-3pt}
		\caption{Maximum Gradient Penalty}
		\label{fig_maxgp}
	\end{subfigure}	
	\vspace{-5pt}
	\caption{Comparison between gradient penalty and maximum gradient penalty, with $P_r$ and $P_g$ consist of \textbf{ten real and noise images}, respectively. The leftmost in each row is a $x \sim \pgen$ and the second is its gradient $\nabla_{\!x} \ff(x)$. The interior are $x+\epsilon\cdot\nabla_{\!x} \ff(x)$ with increasing $\epsilon$, which will pass through a real sample, and the rightmost is the corresponding $y \sim \pdata$.}
	\label{fig_gradient_direction_maxgp_compare}
	\vspace{-3pt}
\end{figure}

\paragraph{The new method.}
Note that the practical methods of imposing k-Lipschitz is not the key contribution of this work, and it is far from well-validated. We plan a further work on this topic for a more rigorous study. But for the necessity for understanding our paper and reproducing of experiments, we introduce it as follows.

Combining the idea of spectral normalization and gradient penalty, we developed an new regularization for Lipschitz continuity in our experiments. Spectral normalization is actually constraining the maximum gradient over the entire space. And as we argued previously, enforcing Lipschitz continuity in the blending region is sufficient. Therefore, we propose to restricting the maximum gradient over the blending region:
\begin{align}
J_{\text{maxgp}} = \max_{\hat{x}\sim B(\mu,\nu)} [ \big\lVert \nabla f(x) \big\rVert_2^2] 
\end{align}
In practice, we sample $\hat{x}$ from training batch as in \citep{wgangp,wganlp}. To improve the stability and reduce the biased introduced via batch sampling, we propose the keep track $\hat{x}$ with the maximum $\big\lVert \nabla f(x) \big\rVert_2$. A practical and light weight method is to maintain a list $S_\text{max}$ that has the currently highest (top-k) $\big\lVert \nabla f(x) \big\rVert_2$ (initialized with random $\hat{x}$ samples), using the $S_\text{max}$ as part of the batch estimation of $J_{\text{maxgp}}$, and update the $S_\text{max}$ after each batch updating of the discriminator. In our experiment, $S_\text{max}$ takes 1/2 batch, and the remaining 1/2 batch are random sampled. $S_\text{max}$ always keeps track of the maximal 1/2 samples in the batch.

We compare the practical result of gradient penalty $\mE_{\hat{x}\sim B} [ \big\lVert \nabla f(x) \big\rVert_2^2]$ and the proposed maximum gradient penalty in Figure \ref{fig_gradient_direction_maxgp_compare}. Before switching to maximum gradient penalty, we struggled for a long time and cannot achieve a high quality result as showed in Figure \ref{fig_maxgp}. The other forms of gradient penalty \citep{wgangp,wganlp} perform similar as $\mE_{\hat{x}\sim B} [ \big\lVert \nabla f(x) \big\rVert_2^2]$. %In our sense, if one would like to enforce dynamic k-Lipschitz, it is unnecessary to regularize $\max(0, \big\lVert \nabla f(x) \big\rVert_2-1)^2$. Because the dynamic Lipschitz constant k is often not equal to 1 and it depends on the $\lambda$.

\section{Discussion on No-Differentiable $\ff(x)$} \label{app_coupling}

If $\ff(x)$ is k-Lipschitz and $\ff(y)-\ff(x)=k \cdot d(x,y)$, we say that $(x, y)$ are coupled. When a sample $x$ is coupled with more than one $y$ and these $y$ lie in different directions of $x$, $\ff(x)$ is non-differentiable at $x$ and it will has sub-gradient along each direction.

%We show in Figure for two typical non-differentiable cases in practice. 
When the $\ff(x)$ non-differentiable, due to the smoothness of practically-used neural network, as we noticed in the experiments, it usually behaviors as that the gradient direction is pointing in the middle of these sub-gradient (more strictly, a linear combination of these sub-gradients). 

It seems that when the $\pgen$ is discrete (simulating discrete token generation, such as language and music), it is easy to become non-differentiable: in the optimal transport perspective, once it is required to move to more than one targets, $\ff(x)$ is non-differentiable at this point. 

One way to alleviate this above problem is adding noise (e.g. Gaussian) to each discrete token from $\pgen$. The discrete token with different noises now disperse to different targets. In the practical generator for continuous token, such as images, this kind of non-differentiable problem naturally get solved.

The more serious non-differentiable problem traces back to the Monge problem \citep{oldandnew}, which theoretically discussed under which condition the optimal transport is a one-one mapping, which by nature solve the non-differentiable problem, as each sample now has a single target. 

However, for the Monge problem is solvable, \emph{i.e.} the mapping from $\pgen$ and $\pdata$ is one-one, it requires the $d(x, y)$ to be a strictly convex and super-linear \citep{oldandnew}. Unfortunately, the Euclidean distance, which is necessary to ensure the gradient direction from fake sample directly points toward real sample, does not fit this condition. So we currently does not figure out a practical solution to take advantage of the Monge problem related theories. 

Nonetheless, even if $\ff(x)$ is non-differentiable, the gradient is also usually somehow pointing towards the real samples. And the empirical founding is that: when the $\pgen$ get close to $\pdata$, the non-differentiable problem diminishes.

\section{Proof of the Theorem \ref{theorem}}
\label{app_proof}

Let $\dloss = \mE_{x\sim\pgen} [\phi(f(x))]+\mE_{x\sim\pdata}[\varphi(f(x))]=\int \pgen(x)\phi(f(x))+\pdata(x)\varphi(f(x))dx.$ Let $\partial_x\dloss$ denotes $\pgen(x)\phi(f(x)) + \pdata(x)\varphi(f(x))$. It has $ \dloss= \int \partial_x\dloss dx$. 

Define $J = \dloss + \lambda \cdot k(f)^2$, where $k(f)$ is Lipschitz constant of $f(x)$  . Let $\ff(x) = \argmin_{f} [\dloss + \lambda \cdot k(f)^2]$. Let $\dloss^*(k)=\min_{f\in \mathcal{F_\text{\textbf{k}-Lip}}} \dloss=\min_{f\in \mathcal{F_\text{1-Lip}}, b}\mE_{x\sim\pgen} [\phi(k\cdot f(x)+b)]+\mE_{x\sim\pdata}[\varphi(k\cdot f(x) + b)]$. 

\vspace{5pt}
\begin{lemma} \label{lem_k_neq_zero}
	$\,\forall x, \pdv{\jxloss}{\ff(x)}=0$ if and only if $k(\ff)=0$.
\end{lemma}
\vspace{-10pt}
\begin{proof} $ $ \newline \newline	
	(\rnum{1}) $\,\forall x, \pdv{\jxloss}{\ff(x)}=0$ implies $k(\ff)=0$. 

	For the optimal $\ff(x)$, it holds that $\pdv{J}{k(\ff)} = \pdv{\dloss^*}{k(\ff)} + 2\lambda\cdot k(\ff) = 0$. $\,\forall x, \pdv{\jxloss}{\ff(x)}=0$ implies %$\mE_{x}\pdv{\jxloss}{\ff(x)}=0$. $\mE_{x}\pdv{\jxloss}{\ff(x)}=0$ 
	$\pdv{\dloss^*}{k(\ff)}=0$. We thus conclude that $k(\ff) = 0$.

	(\rnum{2}) $k(\ff)=0$ implies $\,\forall x, \pdv{\jxloss}{\ff(x)}=0$.
	
	For the optimal $\ff(x)$, it holds that $\pdv{J}{k(\ff)} = \pdv{\dloss^*}{k(\ff)} + 2\lambda\cdot k(\ff) = 0$. So $k(\ff)=0$ implies $\pdv{\dloss^*}{k(\ff)}=0$. 
	$k(\ff)=0$ also implies $\,\forall x, y, \ff(x)=\ff(y)$. If there exists some point $x$ such that $\pdv{\jxloss}{\ff(x)}\neq0$, then, given $\,\forall x, y, \ff(x)=\ff(y)$, it is obviously that $\pdv{\dloss^*}{k(\ff)}\neq0$. It is contradictory to $\pdv{\dloss^*}{k(\ff)}=0$. Thus we has $\,\forall x, \pdv{\jxloss}{\ff(x)}=0$.
\end{proof}

% \vspace{10pt}
\begin{lemma} \label{lemma_pair_f}
If $\,\,\forall x, y, \ff(x)=\ff(y)$, then $\pgen=\pdata$.
\end{lemma}
\vspace{-10pt}
\begin{proof}% $ $ \newline \newline
$\,\,\forall x, y, \ff(x)=\ff(y)$ implies $k(\ff)=0$. According to Lemma \ref{lem_k_neq_zero}, $\forall x, \pdv{\jxloss}{\ff(x)} = \pgen(x)\pdv{\phi(\ff(x))}{\ff(x)} $ $ +\pdata(x)\pdv{\varphi(\ff(x))}{\ff(x)} = 0$. So $\frac{\pgen(x)}{\pdata(x)}=-\frac{\pdv{\varphi(\ff(x))}{\ff(x)}}{\pdv{\phi(\ff(x))}{\ff(x)}}$, and thus $\frac{\pgen(x)}{\pdata(x)}$ has a constant value, which straightforwardly implies $\pgen=\pdata$.
\end{proof}

% \vspace{3pt}
\begin{proof}[\textbf{Proof of Theorem \ref{theorem}}] $ $ \newline \newline
(\rnum{1}) Considering the $\ff(x)$, $\forall x \in \sgen \cup\sdata$, if there does not exist a $y$ such that $|\ff(y)-\ff(x)|=k(\ff) \cdot d(x,y)$, because $\ff(x)$ is the optimal, it must hold that $\pdv{\jxloss}{\ff(x)} = 0$. \footnote{Otherwise, as $\ff(x)$ is not constrained by the Lipschitz-continuity condition, we can construct a better $\ff\,$ by adjusting the value of $\ff(x)$ at $x$ according to the non-zero gradient.}

(\rnum{2}) For $x \in \sgen\cup\sdata - \sgen \cap \sdata $, assuming $\pgen(x)\neq0$ and $\pdata(x)=0$, we have $\pdv{\jxloss}{\ff(x)} = \pgen(x)\pdv{\phi(\ff(x))}{\ff(x)} +\pdata(x)\pdv{\varphi(\ff(x))}{\ff(x)}= \pgen(x)\pdv{\phi(\ff(x))}{\ff(x)}>0$, because $\pgen(x) >0$ and $\pdv{\phi(\ff(x))}{\ff(x)} >0$. Then, according to (\rnum{1}), there must exist a $y$ such that $|\ff(y)-\ff(x)| =k(\ff) \cdot d(x,y)$. The other situation can be proved in the same way.% and second claim is proved.

(\rnum{3}) According to Lemma \ref{lemma_pair_f}, in this situation that $\pgen \neq \pdata$, for the optimal $\ff(x)$, there must exist at least one pair of points $x$ and $y$ such that $x\neq y$ and $\ff(x)\neq \ff(y)$. If there are no $x$ and $y$ satisfying that $|\ff(y)-\ff(x)| =k(\ff) \cdot d(x,y)$, it will be contradictory to that $\ff(x)$ is optimal, because we can construct a better $\ff$ by decreasing the value of $k(f)$ until there are two points, \emph{e.g.} $x$ and $y$, constrained by Lipschitz-continuity condition, \emph{i.e.} $|\ff(y)-\ff(x)| =k(\ff) \cdot d(x,y)$.

%If there are no $x$ and $y$ satisfying that $|\ff(y)-\ff(x)| =k(\ff) \cdot d(x,y)$ when $\sdata = \sgen$, then according to our first claim, for all $x \in \sdata$ we have $\pdv{\dloss}{f(x)} = 0$. In this setting, because $\pgen \neq \pdata$, according to Lemma \ref{lemma_pair_f}, there must exist two points $x$ and $y$ such that $\ff(x)\neq \ff(y)$. We can thus construct a better $\ff$ by decreasing the value of $k$ until there are two points, \emph{e.g.} $x$ and $y$, constrained by Lipschitz-continuity condition, \emph{i.e.} $|\ff(y)-\ff(x)| =k(\ff) \cdot d(x,y)$.

(\rnum{4}) In Nash Equilibrium state, it holds that, for any $x \in \sgen \cup \sdata$, $\pdv{J}{k(f)} = \pdv{\dloss^*}{k(f)} + 2\lambda\cdot k(f) = 0$ and $\pdv{\jxloss}{f(x)}\pdv{f(x)}{x}=0$. We claim that in the Nash Equilibrium state, the Lipschitz constant $k(f)$ must be 0.
%If $k(f) \neq 0$, there must exist two points $\hat{x}, \hat{y} \in \sgen \cup\sdata$ satisfies that $|f(\hat{y})-f(\hat{x})| =k(f)\cdot d(\hat{x}, \hat{y})$. 
If $k(f) \neq 0$, 
%there mush exist a point, \emph{e.g.} $\tilde{x}$, such that $\pdv{k(f)}{f(\tilde{x})}\neq0$. It implies $\pdv{\dloss}{f(\tilde{x})}\neq0$ and further implies $\pdv{f(\tilde{x})}{\tilde{x}}$. For the point satisfying $\pdv{k(f)}{f(\tilde{x})}\neq0$, there must $\exists \tilde{y}$ such that $|f(\tilde{y})-f(\tilde{x})| =k(f)\cdot d(\tilde{x}, \tilde{y})$, otherwise, . According to Lemma \ref{lem_gradient_direction}, we have $|\pdv{f(\tilde{x})}{\tilde{x}}| = k(f) \neq 0$.
%we have, for any point $\tilde{x}$ satisfies that $\exists \tilde{y}$ such that $|f(\tilde{y})-f(\tilde{x})| =k(f)\cdot d(\tilde{x}, \tilde{y})$, according to Lemma \ref{lem_gradient_direction}, we have $|\pdv{f(\tilde{x})}{\tilde{x}}| = k(f) \neq 0$, which implies $\pdv{\dloss}{f(\tilde{x})}=0$ and further implies $2\lambda\cdot k(f)\cdot \pdv{k(f)}{f(\tilde{x})}=0$ and thus $\pdv{k(f)}{f(\tilde{x})}=0$. And it is obviously that, for any point $\hat{x}$ that $\nexists\hat{y}$ fitting $|f(\hat{y})-f(\hat{x})| =k(f)\cdot d(\hat{x}, \hat{y})$, it holds $\pdv{k(f)}{f(\hat{x})}=0$. However, $\forall x, \pdv{k(f)}{f(x)}=0$ is contradictory to the definition of $k(f)$.
according to Lemma \ref{lem_k_neq_zero}, there must exist a point $\hat{x}$ such that $\pdv{[{\partial}_{\hat{x}}\dloss]}{f(\hat{x})}\neq0$. And according to (\rnum{1}), it must hold that $\exists\hat{y}$ fitting $|f(\hat{y})-f(\hat{x})| =k(f)\cdot d(\hat{x}, \hat{y})$. According to Theorem \ref{lemma}, we have $\big\Vert\pdv{f(x)}{x}\big\vert_{x=\hat{x}}\big\Vert_2 = k(f) \neq 0$. This is contradictory to that $\pdv{\jxloss}{f(x)}\pdv{f(x)}{x}\big\vert_{x=\hat{x}}\!=0$.
Thus $k(f)=0$, that is, $\forall x \in \sgen \cup \sdata$, $\pdv{f(x)}{x}=0$, which means $\,\forall x, y, f(x)=f(y)$. According to Lemma \ref{lemma_pair_f}, $\,\forall x, y, f(x)=f(y)$ implies $\pgen=\pdata$. Thus $\pgen=\pdata$ is the only Nash Equilibrium of our system.
\end{proof}

\vspace{-5pt}
\textbf{Remark:} For the Wasserstein distance, $\pdv{\jxloss}{\ff(x)}=0$ if and only if $\pgen(x)=\pdata(x)$. For the Wasserstein distance, penalizing the Lipschitz constant also benefits: at the convergence state, it holds $\pdv{\ff(x)}{x}=0$.

%\vspace{10pt}
%The proof of Theorem \ref{lemma} can be found at Section \ref{lip_direction}.

%\newpage
\section{On the importance of Eq.~\ref{eq_solvable}}

Requiring $\phi(x)$ and $\varphi(x)$ to satisfy Eq.~\ref{eq_solvable} is important, because it is the non-trivial condition that makes sure $\mE_{x\sim\pgen} [\phi(f(x))]+\mE_{x\sim\pdata}[\varphi(f(x))] + \lambda \cdot k(f)^2$ has attainable global minimum with respect to $f$. %From Lemma \ref{lemma_monotonical}, $\dloss$ monotonically increases when $k(f)$ decreases. We will firstly consider the case that $k(f)$ is fixed as $\hat{k}$.

\vspace{5pt}
\begin{theorem} \label{lower_bound}
If $\phi(x)$ and $\varphi(x)$ satisfies Eq.~\ref{eq_solvable}, then for any fixed $\pgen$ and $\pdata$, $\mE_{x\sim\pgen} [\phi(f(x))]+\mE_{x\sim\pdata}[\varphi(f(x))] + \lambda\cdot k(f)^2$ has an lower bound  with respect to $f$. 
\end{theorem}

\begin{proof}$ $\newline \newline
Given $\exists a, \phi'(a)+\varphi'(a)=0$, $\phi''(x)\geq0$ and $\varphi''(x)\geq0$, we have:
\begin{align}
&\quad\,\,\mE_{x\sim\pgen}\nonumber [\phi(f(x))]+\mE_{x\sim\pdata}[\varphi(f(x))] + \lambda\cdot k(f)^2  \\\nonumber
&\geq \mE_{x\sim\pgen} [\phi'(a)(f(x)-a)+\phi(a)]+\mE_{x\sim\pdata}[\varphi'(a)(f(x)-a)+\varphi(a)] + \lambda\cdot k(f)^2 \\\nonumber
&= \mE_{x\sim\pgen} [\phi'(a)f(x)]+\mE_{x\sim\pdata}[\varphi'(a)f(x)] + \lambda\cdot k(f)^2 + c \\\nonumber
&= \phi'(a)[\mE_{x\sim\pgen} [f(x)]-\mE_{x\sim\pdata}[f(x)]] + \lambda\cdot k(f)^2 + c \\\nonumber
&\geq \phi'(a)[k\cdot -W_1(P_r, P_g)] + \lambda\cdot k^2+ c \\\nonumber
&= [-\phi'(a) W_1(P_r, P_g)] \cdot k + \lambda\cdot k^2+ c \\\nonumber
&\geq c - \frac{[\phi'(a) W_1(P_r, P_g)]^2}{4\lambda}  \qedhere
\end{align}
\end{proof}

\begin{remark}
Theorem \ref{lower_bound} implies that there exists an infimum for $\mE_{x\sim\pgen} [\phi(f(x))]+\mE_{x\sim\pdata}[\varphi(f(x))] + \lambda\cdot k(f)^2$. According to the definition of infimum, there exists a sequence of $\{f_n\}^\infty_{n=1}$ such that $\mE_{x\sim\pgen} [\phi(f_n(x))]+\mE_{x\sim\pdata}[\varphi(f_n(x))] + \lambda\cdot k(f_n)^2$ infinitely approaches the infimum. 
\end{remark}

\vspace{3pt}
\begin{remark}
The Lipschitz constant of $f_n$, i.e., $k(f_n)$, as $n$ goes to infinity, is bounded, because $\mE_{x\sim\pgen}\nonumber [\phi(f(x))]+\mE_{x\sim\pdata}[\varphi(f(x))] + \lambda\cdot k(f)^2>[-\phi'(a) W_1(P_r, P_g)] \cdot k(f) + \lambda\cdot k(f)^2+ c$. 
\end{remark}

%For simplicity, we assume that there exists a $\ff$ such that $\mE_{x\sim\pgen} [\phi(\ff(x))]+\mE_{x\sim\pdata}[\varphi(\ff(x))] + \lambda\cdot k(\ff)^2$ equals to the infimum. 

We further present several simply Lemmas. These Lemmas and their proofs would provide some intuitive impressions on why Eq.~\ref{eq_solvable} is necessary and the properties of proposed objectives. 

\vspace{5pt}
\begin{lemma} \label{lemma_solvable}
	Assuming $\pgen$ and $\pdata$ are two delta distributions. If $\phi(x)$ and $\varphi(x)$ satisfies Eq.~\ref{eq_solvable}, then for any fixed $\pgen$ and $\pdata$, $\mE_{x\sim\pgen} [\phi(f(x))]+\mE_{x\sim\pdata}[\varphi(f(x))]$ has global minimum with respect to $f$, for any fixed $k(f)=\hat{k}$.  %  + \lambda \cdot [k(f)]^2
\end{lemma}
% \vspace{-3pt}
\begin{proof}%$ $\newline \newline
	%(\rnum{1}) Assuming $\pgen$ and $\pdata$ are two delta distributions. %(the basic idea)
	%
	Given $\pgen$ and $\pdata$ are two delta distributions, according to Theorem \ref{theorem} and Theorem \ref{corollary}, for $x\sim\pgen$ and $y\sim\pdata$, $\ff(y)-\ff(x)=\hat{k} \cdot d(x,y)$. Let $\ff(x) = \alpha$ and $\beta=\hat{k} \cdot d(x,y)$, then $\ff(y)=\alpha+\beta$. Define $\dloss(\alpha) = \mE_{x \sim \pgen} [ \phi(f(x)) ] + \mE_{x \sim \pdata} [ \varphi(f(x)) ]=\phi(\alpha) + \varphi(\alpha+\beta)$. %we have $\pdv{\dloss}{\alpha} = \pdv{\phi(\alpha)}{\alpha} + \pdv{\varphi(\alpha+\beta)}{\alpha}$. 
	
	Given $\exists \, a$ such that $\phi'(a) + \varphi'(a) = 0$, $\phi''(x) \geq 0$ and $\varphi''(x) \geq 0$,
	we have, 
	when $\alpha$ is small enough (such that, $\alpha<a$ and $\alpha+\beta<a$), $\dloss'(\alpha) = \phi'(\alpha) + \varphi'(\alpha+\beta) \leq \phi'(a) + \varphi'(a) = 0$. Similarly, when $\alpha$ is large enough (such that, $\alpha>a$ and $\alpha+\beta>a$), $\dloss'(\alpha) = \phi'(\alpha) + \varphi'(\alpha+\beta) \geq \phi'(a) + \varphi'(a) = 0$. 
	
	Therefore, $\dloss(\alpha)$ is convex with respect to $\alpha$ and there exists an $\alpha_0$ such that $\dloss'(\alpha_0)=0$, where $\dloss$ achieves its the global minimum. When  $\phi''(x) > 0$ and $\varphi''(x) > 0$, it is the unique global minimum.
\end{proof}	

% \iffalse

% \vspace{6pt}
\begin{lemma} \label{lemma_monotonical}
	Assuming $\pgen$ and $\pdata$ are two delta distributions. If $\phi(x)$ and $\varphi(x)$ satisfies Eq.~\ref{eq_solvable}, then for any fixed $\pgen$ and $\pdata$, $\mE_{x\sim\pgen} [\phi(f(x))]+\mE_{x\sim\pdata}[\varphi(f(x))]$ monotonically increases as $k(f)$ decreases.
\end{lemma}
\begin{proof}%$ $\newline \newline
	%(\rnum{1}) Assuming $\pgen$ and $\pdata$ are two delta distributions %(the basic idea)
	%
	Given $\pgen$ and $\pdata$ are two delta distributions, according to Theorem \ref{theorem} and Theorem \ref{corollary}, for $x\sim\pgen$ and $y\sim\pdata$, $\ff(y)-\ff(x)=k \cdot d(x,y)$. Let $\ff(x) = \alpha$ and $\beta=k \cdot d(x,y) > 0$, then $\ff(y)=\alpha+\beta$. Define $\dloss(\beta) = \min_{f \in \mathcal{F}_\text{k-Lip}} \, \mE_{x \sim \pgen} [ \phi(f(x)) ] + \mE_{x \sim \pdata} [ \varphi(f(x)) ]=\min_\alpha \phi(\alpha) + \varphi(\alpha+\beta)$. We need to prove $\dloss(\beta)$ is monotonically decreasing, for $\beta \geq 0$.
	
	Let $0\leq \beta_1 < \beta_2$, let $\alpha_1=\min_\alpha \phi(\alpha) + \varphi(\alpha+\beta_1)$ and $\alpha_2=\min_\alpha \phi(\alpha) + \varphi(\alpha+\beta_2)$. %We have $\pdv{\phi(x)}{x}\big\vert_{x=\alpha_1} + \pdv{\varphi(x)}{x} \big\vert _{x=\alpha_1+\beta_1} = 0$ and $\pdv{\phi(x)}{x}\big\vert_{x=\alpha_2} + \pdv{\varphi(x)}{x} \big\vert _{x=\alpha_2+\beta_2} = 0$. 
	Given $\varphi'(x)<0$ and $\beta_1 < \beta_2$, we have $\phi(\alpha_1) + \varphi(\alpha_1+\beta_1) > \phi(\alpha_1) + \varphi(\alpha_1+\beta_2)$. Given  $\alpha_2=\min_\alpha \phi(\alpha) + \varphi(\alpha+\beta_2)$, %and $\pdvv{\phi(\alpha) + \varphi(\alpha+\beta_2)}{\alpha}\geq 0$, 
	we further have $\phi(\alpha_1) + \varphi(\alpha_1+\beta_1) > \phi(\alpha_1) + \varphi(\alpha_1+\beta_2) \geq \phi(\alpha_2) + \varphi(\alpha_2+\beta_2)$. Done.
	
	\textbf{Additionally}, with $\phi'(\alpha_1) + \varphi'(\alpha_1+\beta_1) = 0$ and $\varphi''(x)\geq0$, we have $\phi'(\alpha_1) + \varphi'(\alpha_1+\beta_2) \geq 0$. Providing $\phi'(\alpha_2) + \varphi'(\alpha_2+\beta_2) = 0$, $\phi''(x)\geq0$ and $\varphi''(x)\geq0$, we get $\alpha_2\leq\alpha_1$. That is, $\alpha_2\leq\alpha_1<\beta_1<\beta_2$. When $\phi''(x) > 0$ and $\varphi''(x) > 0$, we have $\alpha_2<\alpha_1<\beta_1<\beta_2$.
%
	%\vspace{10pt}
	%(\rnum{2}) The general case.
\end{proof}	

\vspace{5pt}
\begin{lemma} \label{finite}
If the support of $P_g$ and $P_r$ is bounded, i.e., $\exists R$ such that $\lVert x \rVert <R, \forall x \in \bar{P_g} \cup \bar{P_r}$.  Assume $\phi(x)$ and $\varphi(x)$ satisfy Eq.~\ref{eq_solvable} and further have $\phi''(x) > 0 \text{ or } \varphi''(x) > 0$. $\forall N$, $\exists M$, if $\mE_{x\sim\pgen} [\phi(f(x))]+\mE_{x\sim\pdata}[\varphi(f(x))] + \lambda\cdot k(f)^2 < N$, then $|f(x)| < M,\,\, \forall x \in \bar{P_g} \cup \bar{P_r}$. 
\end{lemma}

\begin{proof}$ $\newline \newline
$\exists a, \phi'(a)+\varphi'(a)=0$ and $\phi''(x) + \varphi''(x)>0$ implies $\exists b, \phi'(b)+\varphi'(b)>0$. Then:
\begin{align}
&\quad\,\,\mE_{x\sim\pgen}\nonumber [\phi(f(x))]+\mE_{x\sim\pdata}[\varphi(f(x))] + \lambda\cdot k(f)^2  \\\nonumber
&\geq \mE_{x\sim\pgen} [\phi'(b)(f(x)-b)+\phi(b)]+\mE_{x\sim\pdata}[\varphi'(b)(f(x)-b)+\varphi(b)] + \lambda\cdot k(f)^2 \\\nonumber
&= \mE_{x\sim\pgen} [\phi'(b)f(x)]+\mE_{x\sim\pdata}[\varphi'(b)f(x)] + \lambda\cdot k(f)^2 + c \\\nonumber
&= [\phi'(b)+\varphi'(b)]\mE_{x\sim\pgen} [f(x)]+\varphi'(b) [\mE_{x\sim\pdata}[f(x)] - \mE_{x\sim\pgen}[f(x)]] + \lambda\cdot k(f)^2 + c \\\nonumber
&\geq [\phi'(b)+\varphi'(b)]\mE_{x\sim\pgen} [f(x)]+\varphi'(b) [W_1(P_r, P_g)\cdot k(f)] + \lambda\cdot k(f)^2 + c %\\\nonumber
%&\geq [\phi'(b)+\varphi'(b)]\mE_{x\sim\pgen} [f(x)] - \frac{[\varphi'(b) W_1(P_r, P_g)]^2}{4\lambda} + c 
\end{align}
$\forall x \in \bar{P_g} \cup \bar{P_r}$, if $f(x)=T$, then: $\forall x\in \bar{P_g} \cup \bar{P_r}$, $f(x) \geq T-k(f)\cdot R$.
\begin{align}
\nonumber&\,\,\mE_{x\sim\pgen}\nonumber [\phi(f(x))]+\mE_{x\sim\pdata}[\varphi(f(x))] + \lambda\cdot k(f)^2 \\
\nonumber&\geq [\phi'(b)+\varphi'(b)](T-k(f)\cdot R) +\varphi'(b) [W_1(P_r, P_g)\cdot k(f)] + \lambda\cdot k(f)^2 + c \\ 
\nonumber&\geq [\phi'(b)+\varphi'(b)]T - \frac{[\varphi'(b) W_1(P_r, P_g)-[\phi'(b)+\varphi'(b)]R]^2}{4\lambda} + c 
\end{align}
Given $\mE_{x\sim\pgen}\nonumber [\phi(f(x))]+\mE_{x\sim\pdata}[\varphi(f(x))] + \lambda\cdot k(f)^2 < N$, we have: 
\begin{align}
\nonumber& [\phi'(b)+\varphi'(b)]T - \frac{[\varphi'(b) W_1(P_r, P_g)-[\phi'(b)+\varphi'(b)]R]^2}{4\lambda} + c  < N  \\
\nonumber&\!\Rightarrow T < (N -c + \frac{[\varphi'(b) W_1(P_r, P_g)-[\phi'(b)+\varphi'(b)]R]^2}{4\lambda}) / [\phi'(b)+\varphi'(b)]
\end{align}

Similarly, $\exists a, \phi'(a)+\varphi'(a)=0$ and $\phi''(x) + \varphi''(x)>0$ implies $\exists d, \phi'(d)+\varphi'(d)<0$. And then it implies $T$ is greater than some constant. So, $\exists M$ such that $|f(x)| < M,\,\, \forall x \in \bar{P_g} \cup \bar{P_r}$. 
\end{proof}

%Lemma \ref{finite} tells that when $f$ approaches the optimal (Theorem \ref{lower_bound}), $f(x), \forall x$ is bounded. 

% \fi 

\newpage
\section{Proof on the dual form of Wasserstein distance} \label{app_dual_form}

We here provide a formal proof for our new dual form of Wasserstein distance. The Wasserstein distance is given as follows:
\begin{equation}\label{eq_w_primal_app}
W_1(P_r,P_g) =  \inf_{\pi \in \Pi(P_r,P_g)} \, \E_{(x,y) \sim \pi} \, [d(x, y)],
\end{equation}
where $\Pi(P_r, P_g)$ denotes the collection of all probability measures with marginals $P_r$ and $P_g$ on the first and second factors respectively. 

The dual form of Wasserstein distance is usually written as:
\begin{equation}
\begin{aligned}
I(P_r,P_g) &= {\sup}_{f} \,\, \E_{x \sim P_r} \, [f(x)] - \E_{x \sim P_g} \, [f(x)],  \, \\
&\emph{s.t.} \, f(x) - f(y) \leq d(x, y), \,\, \forall x, \forall y.
\end{aligned}
\label{eq_w_dual_form_app1}
\end{equation}
We will prove that Wasserstein distance in its dual form can also be written as:
\begin{equation}
\begin{aligned}
J(P_r,P_g) &= {\sup}_{f} \,\, \E_{x \sim P_r} \, [f(x)] - \E_{x \sim P_g} \, [f(x)],  \, \\
&\emph{s.t.} \, f(x) - f(y) \leq d(x, y), \,\, \forall x \sim P_r, \forall y \sim P_g,
\end{aligned}
\label{eq_w_dual_form_app2}
\end{equation}
which means the constraint in the dual form of Wasserstein distance can be looser than the common formulation. 

With the compacted formulation, we argued that a well-defined metric, i.e., $J(P_r,P_g)$, which is equivalent to Wasserstein distance $W_1(P_r,P_g)$ and thus can properly measure the distance between two distributions, may also fail to provide a meaningful $\nabla_{\!x} \ff(x)$. This observation indicates that a well-defined distance metric does not necessarily guarantee a meaningful $\nabla_{\!x} \ff(x)$ and thus does not guarantee the convergence of $\nabla_{\!x} \ff(x)$-based GANs.

\subsection{Proving the equivalence}

\begin{theorem}
Given $I(P_r,P_g)=W_1(P_r,P_g)$\footnote{The equivalence of $W_1(P_r,P_g)$ and $I(P_r,P_g)$ is well-known \citep{oldandnew,zemel2012optimal}.}, we have $I(P_r,P_g)=J(P_r,P_g)=W_1(P_r,P_g)$
\end{theorem} 
\vspace{-8pt}
\begin{proof} $ $\newline$ $\newline
    (i) For any $f$ that satisfies ``$f(x) - f(y) \leq d(x, y), \,\, \forall x, \forall y$'', it must satisfy ``$f(x) - f(y) \leq d(x, y), \,\, \forall x \sim P_r, \forall y \sim P_g$''. Thus, $I(P_r,P_g)\leq J(P_r,P_g)$.  
    
    \vspace{3pt}
    (ii) Let $F_J=\{f| \, f(x) - f(y) \leq d(x, y), \,\, \forall x \sim P_r, \forall y \sim P_g \}$.
    
    \hspace{13pt} Let $A=\{(x,y) | x \sim P_r, y \sim P_g \}$ and $I_A=\begin{cases}
    1, (x,y) \in A; \\
    0, otherwise
    \end{cases}$.
    
    \hspace{13pt} Let $A^c$ denote the complementary set of $A$ and define $I_{A^c}$ accordingly.
        
    \vspace{4pt}
    \hspace{13pt} $\forall \pi \in \Pi(P_r, P_g)$， We have the following:
    \vspace{2pt}
    \begin{align}
    J(P_r,P_g) &= \nonumber {\sup}_{f\in F_J} \,\, \E_{x \sim P_r} \, [f(x)] - \E_{x \sim P_g} \, [f(x)]  \\ \nonumber
    & = {\sup}_{f\in F_J} \,\, \E_{(x,y) \sim \pi} [f(x)-f(y)]  \\ \nonumber
    & = {\sup}_{f\in F_J} \,\, \E_{(x,y) \sim \pi} [(f(x)-f(y)) I_A] + \E_{(x,y) \sim \pi} [(f(x)-f(y)) I_{A^c}] \\\nonumber
    & = {\sup}_{f\in F_J} \,\, \E_{(x,y) \sim \pi} [(f(x)-f(y)) I_A] \\\nonumber
    & \leq \E_{(x,y) \sim \pi} [d(x,y) I_A]  \\\nonumber
    & \leq \E_{(x,y) \sim \pi} [d(x,y)].
    \end{align}
    
    \hspace{15pt}$J(P_r,P_g)\leq \E_{(x,y) \sim \pi} [d(x,y)], \forall \pi \in \Pi(P_r, P_g)$ 
    
    \hspace{15pt}$\Rightarrow J(P_r,P_g)\leq \inf_{\pi \in \Pi(P_r,P_g)} \, \E_{(x,y) \sim \pi} \, [d(x, y)] = W_1(P_r,P_g)$.
    
    \vspace{3pt}
    (iii) Combining (i) and (ii), we have $I(P_r,P_g)\leq J(P_r,P_g)\leq W_1(P_r,P_g)$. Given $I(P_r,P_g)=W_1(P_r,P_g)$, we have $I(P_r,P_g)=J(P_r,P_g)=W_1(P_r,P_g)$.        
\end{proof}

\vspace{10pt}
\subsection{Proving $J(P_r,P_g)=W_1(P_r,P_g)$ Directly}

Let $\bar{P_r}$ and $\bar{P_g}$ denote the supports of $P_r$ and $P_g$, respectively. Because $\forall \pi \in \Pi(P_r,P_g)$, it must hold that $\pi(x,y)=0$ for any $(x,y)$ that outside $\bar{P_r}\times\bar{P_g}$, i.e., $P_r(x)=0$ or $P_g(y)=0$ implies $\pi(x,y)=0$. 
We let $\Gamma(\bar{P_r}\times\bar{P_g})$ denote the collection of all probability measures that defined on $\bar{P_r}\times\bar{P_g}$ with marginals $P_r$ and $P_g$ on the first and second factors respectively. Let $M_+(\bar{P_r}\times\bar{P_g})$ be the collection of all non-negative measures (not necessarily probability measures) on $\bar{P_r}\times\bar{P_g}$.

Before the proof, we would like to give several preliminary notes.
\begin{itemize}
    \item A more formal and detailed proof for Kantorovich duality \textbf{with same logic of justification} can be found in Theorem 2.3 of \citep{zemel2012optimal}. And a more relevant version that focused on Wasserstein distance can be found in this \href{https://vincentherrmann.github.io/blog/wasserstein/}{ blog}\footnote{\url{https://vincentherrmann.github.io/blog/wasserstein/}}.
    \item The key change here is that we handle the support of $P_r$ and $P_g$ more carefully, which results in elimination of the unnecessary constraints $f(x)-f(y) \leq d(x, y)$ that involve $(x,y)$ pair where $P_r(x)=0$ or $P_g(x)=0$.
    \item \textbf{The validity of the use of the minimax-principle}, i.e., invert the order of $inf$ and $sup$, in this case, is proved in \citep{zemel2012optimal} and the blog.
\end{itemize} 

\vspace{10pt}
\begin{theorem}
$J(P_r,P_g)=W_1(P_r,P_g)$
\end{theorem} 

%\vspace{10pt}
\begin{proof}
\begin{align}
 &\inf_{\pi \in \Pi(P_r,P_g)} \, \E_{(x,y) \sim \pi} \, [d(x, y)] = \nonumber \inf_{\pi \in \Gamma(\bar{P_r}\times\bar{P_g})} \, \E_{(x,y) \sim \pi} \, [d(x, y)] \\\nonumber
=&\inf_{\pi \in M_+(\bar{P_r}\times\bar{P_g})} \, \Big[ \int_{\bar{P_r}}\int_{\bar{P_g}} d(x, y) \pi(x,y) dx dy + 
\begin{cases} 
0,      \,\,\,\,\,\,\,\,\,\,  \pi \in \Gamma(\bar{P_r}\times\bar{P_g}) \\\nonumber
\infty, \,\,\,\,\,\,\, otherwise 
\end{cases} \!\!\!\! \Big]\\\nonumber
=&\inf_{\pi \in M_+(\bar{P_r}\times\bar{P_g})} \Big[ \int_{\bar{P_r}}\int_{\bar{P_g}} d(x, y) \pi(x,y) dx dy \\\nonumber
& \qquad \qquad \qquad  + \sup_f \big[ \E_{s\sim P_r} [f(s)] - \E_{t\sim P_g} [f(t)] - \int_{\bar{P_r}}\int_{\bar{P_g}} (f(x) - f(y)) \pi(x,y) dx dy \big]\Big]\\\nonumber
=&\inf_{\pi \in M_+(\bar{P_r}\times\bar{P_g})} \sup_f \Big[ \E_{s\sim P_r} [f(s)] - \E_{t\sim P_g} [f(t)] + \int_{\bar{P_r}}\int_{\bar{P_g}} [d(x, y)-(f(x)-f(y))]\pi(x,y)dx dy \Big]\\\nonumber
=&\sup_f \inf_{\pi \in M_+(\bar{P_r}\times\bar{P_g})} \Big[ \E_{s\sim P_r} [f(s)] - \E_{t\sim P_g} [f(t)] +  \int_{\bar{P_r}}\int_{\bar{P_g}} [d(x, y)-(f(x)-f(y))]\pi(x,y)dx dy \Big]\\\nonumber
%=&\sup_f \inf_{\pi \in M_+(\bar{P_r}\times\bar{P_g})} \Big[ \E_{(x,y) \sim \pi} \, [d(x, y)] +  \E_{s\sim P_r} [f(s)] - \E_{t\sim P_g} [f(t)] - \E_{(x,y) \sim \pi} [f(x) - f(y)] \Big]\\\nonumber
=&\sup_f \Big[ \E_{s\sim P_r} [f(s)] - \E_{t\sim P_g} [f(t)] + \inf_{\pi \in M_+(\bar{P_r}\times\bar{P_g})}  \int_{\bar{P_r}}\int_{\bar{P_g}} [d(x, y)-(f(x)-f(y))]\pi(x,y)dx dy \Big]\\\nonumber
=&\sup_f \Big[ \E_{s\sim P_r} [f(s)] - \E_{t\sim P_g} [f(t)] + 
\begin{cases} 
0,       \,\,\, f(x) - f(y) \leq d(x, y), \,\, \forall x \sim P_r, \forall y \sim P_g \\\nonumber
-\infty, \,\, otherwise 
\end{cases} \!\!\!\! \Big] \\\nonumber
=&\sup_f \E_{s\sim P_r} [f(s)] - \E_{t\sim P_g} [f(t)], \,\, \emph{s.t.} \, f(x) - f(y) \leq d(x, y), \, \forall x \sim P_r, \forall y \sim P_g.  \qedhere
\end{align}
\end{proof}

%This proof could provide some insight on why only the constraints that involve the support of $P_g$ and $P_r$ are necessary. 
% \vspace{10pt}
\subsection{Another perspective: eliminating redundant constraint}

To provide a more comprehensive understanding on why constraint $f(x)-f(y)\leq d(x,y)$ that involves point $(x,y)$ where $P_r(x)=0$ or $P_g(x)=0$ is unnecessary, we here give a intuitive explanation in the following way: one can safely remove all constraints that does not involve any point in the support of $P_g$ and $P_r$, because give $f(x)-f(y)\leq d(x,y), \forall x \sim P_r, \forall y \sim P_g$, it is sufficient to bound the value of $f(x)$ for $x\sim P_r$ and $f(y)$ for $y \sim P_g$. 
In other words, given the existence of constraint $f(x)-f(y)\leq d(x,y), \forall x\sim P_r, \forall y \sim P_g$, the other constraints can be safely  eliminated without affecting the final solution.

\section{Connections with Optimal Transport} \label{sec_optimal_transport}
The $1^{st}$-Wasserstein distance, also named as the Earth Mover's distance, is a special form of optimal transport \citep{oldandnew}, which measures the minimal cost of moving the source distribution to the target distribution, and the optimal coupling $\pi (x, y)$ describes the transport plan, \emph{i.e.} how much density we should move from $x$ to $y$. Naturally, updating the generator according to the optimal coupling would pull $\pgen$ towards $\pdata$. 

However, updating the generator according to the optimal coupling $\pi (x, y)$ is a totally different mechanism for training GANs. In typical GANs, we update the generator following $\nabla_{\!x} \ff(x)$. An interesting fact\footnote{Assuming $\ff(x)$ is differentiable.} is that: with Wasserstein GAN objective, when updating the generator according to $\nabla_{\!x} \ff(x)$, it follows the optimal coupling $\pi$, if (and only if) there is Lipschitz-continuity condition and the $d(x, y)$ represents the Euclidean distance.

In general optimal transport, $d(x,y)$ is not required to be a distance and can be any cost function. %\yuxuan{how about change d(x,y) to c(x,y), cost the d in our paper usually refer to a distance. However there is no complete formulation of optimal transport in our paper.} \zhiming{not sure which is better.} 
To the best knowledge of the authors, %As far as the authors known, 
it is hard to access the coupling information from $\nabla_{\!x} \ff(x)$ if $d(x,y)$ is arbitrary. However, fortunately, %to update the generator according to the coupling $\pi$, we does not necessarily use $\nabla_{\!x} \ff(x)$. 
given the optimal coupling $\pi$, directly updating each sample towards its target is also possible. An instance of this line of work can be found in \citep{swgan}, where the objective \citep{largescaleotmap} of generator is 
$\mE_{x \sim \pgen} \, [\mE_{y \sim \pi(\cdot | x)} \, [{d(x, y)}]]$.

%The intuition behind actually reflects another mechanism for GAN training, instead of the traditional gradient back-propagation.

In summary, we think training GANs with optimal mapping and with Lipschitz-continuity condition are two mechanisms with different underlying principles, and Wasserstein GAN in the Lipschitz dual form with Euclidean distance is the connecting point. 
%In Lipschitz-continuity condition based GANs, the optimal $\ff(x)$ might be non-differentiable and the optimal $\ff(x)$ might be hard to achieve. In optimal transport based GANs, the optimal coupling is not guaranteed to be one-to-one mapping (which would result in blurry samples) unless it fits the Monge's condition \cite{oldandnew} and the optimal coupling is also hard to achieve \cite{largescaleotmap}. 

\section{Hyper-parameter \& Network Architecture}
We follow the network architecture proposed in \citep{wgangp} to conduct our experiments on CIFAR-10, Tiny Imagenet, Oxford 102. The details of network architecture are in Table \ref{tab:my_label}.

\begin{table}[h]
	%\vspace{-30pt}
    \centering
    \resizebox{0.58\textwidth}{!}{%
    \begin{tabular}{c|c|c|c}
    \multicolumn{4}{l}{Generator:} \\
    \hline \Tstrut 
    Operation     & Kernel & Resample & Output Dims \\[5pt]
    \hline \Tstrut
    Noise         &   N/A     &    N/A     & 128\\[5pt]
    Linear        &   N/A  &   N/A   &  128$\times$4$\times$4   \\[5pt]
    Residual block & 3$\times$3 &UP& 128$\times$8$\times$8 \\[5pt]
    Residual block & 3$\times$3 & UP & 128$\times$16$\times$16 \\[5pt]
    Residual block & 3$\times$3 & UP & 128$\times$32$\times$32 \\[5pt]
    Conv \& Tanh & 3$\times$3 & N/A & 3$\times$32$\times$32 \\[5pt]
    \hline
    \multicolumn{4}{c}{} \\
    \multicolumn{4}{l}{Critic:} \\
    \hline \Tstrut
    Operation     & Kernel & Resample & Output Dims \\[5pt]
    \hline \Tstrut
    % Add Gaussian Noise       &    N/A    &    N/A     & 32$\times$32$\times$3 & 0.0 &N(0.0, 0.1)\\[5pt]
    Residual Block & 3$\times$3$\times$2 & Down & 128$\times$16$\times$16\\[5pt]
    Residual Block & 3$\times$3$\times$2 & Down & 128$\times$8$\times$8\\[5pt]
    Residual Block &3$\times$3$\times$2 & N/A & 128$\times$8$\times$8\\[5pt]
    Residual Block &3$\times$3$\times$2 & N/A & 128$\times$8$\times$8\\[5pt]
    ReLU,mean pool & N/A & N/A & 128 \\[5pt]
    Linear & N/A & N/A & 1 \\[5pt]
    \hline 
    \multicolumn{4}{c}{} \\
    % \hline
    % \multicolumn{4}{l}{The *layer was only used for class condition experiments} \\
    \hline
    \multicolumn{4}{l}{Optimizer: Adam with beta1=0.0, beta2=0.9;} \\
    \hline
    % \multicolumn{4}{l}{We use weight normalization for each weight} \\
    % \hline
    %\multicolumn{7}{l}{} \\
    %\hline
    \multicolumn{4}{l}{For more details, please refer to our published codes.} \\
    \hline
    \end{tabular}
    }
    \caption{Hyper-parameter and Network Architectures}
    \label{tab:my_label}
% \label{hyper}
\end{table}

\end{document}